\documentclass{article}

\usepackage{amssymb,amsmath,amsthm}

\usepackage[dvips]{epsfig}
\usepackage[latin1]{inputenc}
\usepackage[ruled,vlined,linesnumbered]{algorithm2e}
\usepackage[T1]{fontenc}
\usepackage{graphicx}
\usepackage{graphics}
\usepackage{multirow}
\usepackage{fancyhdr}
\usepackage{bm}
\usepackage[usenames,dvipsnames,x11names]{xcolor}

\usepackage{pgfplots}
\usepackage{etex}
\reserveinserts{18}
\usepackage{morefloats}
\usepackage{subfig}

\usepackage[affil-it]{authblk}
\usepackage{array}
\usepackage{colortbl}

\usepackage{atveryend}
\makeatletter
\let\origcitation\citation
\AtEndDocument{\def\mycites{}%
  \def\citation#1{\g@addto@macro\mycites{#1^^J}\origcitation{#1}}}
\AtVeryEndDocument{\newwrite\citeout\immediate\openout\citeout=\jobname.cit
  \immediate\write\citeout{\mycites}\immediate\closeout\citeout}
\makeatother

\usepackage{tikz}
\usepackage{tikz-3dplot}
\usetikzlibrary{fadings}

\tikzfading[name=radialfade, inner color=transparent!0, outer color=transparent!100]

\DeclareMathOperator*{\minimize}{minimize}
\DeclareMathOperator*{\maximize}{maximize}
\DeclareMathOperator*{\argmax}{argmax}
\DeclareMathOperator*{\argmin}{argmin}

\newcommand{\balpha}{\bm{\alpha}}
\newcommand{\bbR}{\mathbb{R}}
\newcommand{\uno}{\mathbf{1}}
\newcommand{\by}{\mathbf{y}}
\newcommand{\bx}{\mathbf{x}}

\newcommand{\bz}{\mathbf{z}}

\newcommand{\bd}{\mathbf{d}}
\newcommand{\be}{\mathbf{e}}
\newcommand{\bu}{\mathbf{u}}

\newcommand{\bK}{\bm{K}}
\newcommand{\bZ}{\bm{Z}}

\newcommand{\bI}{\bm{I}}
\newcommand{\cI}{\mathcal{I}}

\newcommand{\cZ}{\ensuremath{\mathcal{Z}}}

\newcommand{\cO}{\ensuremath{\mathcal{O}}}

\newcommand{\tset}{\ensuremath{\{(\bx_i,y_i): \bx_i \in \mathcal{X}, y_i \in \{+1,-1\}, i \in \mathcal{I}\}}}

\newcommand{\st}{\ensuremath{\mbox{subject to}}}

\newtheorem{proposition}{Proposition}

\newtheorem{lemma}{Lemma}

\theoremstyle{definition}
\newtheorem{definition}{Definition}
\newtheorem{remark}{Remark}

\begin{document}

\title{A Novel Frank-Wolfe Algorithm. Analysis and Applications to Large-Scale SVM Training}

\author[1]{H\'ector Allende}
\author[2]{Emanuele Frandi}
\author[1]{Ricardo  \~Nanculef}
\author[3]{Claudio Sartori}

\affil[1]{\small{Department of Informatics, Universidad T\'ecnica Federico Santa Mar\'ia, Chile \texttt{\{hallende,jnancu\}@inf.utfsm.cl}}}
\affil[2]{\small{Department of Science and High Technology, University of Insubria, Italy \texttt{emanuele.frandi@uninsubria.it}}}
\affil[3]{\small{Department of Computer Science and Engineering, University of Bologna, Italy, \texttt{claudio.sartori@unibo.it}}}
\date{}

\maketitle

\begin{abstract}
Recently, there has been a renewed interest in the machine learning community for variants of a sparse greedy approximation procedure for concave optimization known as {the Frank-Wolfe (FW) method}. 
In particular, this procedure has been successfully applied to train large-scale instances of non-linear Support Vector Machines (SVMs). Specializing FW to SVM training has allowed to obtain efficient algorithms but also important theoretical results, including convergence analysis of training algorithms and new characterizations of model sparsity. 

In this paper, we present and analyze a novel variant of the FW method based on a new way to perform away steps, a classic strategy used to accelerate the convergence of the basic FW procedure. 
Our formulation and analysis is focused on a general concave maximization problem on the simplex. However, the specialization of our algorithm to quadratic forms is strongly related to some classic methods in computational geometry, namely the Gilbert and MDM algorithms. 

On the theoretical side, we demonstrate that the method matches the guarantees in terms of convergence rate and number of iterations obtained by using classic away steps. In particular, the method enjoys a linear rate of convergence, a result that has been recently proved for MDM on quadratic forms.  

On the practical side, we provide experiments on several classification datasets, and evaluate the results using statistical tests. Experiments show that our method is faster than the FW method with classic away steps, and works well even in the cases in which classic away steps slow down the algorithm. Furthermore, these improvements are obtained without sacrificing the predictive accuracy of the obtained SVM model.
\end{abstract}


\section{Introduction}
In this paper we present a novel variant of the Frank-Wolfe (hereafter FW) method \cite{wolfe1954,Jaggi2013ICMLa}, designed to deal with large-scale instances of the following problem:
\begin{equation}\label{eq:generic-concave-on-the-simplex}
\begin{aligned}
\maximize_{\balpha} & \; g(\balpha) \;\;\; \mbox{subject to} &  \balpha \in S := \left\{\balpha \in \mathbb{R}^m : \sum\nolimits_i \balpha_i = 1\, ,\, \balpha_i \geq 0\right\} \,.
\end{aligned}
\end{equation}

This problem encompasses several models used in machine learning \cite{clarkson08coresets,GartnerJ09}, including hard-margin Support Vector Machines (SVMs) \cite{Keerthi2000} and $L_2$-loss SVMs ($L_2$-SVMs) for binary classification, regression and novelty detection \cite{coreSVMs05tsang,coreSVMs-generalized06tsang}. 

\paragraph{FW Methods and Focus of this Paper} It has been noted by researchers in different fields that approximate solutions to problem (\ref{eq:generic-concave-on-the-simplex}) can be obtained using quite simple iterative procedures. In \cite{yildirim08}, for instance, Yildirim presents two iterative algorithms for the task of approximating the Minimum Enclosing Ball (MEB) of a set of points. In \cite{Damla06linearconvergence}, Ahipasaoglu \textit{et al}. propose similar methods to solve Minimum Volume Enclosing Ellipsoid problems. In \cite{Zhang03Greedy}, Zhang studies similar techniques for convex approximation and estimation of mixture models. All these methods are nowadays identified as variants of a general approximation procedure for maximizing a differentiable concave function on the simplex, which traces back to Frank and Wolfe \cite{wolfe1954, wolfe1970,GuelatMarcotte} and has been recently analyzed by Clarkson \cite{clarkson08coresets} and Jaggi \cite{Jaggi2013ICMLa} under a modern perspective. 

In a nutshell, each iteration of the FW method moves the solution towards the direction along which the linearized objective function increases most rapidly but is still feasible. 
The procedure is related to the idea of \emph{coreset}, coined in the context of computational geometry and denoting a subset of data $C_{\varepsilon}$ which suffices to obtain an approximation to the solution on the whole dataset up to a given precision $\varepsilon$. Clarkson's framework unifies diverse results regarding the existence of small coresets for different instances of problem (\ref{eq:generic-concave-on-the-simplex}). These ideas were used in \cite{GartnerJ09} to characterize the sparsity of SVMs and the convergence properties of training algorithms for geometric formulations of the problem. 

The algorithm studied in this paper is obtained by incorporating a new type of \emph{away step} into the basic FW method. Loosely speaking, instead of moving the solution towards a direction along which the linearized objective function increases, an away step moves the solution away from a direction along which the linearized objective function decreases. This strategy was suggested by Wolfe in \cite{wolfe1970} to improve the convergence rate of the FW method, leading to a variant of the original algorithm called Modified Frank-Wolfe method (hereafter MFW). It has been demonstrated that MFW is linearly convergent under some general assumptions on the properties of problem  (\ref{eq:generic-concave-on-the-simplex}).
However, we have found in \cite{IJPRAI11} that classic away steps do not improve significantly the running times of the FW method on machine learning problems. A similar conclusion was obtained by Ouyang and Gray in \cite{StochasticFW10}. In contrast, our approach experimentally improves on other FW methods and shows theoretical guarantees (e.g. convergence rate) at least as good as those of MFW.  

\paragraph{Applications to SVM Learning} Training non-linear SVMs on large datasets is challenging \cite{smo-second-order05fan}. Effective Interior Point Methods can be devised under some special circumstances, such as kernels which admit low-rank factorizations \cite{low-rank-representations02fine, Gondzio}. However, these methods are not suitable for large-scale problems in a general scenario, mainly due to memory constraints: a general interior point method needs \cO($m^2$) memory and \cO($m^3$) time for matrix inversions, and both are prohibitive even for medium-scale problems. Among the traditional methods devised to cope with this problem, Active Set methods \cite{active-set-implementation06scheinberg,joachims99making-large-scale-svms-practical,advances99} and Sequential Minimal Optimization (SMO) \cite{platt99smo-seminal, smo-second-order05fan} are well-known alternatives among practitioners. Indeed, these are the algorithms of choice in the widely known libraries SVMLight \cite{SVMLIGHT} and LIBSVM \cite{SVMLIB}, respectively. For the linear kernel case, 
Stochastic Gradient Descent (SGD) \cite{bottou-mlss-2004,Bottou07SGD}, specialized sub-gradient methods like Pegasos \cite{Shalev-Shwartz11Pegasos} and Stochastic Dual Coordinate Ascent (SDCA) \cite{hsieh2008dual,shalev2012stochastic} have lately gained popularity in the community as approximate but efficient alternatives to the classic solutions on large-scale problems \cite{yuan2012recent}.   

In the non-linear case, effective methods to deal with large datasets have been recently devised by focusing on formulations which fit problem (\ref{eq:generic-concave-on-the-simplex}) and then applying FW methods. The first work to specialize a variant of the FW method to SVM training is probably due to Tsang \textit{et al}. \cite{coreSVMs05tsang}. Given a labelled set of examples $\tset$, where $\mathcal{X}$ denotes the input space and $\mathcal{I}=\{1,\dots,m\}$ an index set, they adopt the so-called $L_2$-SVM formulation, where the model is built by solving the following optimization problem 
\begin{equation}\label{eq:L2DUAL}
\begin{aligned}
\maximize_{\balpha} & \;\; g(\balpha) = - \balpha^{T}\bK \balpha\ \;\;\;\; \mbox{subject to} & \; \sum\nolimits_i \balpha_i = 1\, , \, \balpha_i \geq 0 \; ,
\end{aligned}
\end{equation}
where $\bK_{i,j}=y_iy_j k(\bx_i,\bx_j) + y_iy_j +\delta_{i,j}/C$, $k: \mathcal{X} \times \mathcal{X} \rightarrow \mathbb{R}$ is the kernel function used in the SVM model and $C$ is the regularization parameter \cite{coreSVMs05tsang,CIARP,IJPRAI11}. Problem (\ref{eq:L2DUAL}) clearly fits problem (\ref{eq:generic-concave-on-the-simplex}). This formulation is adopted mainly because of efficiency: by using the functional of Eqn. (\ref{eq:L2DUAL}), it is possible to exploit the framework introduced in \cite{coreSVMs05tsang}, and further developed in \cite{clarkson08coresets}, to solve the learning problem more easily\footnote{Strictly speaking, \cite{coreSVMs05tsang} is a special case of a FW method which does not address the general form of problem (\ref{eq:L2DUAL}), as a normalization constraint on the quadratic form is required (see Sections \ref{sec_sparsity_FW} and \ref{adaptations_section}).}. Note also that in problem (\ref{eq:L2DUAL}) $\bK$ is positive definite
\footnote{This is easily seen by writing $\bK$ as the sum of two positive semi-definite matrices and a multiple of the identity, $\bK=\by\by^{T} \odot \tilde{\bK} + \by\by^{T} + \tfrac{1}{C} \bI$, where $\by$ is the column vector whose components are the labels $y_i$, $\tilde{\bK}$ is the Gram matrix $\tilde{\bK}_{i,j}= k(\bx_i,\bx_j)$ and $\odot$ is the Hadamard or componentwise product.} for $0 < C < \infty$ and thus $g(\cdot)$ is strictly concave. 

Borrowing a coreset-based algorithm from computational geometry \cite{BadoiuClarkson08optimal-coresets}, the authors obtain that the total number of iterations needed to identify a coreset, i.e. an approximation to the $L_2$-SVM model up to an arbitrary precision $\epsilon$, is bounded by $\cO(1/\varepsilon)$, independently of the size of the dataset. From the iterative structure the algorithm, it follows easily that the size of the coreset is also bounded by $\cO(1/\varepsilon)$. A similar result regarding linear SVMs trained with SDCA has been recently demonstrated in \cite{shalev2012stochastic}.

The latter properties imply in particular that the size of the set of examples required to represent the (approximate) solution, i.e. the number of support vectors in the model, is also independent from the size of the dataset, an improvement on previous lower bounds for the support set size, such as those in \cite{steinwart2003sparseness}, where the bound grows linearly in the size of the dataset. The obtained training algorithm also exhibits linear running times in the number of examples. These are remarkable results in the context of non-linear SVM models, where the support set needs to be explicitly stored in memory to implement predictions and determines the cost of a classification decision in terms of testing time. In addition, a combination of this procedure with certain sampling techniques allows to obtain sub-linear time approximation algorithms \cite{coreSVMs05tsang,CIARP}\footnote{To be rigorous, the probability of identifying a good point in a given iteration depends on the size of the sampling. A sub-linear procedure guaranteeing a constant success probability is studied in \cite{ClarksonHW12}, though it seems that results on the non-linear case are provided only for some kernels.}. In practice, the method was found to be competitive with most well-known SVM software using non-linear kernels \cite{coreSVMs05tsang,coreSVMs-generalized06tsang}. 

Several other papers have recently stressed the efficiency of FW and coreset-based methods in machine learning. In \cite{CIARP} and \cite{Quaderni} the authors investigate the direct application of the FW method to large-scale non-linear SVM training, demonstrating that running times of \cite{coreSVMs05tsang} can be significantly improved as long a minor loss in accuracy is acceptable. Variations of the algorithm based on geometrical reformulations of the learning problem \cite{Kumar2011, GartnerJ09}, stochastic variants of the method \cite{StochasticFW10}, and applications to SVM training on data streams \cite{Wang2010OnlineCore, Rai09Streaming-SVM-usingMEB} and structural SVMs \cite{Jaggi2013ICMLb} have also been proposed.

\paragraph{Contributions} We present a FW method endowed with a new type of optimization step devised to overcome the difficulties observed with the classic MFW approach, while preserving the intuition and benefits behind the introduction of away steps. On the theoretical side, we formulate and analyze the algorithm for the general case of problem (\ref{eq:generic-concave-on-the-simplex}), demonstrating that the method matches the guarantees in terms of convergence rate and number of iterations obtained by using classic away steps. In particular, we show that the method converges linearly to the optimal value of the objective function, and achieves a predetermined accuracy $\varepsilon$ (primal-dual gap) in at most $\cO(1/\varepsilon)$ iterations. Focusing on quadratic objectives, it turns out that the method is strongly related to the Gilbert and Mitchell-Demnyanov-Malozemov (MDM) algorithms, two classic methods in computational geometry \cite{gilbert1966,mitchell1974finding}. Such methods are well-known in machine learning and their properties, in particular their rate of convergence, have been the focus of recent research \cite{Lopez2008,lopez2012convergence}.  

On the practical side, we specialize the algorithm to SVM training and perform detailed experiments on several classification problems. We conclude that our algorithm improves the running times of existing FW approaches without any statistically significant difference in terms of prediction accuracy. In particular, we show that the method is faster than the FW and MFW methods, while MFW is not statistically faster than FW. In addition,
we show that the method is faster than or equal to the FW method when MFW is significantly slower, i.e. when classic away steps fail. In addition, the method is competitive with MFW when FW is significantly slower, i.e., if classic away steps work, our algorithm works as well. Thus, the method represents a robust alternative to implement away steps, enjoying strong theoretical guarantees and providing significant improvements in practice.

\paragraph{Organization} The paper is organized as follows. In Section 2 we give an overview of FW methods and introduce the basic concepts required for their analysis. In Section 3 we present the new method, including a minor variant, and provide some details about its specialization to SVMs. The analysis of convergence is provided in Section 4. In Section 5, we discuss the relation of the proposed method to some classic geometric approaches for a quadratic objective. Experiments on SVM problems are presented in Section 6. Finally, Section 7 closes the paper with some concluding remarks. In addition, some technical results required for the proofs of Section 4 are reported in the Appendix.

\paragraph{Notation} An optimal solution for problem (\ref{eq:generic-concave-on-the-simplex}) is denoted $\balpha^\ast$. A sequence of approximations $\balpha_0,\balpha_1, \ldots, \balpha_k$ to a solution of problem (\ref{eq:generic-concave-on-the-simplex}) is abbreviated $\{\balpha_k\}_k$. The set of indices ${1,2,\ldots,m}$ is denoted $[m]$. The \emph{face} $S_{\cI}$ of the unit simplex $S$ corresponding to a set of indices $\cI \subset [m]$ is the subset of points $\balpha \in S$ such that $\balpha_j =0 \; \forall j \notin \cI$. The term \emph{active face} indicates the face corresponding to the non-zero indices, $\cI_k$, of the current solution $\balpha_{k}$. The term \emph{optimal face}, denoted by $S^{\star}$, indicates the face corresponding to an optimal solution $\balpha^\ast$. The vector $\be_i$ denotes the $i$-th vector of the canonical basis.

\section{Frank-Wolfe Methods}\label{bakcground_section}

The FW method computes a sequence of approximations $\{\balpha_k\}_k$ to a solution of problem (\ref{eq:generic-concave-on-the-simplex}) by iterating until convergence the following steps. First, a linear approximation of $g(\cdot)$ at the current iterate $\balpha_k$ is performed in order to find an ascent direction $\mathbf{d}_k^{\mbox{\tiny FW}}= (\mathbf{u}_k - \bm{\alpha}_k)$, with
\begin{equation}\label{eq:FW-linear-app}
\mathbf{u}_k \in \argmax_{\mathbf{u} \in S} \psi_k(\mathbf{u}) := g(\bm{\alpha}_k) + (\mathbf{u} - \bm{\alpha}_k)^T \nabla g(\bm{\alpha}_k) \, . 
\end{equation}
Since $\mathbf{u}_k$ lies in $S$, it is easy to see that the linear approximation step reduces to $\mathbf{u}_k=\be_{i^{\ast}}$ where $i^{\ast}$ is the largest coordinate of the gradient, i.e.   $i^{\ast} \in \argmax_{i} \nabla g(\bm{\alpha}_k)_i$. The iterate $\balpha_k$ is then moved towards $\be_{i^{\ast}}$, seeking for the best feasible improvement of the objective function. The procedure is summarized in Algorithm \ref{alg:base-fw}. In the rest of this paper we refer to $\be_{i^{\ast}} \in S$ as the \emph{ascent vertex} used by the method.   

\begin{algorithm}
\caption{FW method for problem (\ref{eq:generic-concave-on-the-simplex}).\label{alg:base-fw}}
Compute an initial estimate $\balpha_0$.\\%
Set $\cI_{0}=\{i : \alpha_{0,i} \neq 0\}$.\\
\For{$k=0,1,\ldots$}{%
Search for $i^{\ast} \in \argmax_{i} \nabla g(\balpha_k)_i$\, and define $\mathbf{d}_k^{FW} = \be_{i^{\ast}} - \bm{\alpha}_k$.\\
Perform a line-search to find $\lambda_k \in  \argmax_{\lambda \in [0,1]}  g(\bm{\alpha}_k + \lambda \mathbf{d}_k^{FW})\label{linesearch}$.\\
Update the iterate by $\bm{\alpha}_{k+1}=  \bm{\alpha}_k + \lambda_k \mathbf{d}_k^{FW} = (1-\lambda_k)\bm{\alpha}_k + \lambda_k \be_{i^{\ast}}$.\\
Set $\cI_{k+1}=\cI_{k} \cup \{i^{\ast}\}$.\\
}
\end{algorithm} 

As discussed below, the procedure can be stopped when $g(\balpha_k)$ is ``close enough'' to the optimum.

\subsection{Optimality Measures and Stopping Condition}

It can be shown that the FW method is globally convergent under rather weak assumptions on the properties of the objective function \cite{GuelatMarcotte,wolfe1954}, which are guaranteed to hold for the SVM problem (\ref{eq:L2DUAL}) \cite{yildirim08,IJPRAI11}. In addition, it can be shown that the iterates of this procedure satisfy 
\begin{equation}\label{eq:iterates-quality-primal}
\Delta^p(\balpha_k) := g(\balpha^{\ast}) - g(\balpha_k) \leq \frac{4C_g}{k+3} \ ,
\end{equation} 
where $C_g$ is a constant related to the second derivative of $g$ \cite{clarkson08coresets}. This convergence rate is slow compared to other methods.  However, the simplicity of the procedure implies that the amount of computation per iteration is usually very small. This kind of tradeoff can be favorable for large-scale applications, as testified for example by the widespread adoption of the SMO method in the context of SVMs \cite{platt99smo-seminal, smo-second-order05fan}. 

When $g(\balpha)$ is continuously differentiable, the Wolfe dual of problem (\ref{eq:generic-concave-on-the-simplex}) is
\begin{equation}\label{eq:dual-of-generic}
\minimize_{\balpha} w(\balpha) \,\, , \,\, \mbox{with} \,\,\,\, w(\balpha)= g(\balpha) +  \max_{i} \nabla g(\balpha)_i - \balpha^{T}\nabla g(\balpha) \ .
\end{equation}
As shown in \cite{clarkson08coresets}, the strong duality condition 
\begin{equation}\label{eq:strong-duality}
g(\balpha) \leq g(\balpha^{\ast})=w(\balpha^{\ast}) \leq w(\balpha) 
\end{equation}
holds for any feasible $\balpha$. Thus, another reasonable measure of optimality for the Frank-Wolfe iterates is the so-called \emph{primal-dual gap} 
\begin{equation}\label{eq:primal-dual-gap}
\Delta^d({\balpha}) := w(\balpha) - g(\balpha) = \max_{i} \nabla g(\balpha)_i - \balpha^{T}\nabla g(\balpha) \ . 
\end{equation}

Up to a multiplicative constant ($4C_g$), the primal-dual gap in Eqn. (\ref{eq:primal-dual-gap}) and the primal measure of approximation in Eqn. (\ref{eq:iterates-quality-primal}) are the metrics employed in \cite{clarkson08coresets} to analyze the convergence of Algorithm \ref{alg:base-fw}. The advantage of $\Delta^d(\balpha_k)$ with respect to $\Delta^p(\balpha_k)$ is that the former does not depend on the optimal value of the objective function. Therefore, $\Delta^d(\balpha_k)$ can be explicitly monitored during the execution of the algorithm and can be adopted to implement a stopping condition for Algorithm \ref{alg:base-fw}. In this paper, we adopt this measure to stop the FW method and any of its variants. That is, the algorithms are terminated when
\begin{equation}\label{eq:STOPPING-CONDITION}
\Delta^d({\balpha}_k) = \max_{i} \nabla g(\balpha_k)_i - \balpha_k^{T}\nabla g(\balpha_k) \leq \varepsilon \ ,
\end{equation}
where $\varepsilon > 0$ is a given tolerance parameter. Note that the strong duality condition implies $\Delta^p(\balpha_k) \leq \Delta^d(\balpha_k)$. Therefore, if the algorithm stops at iteration $k$ we also have $\Delta^p(\balpha_k) \leq \varepsilon$.  

Note also that Eqn. (\ref{eq:iterates-quality-primal}) implies that the FW method finds a solution fullfiling $\Delta^p(\balpha_k) \leq \varepsilon$ in at most $K \sim\cO(1/\varepsilon)$ iterations. Clarkson has recently shown that we also have $\Delta^d({\balpha}) \leq \varepsilon$ after at most $\tilde{K} \sim\cO(1/\varepsilon)$ iterates \cite{clarkson08coresets}. Thus, the solution found by the FW method using the stopping condition (\ref{eq:STOPPING-CONDITION}) is guaranteed to be ``close'' to the optimum both primally and dually after $\cO(1/\varepsilon)$ iterations. 

In the analysis presented in this paper, we make use of the following notion of approximation quality introduced in \cite{Damla06linearconvergence}.
\begin{definition} A feasible solution $\balpha$ to problem (\ref{eq:generic-concave-on-the-simplex}) is said a \emph{$\Delta$-approximate solution} if
\begin{align}\label{eq:delta-approximate-condition}
&\Delta^d({\balpha}) \leq \Delta\\
\mbox{and} \, \, \, \, \, \,  &\Delta^s_i({\balpha}) := \nabla g(\balpha)_i - \balpha^{T}\nabla g(\balpha) \geq - \Delta, \;\; \forall i: \alpha_i > 0 \, \ . 
\end{align}
\end{definition} 
The first condition guarantees that a $\Delta$-approximate solution is ``close'' to the optimum both primally and dually. In addition, the second condition ensures that $-\Delta \leq \Delta^s_i({\balpha}) \leq \Delta$ for the active face, that is, the primal-dual gap computed on each active coordinate $i: \balpha_i > 0$ is not far from the largest gap computed among all the coordinates of the gradient. This implies also that the solution $\balpha_{k}$ is ``almost'' optimal in the face of the simplex defined by the non-zero indices.

\subsection{Sparsity of the FW solutions and Coresets}\label{sec_sparsity_FW}

On of the main points of interest for the FW method is the sparsity of the solutions it finds. It should be observed that, in contrast to other methods such as projected or reduced gradient methods, Algorithm \ref{alg:base-fw} modifies only one coordinate of the previous iterate at each step. If the starting solution has $K_0$ non-zero coordinates, iterate $\balpha_k$ has at most $K_0 + k$ non-zero entries. Therefore, our previous remarks about the convergence of the FW method show that there exist solutions with space-complexity $K_0 + \cO(1/\varepsilon)$ that are good approximations for problem (\ref{eq:generic-concave-on-the-simplex}), even if $m$ (the dimensionality of the feasible space and the number of data points in SVM problems) is much larger.

The above properties are essential for in the context of training non-linear SVMs. In this case, each non-zero coordinate
in $\balpha_k$ represents a support training example (a document, image or protein profile) that
needs to be explicitly stored in memory during the execution of the algorithm. In addition, the test complexity of non-linear SVMs is proportional to the number of non-zero
coordinates in $\balpha_k$, which determines the cost of each iteration in training time, and the cost of a classification decision in testing time.

Existence of sparse approximate solutions for problem (\ref{eq:generic-concave-on-the-simplex}) can be linked to the idea of $\varepsilon$-\emph{coreset}, first described for the MEB and other geometrical problems \cite{yildirim08}. For $\varepsilon > 0$, an $\varepsilon$-coreset $P' \subset P$ has the property that if the smallest ball containing $P'$ is expanded by a factor of $1 + \varepsilon$, then the resulting ball contains $P$. That is, if the problem is solved on $P'$, the solution is ``close'' to the solution on $P$. The existence of $\varepsilon$-coresets of size $\cO(1/\varepsilon)$ for the MEB problem was first demonstrated by B\u{a}doiu and Clarkson in \cite{BadoiuClarkson03smaller-coresets,BadoiuClarkson08optimal-coresets}. Note that in large-scale applications $1/\varepsilon$ can be much smaller than the cardinality of $P$.

In \cite{clarkson08coresets}, Clarkson provides a definition of coreset that applies in the general setting of problem (\ref{eq:generic-concave-on-the-simplex}). Basically, a $\varepsilon$-coreset for problem (\ref{eq:generic-concave-on-the-simplex}) is a subset of indices spanning a face of $S$ on which we can compute a good approximate solution. The existence of small $\varepsilon$-coresets implies the existence of sparse solutions which are optimal in their respective active faces. The practical consequence of this result would be the possibility of solving large instances of (\ref{eq:generic-concave-on-the-simplex}) working with a small set of variables of the original problem. 

\begin{definition} \label{def:coreset} An $\varepsilon$-coreset for problem (\ref{eq:generic-concave-on-the-simplex}) is a set of indices $\cI \subset [m]$ such that the solution $\balpha_{\cI}^\ast$ to the reduced problem 
\begin{equation}\label{eq:reduced-problem}
\begin{aligned}
\maximize_{\balpha} & \;\; g(\balpha) \;\;\;\; \mbox{subject to} & \; \balpha \in S_{\cI} := \{\balpha \in S: \balpha_i = 0, \forall i \notin \cI\} \;\; .
\end{aligned}
\end{equation}
satisfies $\Delta^d(\balpha_{\cI}^\ast) \leq \varepsilon$. 
\end{definition} 

As discussed in \cite{clarkson08coresets}, the FW method is not guaranteed to find a $\varepsilon$-coreset after $\cO(1/\varepsilon)$ iterations for problem (\ref{eq:generic-concave-on-the-simplex}). It has been demonstrated that FW is able to find such a coreset in some special cases, e.g., in polytope distance problems \cite{GartnerJ09}. However, in general, $\cO(1/\varepsilon^2)$ iterations may be required. Instead, the computationally intensive modification presented in Algorithm \ref{alg:bc-fw}, generally known as the \textit{fully corrective} variant of FW, does the job. 
\begin{algorithm}
\caption{\label{alg:BCFW-steps} Fully-corrective FW method for problem (\ref{eq:generic-concave-on-the-simplex}).\label{alg:bc-fw}}
Compute an initial estimate $\balpha_0$.\\%
Set $\cI_{0}=\{i : \alpha_{0,i} \neq 0\}$.\\
\For{$k=0,1,\ldots$}{%
Search for $i^{\ast} \in \argmax_{i} \nabla g(\balpha_k)_i$\,.\\ 
Set $\cI_{k+1}=\cI_{k} \cup \{i^{\ast}\}$.\\
Solve the reduced problem (\ref{eq:reduced-problem}) with $\cI=\cI_k$. 
}
\end{algorithm}

Note that Algorithm \ref{alg:BCFW-steps} needs to solve an optimization problem of increasing size at each iteration. This can be considered a generalized version of the well-known B\u{a}doiu-Clarkson (BC) method to compute MEBs in computational geometry and, up to our knowledge, corresponds to the first variant of the FW method applied to SVM problems \cite{coreSVMs05tsang}. 

\subsection{Boosting the Convergence using Away-steps}

It is well-known that the FW method often exhibits a tendency to stagnate near the solution $\alpha^\ast$, resulting in a slow convergence rate \cite{GuelatMarcotte}. As discussed in \cite{yildirim08,IJPRAI11}, this problem can be explained geometrically. Near the solution, the gradient at $\balpha_k$ has a tendency to become nearly orthogonal to the face of the simplex spanned by $\cI_k$ (the non-zero coordinates of $\balpha_k$). Therefore, very little improvement can be achieved by moving $\balpha_k$ towards the ascent vertex $\bu_k$. However, since the solution is not optimal, it is reasonable to think that the solution can be improved working \emph{on the face} spanned by $\cI_k$. Actually, Algorithm \ref{alg:BCFW-steps} works on $\cI_k$ until approximate optimality before exploring the next ascent direction. 

It can be shown that the convergence of the FW method can be boosted by introducing a new type of optimization step. In short the idea is that, instead of moving \emph{towards} the point $\mathbf{u}_k$ maximizing the local linear approximation $\psi_k(\cdot)$ of $g(\cdot)$, we can move \emph{away} from the point of the current face $\mathbf{v}_k$ minimizing $\psi_k(\cdot)$. 
At each iteration, a choice between these two options is made by determining which of the directions (moving towards $\mathbf{u}_k$ or moving away from $\mathbf{v}_k$) is more promising.

Since the point $\mathbf{v}_k$ must lie in the current active face, it is easy to see that the linear approximation step reduces to $\mathbf{v}_k=\be_{j^{\ast}}$, where $j^{\ast}$ is the smallest active coordinate of the gradient, i.e.,  $j^{\ast} \in \argmin_{j \in \cI_k} \nabla g(\bm{\alpha}_k)_j$. The whole procedure, known as the \textit{Modified Frank-Wolfe} (MFW) method, is summarized in Algorithm \ref{alg:mod-fw}. In the rest of this paper, we refer to $\be_{j^{\ast}} \in S$ and $\mathbf{d}_k^{\mbox{\tiny A}} =  (\bm{\alpha}_k - \be_{j^{\ast}})$ as the \emph{descent vertex} and the \emph{away direction} used by the method respectively.  

\begin{algorithm}
\caption{\label{alg:MFW-steps} MFW method for problem (\ref{eq:generic-concave-on-the-simplex}).\label{alg:mod-fw}}
Compute an initial estimate $\balpha_0$.\\%
Set $\cI_{0}=\{i : \alpha_{0,i} \neq 0\}$.\\
\For{$k=0,1,\ldots$}{%
Search for $i^{\ast} \in \argmax_{i} \nabla g(\balpha_k)_i$ and define $\mathbf{d}_k^{FW} = \be_{i^{\ast}} - \bm{\alpha}_k$\,.\\ 
Search for $j^{\ast} \in \argmin_{j \in \cI_{k}} \nabla g(\balpha_k)_j$ and define $\mathbf{d}_k^{A} =  \bm{\alpha}_k - \be_{j^{\ast}}$\, .\\
\eIf{$\nabla g(\bm{\alpha}_k)^T \mathbf{d}_k^{FW} \geq \nabla g(\bm{\alpha}_k)^T \mathbf{d}_k^{A}$}{
Perform a line-search to find $\lambda_{\mbox{\tiny fw}} \in \argmax_{\lambda \in [0,1]}  g(\bm{\alpha}_k + \lambda \mathbf{d}_k^{FW})$. \\
Perform the FW step $\balpha_{k+1}= \balpha_k + \lambda_{\mbox{\tiny fw}} (\be_{i^{\ast}} - \balpha_k)$. \\
Update $\cI_k$ by $\cI_{k+1}= \cI_{k} \cup \{i^*\}$.\\
}{
Perform a line-search to find $\lambda_{\mbox{\tiny away}} \in \argmax_{\lambda \in [0,1]}  g(\bm{\alpha}_k + \lambda \mathbf{d}_k^{A})$. \\
Clip the line-search parameter, $\lambda_{\mbox{\tiny away}\star} = \max(\lambda_{\mbox{\tiny away}},\alpha_{k,j^{\ast}}/(1-\alpha_{k,j^{\ast}}))$\\
Perform the AWAY step $\balpha_{k+1}= \balpha_k + \lambda_{\mbox{\tiny away}\star} (\balpha_k - \be_{j^{\ast}})$.\\
Set $\cI_{k+1}= \cI_{k} \cup \{i^*\}$.\\
If $\lambda_{\mbox{\tiny away}\star} = \alpha_{k,j^{\ast}}$, $\cI_{k+1}= \cI_{k+1} \setminus \{j^*\}$.\\
}
}
\end{algorithm}

In contrast to the FW method, for which only a sub-linear rate of convergence can be expected in general \cite{GuelatMarcotte,yildirim08}, it has been shown that MFW asymptotically exhibits linear convergence to the solution of problem (\ref{eq:generic-concave-on-the-simplex}) under some assumptions on the form of the objective function and the feasible set \cite{GuelatMarcotte,yildirim08,Damla06linearconvergence}. In addition, the MFW algorithm has the potential to compute sparser solutions in practice, since in contrast to the FW method it allows reducing the coordinates of $\balpha_k$ at each step.

\subsection{Adaptations to SVMs}\label{adaptations_section}

In the context of SVM learning, the work of Tsang \emph{et al}. in \cite{coreSVMs05tsang} was arguably the first to point out the properties of the algorithms than can be obtained by applying FW methods to formulations fitting problem (\ref{eq:generic-concave-on-the-simplex}). Their work relies on the equivalence between the SVM problem (\ref{eq:L2DUAL}) and a MEB problem, which holds under a normalization assumption on the kernel function employed in the model \cite{coreSVMs05tsang,coreSVMs-generalized06tsang}. Exploiting this equivalence, and adapting the B\u{a}doiu-Clarkson algorithm for computing a MEB to the problem of training non-linear SVMs, an algorithm called Core Vector Machine (CVM) is obtained, which enjoys remarkable theoretical properties and competitive performance in practice \cite{coreSVMs05tsang}. 

First, the number of support vectors of the model obtained by the CVM is $K_0 + \cO(1/\varepsilon)$ where $K_0$ is a constant and $\varepsilon$ is the tolerance parameter of the method. Therefore, the space complexity of the model is independent of the size and dimensionality of the training set. Second, the number of iterations of the algorithm before termination is also $\cO(1/\varepsilon)$, independent of the size and dimensionality of the training set. To determine the overall time complexity of this method, we note that Algorithm \ref{alg:BCFW-steps} requires a search for the point $i^{\ast}$ representing the best ascent direction in the current approximation of the objective function, an operation that is also performed by the FW and MFW methods. Searching among all of the $m$ training points requires a number of kernel evaluations of order $\mathcal{O}(q_k^2+mq_k) = \mathcal{O}(mq_k)$, where $q_k$ is the cardinality of $\mathcal{I}_k$. Since the cardinality of $\mathcal{I}_k$ is bounded as $\cO(1/\varepsilon)$ (the worst-case number of iterations), we obtain that the CVM has an overall time complexity (measured as the total number of kernel evaluations) of $\cO(1/\varepsilon)*\cO(m/\varepsilon) = \cO(m/\varepsilon^2)$, linear in the number of examples, improving on the super-linear time complexity reported empirically for popular methods like SMO to train SVMs  \cite{platt99smo-seminal, smo-second-order05fan}.   

If $m$ is very large, however, the complexity per iteration can still become prohibitive in practice. A sampling technique, called \emph{probabilistic speedup}, was proposed in \cite{SS2} to overcome this obstacle. This technique was also used to implement the CVM in \cite{coreSVMs05tsang,LibCVM09} leading to SVM training algorithms with an overall time complexity which is independent of the number of training examples. In practice, the index $i^{\ast}$ is computed just on a random subset $\varphi(S^{\prime}) \subset \varphi(S)$ of coordinates, with $|S^{\prime}| \ll |S| = \mbox{constant}$. The overall complexity per iteration is thereby reduced to order $\mathcal{O}(q_k^2 + q_k) = \mathcal{O}(q_k^2)$, a major improvement on the previous estimate, since we generally have $q_k \ll m$. Refer to \cite{Smola01Learning} or \cite{coreSVMs05tsang} for details about this speed-up technique.

More recently, several authors have explored the adaptation of the original FW methods to the task of training SVMs. The advantage of Algorithms \ref{alg:base-fw} and \ref{alg:mod-fw} over Algorithm \ref{alg:BCFW-steps} is that they rely only on analytical steps. As a result, each training iteration becomes significantly cheaper than a CVM iteration and does not depend on any external numerical solver. In practice, the training algorithm might probably require more iterations in order to obtain a solution within the predefined tolerance criterion $\varepsilon$, but the work per iteration is significantly smaller. Such a trade-off has been shown to be worthwhile when dealing with large-scale applications \cite{platt99smo-seminal, smo-second-order05fan,CIARP}. 

In \cite{CIARP, IJPRAI11} the authors show that adopting Algorithms \ref{alg:base-fw} and \ref{alg:mod-fw} the running times of \cite{coreSVMs05tsang} can be significantly improved as long a minor loss in accuracy is acceptable. From the analysis presented in \cite{clarkson08coresets}, it is possible to conclude that this approach enjoys similar theoretical guarantees, namely, linear time in the number of examples and a number of iterations which is independent of the number of examples. The sampling technique to speed-up the computation of $i\ast$ introduced above can be used with these methods as well, in order to obtain overall time complexities which are independent of the number of training patterns. 

In a closely related work \cite{Kumar2011}, Kumar and Yildirim present a specialization of the MFW method to SVM problems, adopting the geometrical formulation studied in \cite{Bennet00DualityGeometry}. This approach reformulates the SVM problem as a minimum polytope distance problem. The obtained method and its properties are also strongly related to the work of Gartner and Jaggi \cite{GartnerJ09}, in which the authors in which the authors show (theoretically) that the FW method as well as the coreset framework introduced in \cite{clarkson08coresets} can be applied to all the currently most used hard and soft margin SVM variants, with arbitrary kernels, to obtain approximate training algorithms needing a number of iterations independent of the number of attributes and training examples. In \cite{StochasticFW10}, Ouyang and Gray propose a stochastic variant of FW methods for online learning of $L_2$-SVMs, obtaining comparable and sometimes better accuracies than state-of-the-art batch and online algorithms for training SVMs. A similar technique has recently been proposed in \cite{Hazan2012ProjectionFree} to allow smooth and general online convex optimization with sub-linear regret bounds \cite{Shalev-Shwartz12OnlineConvexOpt}. Variants of the method proposed in \cite{coreSVMs05tsang} have been introduced in \cite{Wang2010OnlineCore} and \cite{Rai09Streaming-SVM-usingMEB} for training SVMs on data streams. In \cite{Jaggi2013ICMLb} the authors adapted the FW method to train SVMs with structured outputs like graphs and other combinatorial objects \cite{Tsochantaridis2005StructuredSVMs,Bakir2007StructuredData}, obtaining an algorithm which outperforms competing structural SVM solvers\footnote{To be precise, the block-coordinate FW in \cite{Jaggi2013ICMLb}, when applied on the binary SVM as a special case of the structured SVM, becomes a variant of dual coordinate ascent \cite{hsieh2008dual}.}.

\section{The SWAP Method}

We have described in the previous sections how the basic FW method can be modified in order to avoid stagnation near a solution, in this way obtaining an algorithm with a guaranteed rate of convergence. Our previous remarks about the MFW method suggest that this algorithm should terminate faster and find sparser solutions. In practice however, the MFW method is not always as fast as one could expect from the theory.  For instance, the experimental results reported in \cite{yildirim08} and \cite{Damla06linearconvergence} for the MEB and Minimum Volume Enclosing Ellipsoid problems respectively, show that very tight improvements, if any, are obtained using the enhanced method (MFW) with respect to the basic approach. As concerns the problem of training SVMs, results in \cite{IJPRAI11} confirm using statistical tests that MFW is not systematically better than FW. Indeed it may sometimes be slower. Similarly, the authors of \cite{StochasticFW10} argue that the use of away steps does not provide a clear advantage with respect to the standard FW method. 

A possible interpretation of these results can be given by looking at the way in which MFW implements the away steps to keep feasibility, i.e., to ensure the constraint $\sum_i \balpha_i = 1$ is satisfied. The basic idea in the MFW approach is to include the alternative of getting away from a descent vertex of the current face $\be_{j^{\ast}}$, decreasing the $j^{\ast}$-th weight in $\balpha_k$, instead of going toward an ascent vertex $\be_{i^{\ast}}$, which would increase the $i^{\ast}$-th weight in $\balpha_k$. The choice is mutually exclusive. If the algorithm decides to work around $j^{\ast}$, it may lose the opportunity to explore a promising direction of the feasible space, and vice-versa.

On the other hand, if an away step is performed, the weights of the active vertices $i \in \cI_k$ are uniformly scaled by $(1+\lambda)$ to keep feasibility. This scheme not only does considerably perturb the current approximation, since all the weights are modified, but, more importantly, can increase the weights of vertices which do not belong to the optimal face $S^{\star}$. Away steps in the MFW method are thus prone to increase the need of further away steps to eliminate such ``spurious points'' ($i \in \cI_k$, but $i \notin S^{\star}$).

Here, we introduce a new type of away step devised to circumvent these problems while preserving the advantages of MFW. We discuss two variants of the method, obtained by using first and second order approximations of the objective function at each iteration, respectively.

\subsection{Main Construction}
Our method is obtained as follows. As in the previous FW methods, we find, at each iteration, an \emph{ascent vertex} $\be_{i^{\ast}}$, as
\begin{equation}\label{eq:maximum_ascent}
\begin{aligned}
i^{\ast} \in \argmax_i \nabla g(\balpha_k)_i \ ,
\end{aligned}
\end{equation}
and a \emph{descent vertex} $\be_{j^{\ast}}$ on the face spanned by the current solution $\balpha_k$, as
\begin{equation}\label{eq:maximum_ascent}
\begin{aligned}
j^{\ast} \in \argmax_{j \in \cI_k} - \nabla g(\balpha_k)_j  = \argmin_{j \in \cI_k} \nabla g(\balpha_k)_j \ .
\end{aligned}
\end{equation}

However, instead of considering the update $\balpha_{k+1} = \balpha_{k} + \lambda \left(\balpha_{k} - \be_{j^{\ast}}\right)$ for the away step, we
propose a step of the form
\begin{equation}\label{eq:away-step-SWAP}
\begin{aligned}
\balpha_{k+1} = \balpha_{k} + \lambda \left(\be_{i^{\ast}}-\be_{j^{\ast}}\right)%
\ ,
\end{aligned}
\end{equation}
where $\lambda$ is determined by a line-search. That is, instead of exploring the away direction $\mathbf{d}_k^{\mbox{\tiny MFW}} =  (\bm{\alpha}_k - \be_{j^{\ast}})$, our algorithm explores the direction $\mathbf{d}_k^{\mbox{\tiny SWAP}} =  (\be_{i^{\ast}} - \be_{j^{\ast}})$. A sketch is included in Figure (\ref{sketch}). This scheme for implementing away steps provides the following conceptual advantages.

\begin{enumerate}
\item This away step perturbs the current solution $\balpha_k$ only locally, in the sense that the weight of any vertex other than $\be_{i^{\ast}}$ and $\be_{j^{\ast}}$ is preserved.

\item This away step does not increase the weight of vertices $\be_{j}$ of the active face corresponding to descent vertices. These points may correspond to spurious points that need to be removed from the active face to reach the optimal face of the problem.

\item This away step moves the current solution in the away direction and simultaneously in the direction of a \emph{toward step}. That is, it moves away from the descent vertex $\be_{j^{\ast}}$, but also gets closer to the ascent vertex $\be_{i^{\ast}}$ in the same iteration. The step (\ref{eq:away-step-SWAP}) can actually be written as the superposition of two separate steps,

\begin{equation}\label{eq:away-step-SWAP-super}
\begin{aligned}
\balpha_{k+1} & = \frac{1}{2}\left(\balpha_{k} + \lambda \left(\be_{i^{\ast}}-\balpha_{k}\right)\right) \ \ \mbox{toward step}\\ 
& \ + \frac{1}{2}\left(\balpha_{k} + \lambda \left(\balpha_{k} - \be_{j^{\ast}}\right)\right) \ \ \mbox{away step} \ ,
\end{aligned}
\end{equation}
where the first term of the right-hand side $\balpha_{k} + \lambda \left(\be_{i^{\ast}}-\balpha_{k}\right)$ represents the standard toward step in the FW method and the second term, $\balpha_{k} + \lambda \left(\balpha_{k} - \be_{j^{\ast}}\right)$, the away step considered in the MFW approach. Note that the term $\lambda \balpha_{k}$ disappears in the sum, so that only the components corresponding to $i^{\ast}$ and $j^{\ast}$ are updated, leaving the rest of the current solution unchanged.
\end{enumerate}

The new type of away step is called a \emph{SWAP step} and substitutes the MFW away steps in Algorithm \ref{alg:MFW-steps}. The procedure is summarized in Algorithm \ref{alg:SWAP-generic}. Note that we deliberately include some steps which do not represent computational tasks but definitions which simplify the convergence analysis of the next section. 

\begin{small}
\begin{algorithm}[!ht]
Set $k=0$.\\
Compute an initial estimate $\balpha_0$.\\%
Set $\cI_{0}=\{i : \alpha_{0,i} \neq 0\}$.\\
\For{$k=0,1,\ldots$}{%
Search for $i^{\ast} \in \argmax_{i} \nabla g(\balpha_k)_i$ (\emph{ascent direction}).\\
Search for $j^{\ast} \in \argmin_{j: \alpha_{k,j} \neq 0} \nabla g(\balpha_k)_j \label{step:SWAP_descent}$ (\emph{descent direction}).\\
Perform a line-search to find
$\lambda_{\mbox{\tiny swap}} \in \argmax_{\lambda \in [0,1]} g\left(\balpha_k + \lambda (\be_{i^{\ast}} - \be_{j^{\ast}})\right). $\\
Perform a line-search to find
$ \lambda_{\mbox{\tiny fw}} \in \argmax_{\lambda \in [0,1]} g\left(\balpha_k + \lambda (\be_{i\ast} - \balpha_k)\right).$\\
Compute $\delta_{\mbox{\tiny swap}} =  g\left(\balpha_k + \lambda_{\mbox{\tiny swap}} (\be_{i\ast} - \be_{j\ast})\right) - g(\balpha_k)$ (\emph{improvement of a SWAP step}).\\
Compute $\delta_{\mbox{\tiny fw}} =  g\left(\balpha_k + \lambda_{\mbox{\tiny fw}} (\be_{i\ast} - \balpha_k)\right) - g(\balpha_k)$ (\emph{improvement of a toward step}).\\
Compute $\label{step:swap_compute_deltak} \delta_{k} = \max\left(\delta_{\mbox{\tiny swap}},\delta_{\mbox{\tiny fw}}\right)$ (\emph{the best improvement}).\\
\eIf{$\delta_{k} = \delta_{\mbox{\tiny \emph{swap}}}$}{%
Clip the line-search parameter, $\lambda_{\mbox{\tiny swap}\star} = \max(\lambda_{\mbox{\tiny swap}},\alpha_{k,j\ast})$\\
If $\lambda_{\mbox{\tiny swap}\star} = \alpha_{k,j\ast}$ mark the iteration as a SWAP-drop step.\\
If $\lambda_{\mbox{\tiny swap}\star} = \lambda_{\mbox{\tiny swap}}$ mark the iteration as a SWAP-add step.\\
Perform the SWAP step $\balpha_{k+1}= \balpha_k + \lambda_{\mbox{\tiny swap}\star} (\be_{i\ast} - \be_{j\ast})$.\\
Set $\cI_{k+1}= \cI_{k} \cup \{i^*\}$.\\
If a SWAP-drop step was performed, $\cI_{k+1}= \cI_{k+1} \setminus \{j^*\}$.\\
}{
Mark the iteration as a FW step.\\
Perform the FW step $\balpha_{k+1}= \balpha_k + \lambda_{\mbox{\tiny fw}} (\be_{i\ast} - \balpha_k)$. \\
Set $\cI_{k+1}= \cI_{k} \cup \{i^*\}$.\\
}
}%
\caption{\label{alg:SWAP-generic} \small{The SWAP algorithm for problem (\ref{eq:generic-concave-on-the-simplex}).}}
\end{algorithm}
\end{small}

\tdplotsetmaincoords{25}{198}
\begin{figure}[h!!!!]
\begin{center}
\begin{tikzpicture}[tdplot_main_coords,scale=1.195]

\node[] (a) at (6,0,0) {};
\node[] (b) at (0,6,0) {};
\node[] (c) at (0,0,6) {};
\node[] (o) at (0,0,0) {};

 \filldraw[
        draw=red!0,%
        fill=gray!50,%
    ]          (6,0,0)
            -- (0,6,0)
            -- (0,0,6)
            -- cycle;

\foreach \to/\from in {a/o,b/o,c/o}
\draw [dashed] (\to)--(\from);
\draw [] (a)--(b);
\draw [] (c)--(b);
\draw [] (a)--(c);

\node at (b) [circle,fill=black,scale=0.4] {};
\node at (c) [circle,fill=black,scale=0.4] {};
\node at (o) [circle,fill=black,scale=0.4] {};

\fill [blue!80, path fading=radialfade] (a) circle (2);
\fill [violet, path fading=radialfade] (b) circle (2);
\fill [red, path fading=radialfade] (c) circle (2);

    \draw[dashed,->] (0,0,0) -- (6.5,0,0) node[anchor=north east]{$x$};
    \draw[->] (0,0,0) -- (0,6.5,0) node[anchor=north west]{$y$};
    \draw[->] (0,0,0) -- (0,0,6.5) node[anchor=south]{$z$};

\node[label= {[label distance=-0.75cm]0:{$\bm{\alpha}_{k}$}}] (current) at ($0.75*(3.2,2.4,2.4)$) [circle,fill=black,scale=0.4] {};
\draw [thick, teal] (current)--(c) node [near end, above, color=black] {$\bm{d}^{\mbox{\tiny FW}}$} ;
\draw [thick, teal] (current)--($(current) + 0.67*(current) - 0.67*(a)$) node [above,pos=1.2, color=black] {$\bm{d}^{\mbox{\tiny MFW}}$};
\draw [thick, teal] (current)--($(current) + 0.4*(c) - 0.4*(a)$) node [pos=1.2, color=black] {$\bm{d}^{\mbox{\tiny SWAP}}$};

\node[label= below:{$\bm{\alpha}_{k+1}^{\mbox{\tiny MFW}}$} ] (mfw) at ($0.75*(1.76,3.12,3.12)$) [circle,fill=black,scale=0.2] {};
\node[label= below:{$\bm{\alpha}_{k+1}^{\mbox{\tiny SWAP}}$}] (swap) at ($0.75*(0.8,2.4,4.8)$) [circle,fill=black,scale=0.2] {};
\node[label=above:{$\bm{\alpha}_{k+1}^{\mbox{\tiny FW}}$}] (fw) at ($0.75*(2.24,1.68,4.08)$) [circle,fill=black,scale=0.2] {};

\node[label=below:{$\bm{e}_{1}$}] at (a) [circle,fill=black,scale=0.4] {};
\node[label=right:{$\bm{e}_{2}$}] at (b) [circle,fill=black,scale=0.4] {};
\node[label={[label distance=0.05cm]0:{$\bm{e}_{3}$}}] at (c) [circle,fill=black,scale=0.4] {};

\draw [thick, dashed,->,>=triangle 45] (current)--(mfw);
\draw [thick,dashed,->,>=triangle 45] (current)--(swap);
\draw [thick,dashed,->,>=triangle 45] (current)--(fw);

\end{tikzpicture}
 \caption{\label{sketch} A sketch of the search directions used by FW, MFW and SWAP methods. In this representation, $\bm{e}_{i^{\ast}} =  \bm{e}_3$ (ascent vertex wrt the current iterate) and $\bm{e}_{j^{\ast}} =  \bm{e}_1$ (descent vertex wrt the current iterate). The search directions explored by the algorithms (solid lines along the updates) are $\bm{d}^{\mbox{\tiny FW}} =  (\bm{e}_3 - \bm{\alpha}_k)$,  $\bm{d}^{\mbox{\tiny MFW}} = (\bm{\alpha}_k - \bm{e}_1)$ and $\bm{d}^{\mbox{\tiny SWAP}} = (\bm{e}_3 - \bm{e}_1)$ respectively. Note that MFW and SWAP are more effective than FW in reducing the weight of the descent vertex $\bm{e}_1$. However, if $\bm{e}_2$ is also a descent vertex, the MFW update has the side effect of increasing the weight of another descent vertex. This is avoided by the SWAP update, which only increases the weight corresponding to the best ascent vertex.}
\end{center}  
\end{figure}
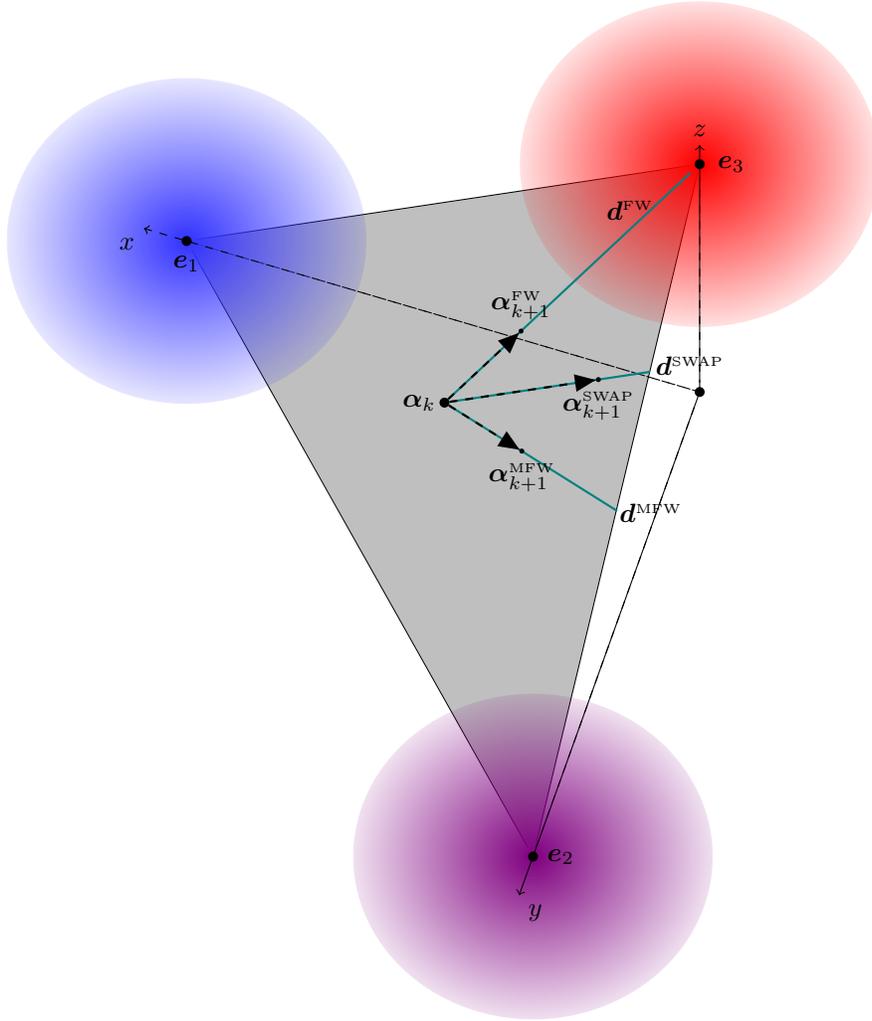  

So as to choose the type of step to perform, the MFW criterion cannot be employed in our method. The MFW method employs a first order approximation of $g(\cdot)$ at the current iterate to predict the value of the objective function at the next iterate. That is, if $\bd$ denotes the search direction, 
\begin{equation}\label{eq:linear-app-to-choose}
\psi_k(\bm{\alpha}_k + \lambda \bd) = g(\bm{\alpha}_k) +  \lambda \bd^T \nabla g(\bm{\alpha}_k) \, . 
\end{equation} 
is computed. The step which gives the largest value of $\psi_k$ is selected. However, a SWAP step \emph{always} gives a larger value of $\psi_k$ than the value obtained using a toward step. Indeed, the value of $\psi_k$ using a SWAP step is
\begin{align}\label{eq:linear-app-to-SWAP}
\psi_k(\bm{\alpha}_k + \lambda \bd_{swap}) &= g(\bm{\alpha}_k) +  \lambda \left(\be_{i^{\ast}}-\be_{j^{\ast}}\right)^T \nabla g(\bm{\alpha}_k)\notag\\ 
&= g(\bm{\alpha}_k) +  \lambda \nabla g(\bm{\alpha}_k)_{i^{\ast}} - \lambda \nabla g(\bm{\alpha}_k)_{j^{\ast}} \, . 
\end{align} 
The value of $\psi_k$ using a toward step is
\begin{align}\label{eq:linear-app-to-SWAP}
\psi_k(\bm{\alpha}_k + \lambda \bd_{fw}) &= g(\bm{\alpha}_k) +  \lambda \left(\be_{i^{\ast}}-\bm{\alpha}_k\right)^T \nabla g(\bm{\alpha}_k)\notag\\ 
&= g(\bm{\alpha}_k) +  \lambda \nabla g(\bm{\alpha}_k)_{i^{\ast}} - \lambda \bm{\alpha}_k^{T}\nabla g(\bm{\alpha}_k) \, . 
\end{align} 
Since $\bm{\alpha}_k^{T}\nabla g(\bm{\alpha}_k)$ is always larger than $\nabla g(\bm{\alpha}_k)_{j^{\ast}}$, a SWAP step would always be preferred using  first-order information to predict the objective function value. 

To address this problem, we observe that the MFW method computes an exact line-search for the search direction selected using $\psi_k$. We thus formulate our method computing the line-search before deciding the type of step to perform. This design requires to perform two line-searches instead of one. However, the estimation of the objective function value at the next iterate is more accurate. 

As we will discuss in the section regarding the adaptation of the procedure to the SVM problem, this computation is particularly simple for the objective function in problem (\ref{eq:L2DUAL}). All the computations are analytical. Furthermore, the exact computation of $\delta_{\mbox{\tiny fw}}$ and $\delta_{\mbox{\tiny swap}}$ involve terms already computed in the line-searches and therefore does not represent an additional overhead for the algorithm.    

\subsection{A Variant using Second Order Information}
\begin{small}
\begin{algorithm}[ht]
Proceed as in Algorithm \ref{alg:SWAP-generic}, but modify step \ref{step:SWAP_descent} as follows:
\begin{equation}\label{eq:bestLAMBDA_second_order-MSW-IMP}
j^{\ast} \in \argmax_{j \in I_k} \frac{\left(\nabla g(\balpha_k)_{i^{\ast}} - \nabla g(\balpha_k)_{j}\right)^{2}}{- 2 \left( \nabla^{2}_{i^{\ast},i^{\ast}} - 2 \nabla^{2}_{i^{\ast},j}  + \nabla^{2}_{j,j} \right)} \ .
\end{equation}
\caption{\label{alg:SWAP-2o} \small{The SWAP-2o algorithm for problem (\ref{eq:generic-concave-on-the-simplex}).}}
\end{algorithm}
\end{small}
All the FW methods introduced previously make use of first-order approximations of the objective function in order to determine the direction toward which the current iterate should be moved. Here, we consider the possibility of using second-order information. If we assume that the objective function is twice differentiable, the second-order Taylor approximation of $g(\cdot)$ in a neighborhood of $\balpha_k$ is 
\begin{equation}\label{eq:taylor_second_order-generic}
\begin{aligned}
g\left(\balpha_k + \lambda\bd\right) &\approx g(\balpha_k) + \lambda \nabla g(\balpha_k)^{T}\bd + \frac{1}{2} \lambda^{2} \bd^{T} \nabla^{2}g(\balpha_k) \bd\ ,
\end{aligned}
\end{equation}
where the Hessian matrix $\nabla^{2} g(\balpha_k)$ is negative semi-definite. Finding the best ascent direction would thus require the computation of the quadratic form $\bd^{T} \nabla^{2}g(\balpha_k) \bd$. Since the matrix $\nabla^{2}g(\balpha_k)$ may be highly dense, which is usually the case in SVM applications, employing a first order relaxation as in Frank-Wolfe methods makes sense in order to obtain lighter iterations. However, we note that the search direction for a SWAP step $\bd_{\mbox{\tiny swap}}=\be_{i^{\ast}} - \be_{j^{\ast}}$ yields a particularly simple expression
\begin{equation}\label{eq:taylor_second_order-MSW}
\begin{aligned}
& g(\balpha_k + \lambda \bd_{\mbox{\tiny swap}})  \\ \approx
\, &  g(\balpha_k) + \lambda \nabla g(\balpha_k)^{T}(\be_{i^{\ast}} -
\be_{j^{\ast}}) + \frac{1}{2} \lambda^{2} (\be_{i^{\ast}}
- \be_{j^{\ast}})^{T} \nabla^{2}g(\balpha_k)
(\be_{i^{\ast}} - \be_{j^{\ast}})\\
= \, & g(\balpha_k) + \lambda \left(\nabla g(\balpha_k)_{i^{\ast}} - \nabla g(\balpha_k)_{j^{\ast}}\right) +
\frac{1}{2} \lambda^{2} \left( \nabla^{2}_{i^{\ast},i^{\ast}} - 2 \nabla^{2}_{i^{\ast},j^{\ast}}  + \nabla^{2}_{j^{\ast},j^{\ast}}  \right) \ ,
\end{aligned}
\end{equation}
where $\nabla^{2}_{i,j} = \nabla^{2}g(\balpha_k)_{i,j}$. 

In order to determine the best pair $\be_{i^{\ast}}, \be_{j^{\ast}}$ we thus need to evaluate three entries of the Hessian matrix. However this is still a computationally hard task for each iteration, since we would need to consider $m |\cI_k|$ pairs of points in order to take a step. We thus adopt the strategy used in the second-order version of SMO proposed in \cite{smo-second-order05fan}. We fix the ascent index $i^{\ast}$ as in the first-order SWAP, and search for the index $j^{\ast}$ in the active set which maximizes the improvement of the second order approximation (\ref{eq:taylor_second_order-MSW}). We call the obtained procedure \emph{second-order SWAP} and we denote it as \emph{SWAP-2o} in the next Sections.

It is worth to note that this approximation is exact for quadratic objective functions, which is the case for the SVM problem (\ref{eq:L2DUAL}). 
Note also that in this case the line-search along the ascent direction $\bd_k$ defined by $i^{\ast}$ and $j^{\ast}$ has a closed-form solution. Indeed,
\begin{equation}\label{eq:bestLAMBDA_second_order-MSW}
\begin{aligned}
\lambda_{\ast} &= \frac{ \left(\nabla g(\balpha_k)_{i^{\ast}} - \nabla g(\balpha_k)_{j^{\ast}}\right)}{-\left( \nabla^{2}_{i^{\ast},i^{\ast}} - 2 \nabla^{2}_{i^{\ast},j^{\ast}}  + \nabla^{2}_{j^{\ast},j^{\ast}} \right) } \ .
\end{aligned}
\end{equation}
From the negative semi-definiteness of $\nabla^{2}g(\balpha_k)$ it follows that $\lambda_{\ast}$ is non-negative. Substituting this step-size in (\ref{eq:taylor_second_order-MSW}), the improvement in the objective function becomes
\begin{equation}\label{eq:bestLAMBDA_second_order-MSW-IMP}
\begin{aligned}
g(\balpha_k + \lambda_{\ast} \bd_{\mbox{\tiny swap}}) - g(\balpha_k) &= \frac{\left(\nabla g(\balpha_k)_{i^{\ast}} - \nabla g(\balpha_k)_{j^{\ast}}\right)^{2}}{-2 \left( \nabla^{2}_{i^{\ast},i^{\ast}} - 2 \nabla^{2}_{i^{\ast},j^{\ast}}  + \nabla^{2}_{j^{\ast},j^{\ast}} \right)} \ ,
\end{aligned}
\end{equation}
which again, from the negative semi-definiteness of $\nabla^{2}g(\balpha_k)$, is non-negative. Naturally, we need to restrict the value of  $\lambda_{\ast}$ to the interval $[0,1]$ in order to obtain a feasible solution for the next step. We thus modify Algorithm \ref{alg:SWAP-generic} as specified in Algorithm \ref{alg:SWAP-2o}. 

\subsection{Notes on the Adaptation to SVM Training}

Here we provide analytical expressions for all the computations required by Algorithm \ref{alg:SWAP-generic} and Algorithm \ref{alg:SWAP-2o} applied to the  
the SVM problem (\ref{eq:L2DUAL}). Similar expressions follow for any quadratic objective function.

For problem (\ref{eq:L2DUAL}), the gradient and Hessian at given iterate $\balpha_k$ take particularly simple expressions:
\begin{equation}\label{eq:gradientSVM}
\nabla g(\balpha_k) = -2 \bK\balpha_k \, , \,\,\,\,\,\,\,  \nabla^2 g(\balpha_k) = -2 \bK \, . 
\end{equation}

Notice that $\balpha_k^{T}\nabla g(\balpha_k) = 2 g(\balpha_k)$. Therefore, the line-searches in Algorithm \ref{alg:SWAP-generic} or Algorithm \ref{alg:SWAP-2o} can be performed analytically as follows. For FW steps,
\begin{equation}\label{eq:lineSEARCH-FW}
\lambda_{\mbox{\tiny fw}} = \frac{\nabla g(\balpha_k)_{i^{\ast}} - 2 g(\balpha_k)}{2 \left( \bK_{i^{\ast},i^{\ast}} + \nabla g(\balpha_k)_{i^{\ast}} - g(\balpha_k) \right)} \, . 
\end{equation}
Note that the quantity $\nabla g(\balpha_k)_{i^{\ast}}$ has been already computed to find the ascent vertex $i^{\ast}$. For SWAP steps,
\begin{equation}\label{eq:lineSEARCH-SWAP}
\lambda_{\mbox{\tiny swap}} = \frac{\nabla g(\balpha_k)_{i^{\ast}} - \nabla g(\balpha_k)_{j^{\ast}}}{2 \left( \bK_{i^{\ast},i^{\ast}} - 2 \bK_{i^{\ast},j^{\ast}}  + \bK_{j^{\ast},j^{\ast}}  \right)} \, . 
\end{equation}
Again, the quantity $\nabla g(\balpha_k)_{j^{\ast}}$ has been already computed to choose the descent vertex $j^{\ast}$. 

The improvement in the objective function can also be calculated analytically. For FW steps,
\begin{equation}\label{eq:improvement-FW}
\delta_{\mbox{\tiny fw}} =  \frac{\left(\nabla g(\balpha_k)_{i^{\ast}} - 2 g(\balpha_k) \right)^2 }{4 \left( \bK_{i^{\ast},i^{\ast}} + \nabla g(\balpha_k)_{i^{\ast}} -g(\balpha_k) \right)} \, . 
\end{equation}
All the terms involved here have already been computed to perform the line-search. Similarly, for SWAP steps,
\begin{equation}\label{eq:improvement-SWAP}
\delta_{\mbox{\tiny swap}} =  \frac{\left(\nabla g(\balpha_k)_{i^{\ast}} - \nabla g(\balpha_k)_{j^{\ast}} \right)^2 }{4 \left( \bK_{i^{\ast},i^{\ast}} - 2 \bK_{i^{\ast},j^{\ast}}  + \bK_{j^{\ast},j^{\ast}}  \right)}  \, . 
\end{equation}

With the exception of the term $\bK_{i^{\ast},j^{\ast}}$, all the computations have already been performed to compute $\delta_{\mbox{\tiny fw}}$ and to choose the descent vertex $j^{\ast}$. We conclude that, compared with MFW procedure, the SWAP method adapted for problem (\ref{eq:L2DUAL}) involves the computation of just one additional term, which is an entry of the kernel matrix $\bK$ defining the SVM problem. 

The objective function value $g(\balpha_k)$ can be computed recursively from the relationship\footnote{A similar recursive equation can be derived to handle the case of SWAP-drop steps.} $g(\balpha_{k+1}) = g(\balpha_k) + \delta_k$. Finally, we observe that the stopping criterion of Eqn. (\ref{eq:STOPPING-CONDITION}) takes the form

\begin{equation}\label{eq:STOPPING-CONDITION-SVM}
\Delta^d({\balpha}_k) = \nabla g(\balpha_k)_{i\ast} - 2 g(\balpha_k) \leq \varepsilon \ ,
\end{equation}
which involves the same already computed terms.

\section{Convergence Analysis of the SWAP Method}

In this section we study the convergence of the SWAP method on problem (\ref{eq:generic-concave-on-the-simplex}), of which the $L_2$-SVM problem (\ref{eq:L2DUAL}) is a particular instance.

We start by demonstrating the global convergence of the SWAP method. Then, we analyze its rate of convergence towards the optimum. For this purpose we will adapt the analysis presented by Ahipasaoglu \textit{et al}. in \cite{Damla06linearconvergence}. Using this framework and using a set of observations concerning the improvement in the objective function after an iteration of the SWAP method, we will be able to prove that the algorithm converges
linearly to the optimal value of the objective function. From a theoretical point of view, these results show that the SWAP enjoys the same
mathematical properties of the MFW method. Finally, we provide bounds on the number of iterations required to fulfill the stopping condition of Eqn. (\ref{eq:STOPPING-CONDITION}). We demonstrate that the algorithm stops in at most $\cO(1/\varepsilon)$ iterations independently of the number of variables $m$, which coincides with the number of training examples in the SVM problem (\ref{eq:L2DUAL}).

Here we only provide proofs for the first-order SWAP method, described in Algorithm \ref{alg:SWAP-generic}. However all the convergence results follow easily for the second-order variant as well. The statements and proofs of some of the technical results used in this section can be found in the Appendix.

We develop our analysis under the following assumptions:
\begin{itemize}
\item[B1.]\label{B1} $g$ is twice continuously differentiable;
\item[B2.]\label{B2} There is an optimal solution $\balpha^{\ast}$ of the optimization problem satisfying the strong sufficient condition of Robinson in \cite{Robinson1982Perturbed}.
\end{itemize}
This is the same set of hypotheses imposed by Yildirim in \cite{yildirim08} and Ahipasaoglu \textit{et al}. in \cite{Damla06linearconvergence} to study the convergence of Frank-Wolfe methods for the MEB problem and the Minimum Volume Enclosing Ellipsoid problem, respectively.

\begin{remark}\label{remark:Robinson}
Robinson's condition is a general version of the classical second order sufficient condition for a solution $\balpha^{\ast}$ to be an isolated local extremum, i.e. locally unique \cite{fiacco}. Referring to the case of a constrained maximization problem for a concave objective $g(\balpha)$, this result requires two conditions to be fulfilled:
\begin{itemize}
\item $\balpha^{\ast}$ is a KKT point \cite{fiacco}.
\item The Hessian of the Lagrangian function at $\balpha^{\ast}$ behaves as a negative definite matrix (positive definite for minimization problems) along the directions belonging to the \textit{critical cone} of the KKT point \cite{nocedal}. Specialized to a quadratic problem on the simplex, i.e. a problem with the form of (\ref{eq:L2DUAL}), this condition assumes the form:
$$
\by^T \bK \by > 0
$$
for all $\by \neq 0$ such that $\balpha^*_i y_i = 0 \; \forall i \mbox{ s.t. } \bu^*_i > 0$, $\balpha^*_i \by_i \geq 0 \; \forall i \mbox{ s.t. } \balpha^*_i = 0 \mbox{ and } \bu^*_i = 0$, $\sum_i \by_i = 0$ (where $\bu^*$ is the vector of Lagrange multipliers at $\balpha^*$ corresponding to inequality constraints, which is unique since the constraints are linear and linearly independent).  
\end{itemize}

The additional analysis in \cite{Robinson1982Perturbed}, which plays a key role in our convergence analysis, essentially describes the conditions under which the stationary points of a small perturbation of the problem lie in a neighborhood of the solution of the original problem. This is also the key step in the proofs of linear convergence provided in \cite{yildirim08} and \cite{Damla06linearconvergence} for the MFW method.

\end{remark}

In \cite{GuelatMarcotte}, Gu\'elat and Marcotte analyzed the convergence properties of FW and MFW methods under the following alternative hypotheses:
\begin{itemize}
\item[A1.]\label{A1} $\nabla g$ is Lipschitz-continuous on the feasible set;
\item[A2.]\label{A2} $g(\balpha)$ is strongly concave;
\item[A3.]\label{A3} Let $\balpha^{\ast}$ be optimal for (\ref{eq:generic-concave-on-the-simplex}) and $T^{\ast}$ be the smallest face of the feasible set containing $\balpha^{\ast}$. Then
\begin{equation*}
(\balpha - \balpha^{\ast})^T \nabla g(\balpha^{\ast}) = 0
\Leftrightarrow \balpha \in T^{\ast} \;\; \mbox{(strict complementarity)}.
\end{equation*}
\end{itemize}
However, this set of assumptions can be difficult to satisfy in practice. In particular, A$3$ is a quite strong assumption and cannot be guaranteed in general. 

Note that assumption B1 implies A1, as from the mean value theorem it follows that
\begin{small}
\begin{equation}\label{eq:lipschitz_concave}
\begin{aligned}
\|\nabla g(\bx) - \nabla g(\by)\| &\leq L \|\bx - \by \| 
\end{aligned}
\end{equation}
\end{small}
for any $\bx,\by$ in the unit simplex, where $L$ is the largest eigenvalue (in modulus) of the matrix $\nabla^{2}g(\bz)$ for $\bz$ in the unit simplex. In addition, B1 holds most of the times in machine learning problems.

It can also be shown that, if problem (\ref{eq:generic-concave-on-the-simplex}) is strongly concave, the strong sufficient condition of Robinson holds, i.e. A2 implies B2 \cite{Damla06linearconvergence}. In particular, this is satisfied by the Wolfe dual of the $L_2$-SVM problem. This fact has been used by Kumar and Yildirim in \cite{Kumar2011} to demonstrate the linear convergence of the MFW method on SVM problems which match (\ref{eq:generic-concave-on-the-simplex}) with quadratic objectives. The convergence of the FW method specialized to a quadratic program (arising from a linear system) has been also studied in \cite{beck2004}. An implicit assumption to demonstrate linear convergence is that the Gram matrix involved in the quadratic form is positive definite. From Remark (\ref{remark:Robinson}) it is easy to see that, for quadratic objectives and linear constraints, this implies the Robinson condition. Recently, the linear convergence of some variants of the FW method for convex optimization on polytopes has been demonstrated in \cite{Garber2013}. The key ingredient in the proof is A2. 

It is worth noting that the Robinson condition is not only weaker than A2 in the sense that A2 implies B2, but also in the sense that it is a local property of the objective function at the solution, instead than a global condition on $g$.        

\subsection{Global Convergence}

\begin{proposition}\label{prop:global_convergence_SWAP} Suppose hypothesis A1 is satisfied. Starting from any feasible $\alpha_0$, Algorithm \ref{alg:SWAP-generic} produces a series of iterates $\{\balpha_{k}\}_k$ such that $g(\balpha_{k})$ converges to $g(\balpha^{\ast})$, where $\balpha^{\ast}$ is a solution of problem (\ref{eq:generic-concave-on-the-simplex}). If $\balpha^{\ast}$ is unique, $\{\balpha_{k}\}_k$ converges to $\balpha^{\ast}$.
\end{proposition}
\begin{proof} 
The key observation is that
both FW and SWAP search directions $\bd_k$ in Algorithm \ref{alg:SWAP-generic} satisfy
\begin{small}
\begin{equation}\label{eq:condition_proof_global_convergence0}
\begin{aligned}
g(\balpha^{\ast}) - g(\balpha_k) &\leq \bd_k^{T}\nabla g(\balpha_k)
\ ,
\end{aligned}
\end{equation}
\end{small}
where $\bd_k = (\be_{i\ast} - \balpha_k)$ 
for FW steps and $\bd_k = (\be_{i\ast} - \be_{j\ast})$ for SWAP steps. 
In the case of FW steps, the
result was stated in \cite{GuelatMarcotte}. However,
it is not hard to see that
\begin{small}
\begin{equation}\label{eq:condition_proof_global_convergence1}
\begin{aligned}
\left(\be_{i\ast} - \be_{j\ast}\right)^{T}\nabla g(\balpha_k) &\geq
\left(\be_{i\ast} - \balpha_k\right)^{T}\nabla g(\balpha_k) \ ,
\end{aligned}
\end{equation}
\end{small}
and thus Eqn. (\ref{eq:condition_proof_global_convergence0}) also holds for SWAP steps. The rest of the proof (see \cite{SWAP_arxiv}) follows easily by replicating the strategy used to demonstrate Theorem $1$ in \cite{GuelatMarcotte}. 

\end{proof}
Note that this result holds in particular under our set of hypotheses, since, as stated above, B1 implies A1.

\subsection{Analysis of the Rate of Convergence}

We now prove a linear convergence result for the SWAP algorithm.  In the proof, we make use of the following technical Lemma. 

\begin{lemma} \label{lemma:swap_guarantees_strong_conditions} 
Suppose B1 holds. After any iteration marked as SWAP-add or FW in Algorithm \ref{alg:SWAP-generic}, the iterate $\balpha_k$ is a $\Delta$-approximate solution with $\Delta = \max\left(2 \sqrt{L \delta_k} \ , \ 2 \delta_k\right)$.
\end{lemma}
\begin{proof} It follows immediately from reordering Eqns. (\ref{eq:bound_weak}) and (\ref{eq:bound_strong}) in the Appendix.
\end{proof}

Note that this result holds for the SWAP algorithm and not for the FW method, since Eqn. (\ref{eq:bound_strong}) requires the SWAP steps.

Note also that for convergence analysis purposes we can assume that $\delta_k \leq L$ for $k$ sufficiently large. This follows from the fact that Algorithm \ref{alg:SWAP-generic} converges globally and that an iterate $\balpha_k$ generated by the algorithm is always feasible. From the first fact it follows that $g(\balpha^{\ast}) - g(\balpha_k)$ becomes arbitrarily small for a sufficiently large $k$. From the second fact it follows that $g(\balpha^{\ast}) \geq g(\balpha_k)$. Since $\delta_k$ is the improvement in the objective function at each iteration of Algorithm \ref{alg:SWAP-generic}, this quantity will be, from some iteration onwards, lower than any predefined constant, in particular $L$. Note now that, if $\delta_k < L$, then
\begin{equation}\label{eq:case_delta_less_L}
\begin{aligned}
2 \delta_k < 2 \sqrt{L \delta_k} \\
\end{aligned} \ .
\end{equation}
Thus, Lemma \ref{lemma:swap_guarantees_strong_conditions} states that for sufficiently large $k$ the iterate $\balpha_k$ produced by a SWAP-add or FW step is a $\Delta$-approximate solution with $\Delta = 2 \sqrt{L \delta_k}$.

\begin{proposition} \label{prop:linear_convergence_without_drops} 
Suppose hypotheses B1 and B2 hold. Let $\balpha^{\ast}$ be the solution of problem (\ref{eq:generic-concave-on-the-simplex}). Then, for sufficiently large $k$, any iteration marked as SWAP-add or FW in Algorithm \ref{alg:SWAP-generic} produces an iterate $\balpha_k$ satisfying the inequality
\begin{equation}\label{eq:convergence without_drops}
\begin{aligned}
\frac{g(\balpha^{\ast}) - g(\balpha_{k+1})}{g(\balpha^{\ast}) - g(\balpha_k)} \leq \left(1 - \frac{1}{M}\right)
\end{aligned} 
\end{equation}
for some constant $M > 1$.
\end{proposition}
\begin{proof} Lemma \ref{lemma:swap_guarantees_strong_conditions} shows that for sufficiently large $k$ the iterate $\balpha_k$ produced by Algorithm \ref{alg:SWAP-generic} after a SWAP-add or FW step is a $\Delta$-approximate solution, with $\Delta = 2 \sqrt{L \delta_k}$. In addition, since the SWAP is globally convergent, $\delta_k$ can be chosen to be arbitrarily small. Thus, for $k$ large enough, the
conditions of Lemma \ref{lemma:key_lemma_for_linear_convergence} hold with $\Delta_\star = 2 \sqrt{L \delta_k}$. From Eqn. (\ref{eq:key_inequality}), we then have that there exists a constant $N$ such that
\begin{equation}\label{eq:key_inequality-for-SWAP}
g(\balpha^{\ast}) - g(\balpha_k) \leq N m \left(2 \sqrt{L
\delta_k}\right)^{2} = 4 N m L\delta_k\ ,
\end{equation}
where $m \gg 1$ is the dimensionality of
$\balpha$, $N$ is a Lipschitz constant depending on the
problem, and $L$ is the largest eigenvalue (in modulus) of the Hessian matrix of
$g(\balpha)$ on the simplex.
\noindent Now, for a SWAP-add or a FW step we have, by definition
of $\delta_k$,
\begin{equation}\label{eq:delta_k_def}
g(\balpha_{k+1}) - g(\balpha_k) = \delta_k\ .
\end{equation}
\noindent Note that the latter is not true for SWAP-drop steps because the real improvement in the objective function differs from the value computed to decide the type of step to perform. Thus,
\begin{equation}\label{eq:toward_linear_inequality}
g(\balpha^{\ast}) - g(\balpha_k) \leq M
\left(g(\balpha_{k+1}) - g(\balpha_k)\right)\ ,
\end{equation}
with $M = 4 N m L$. Adding and subtracting $M g(\balpha^{\ast})$ to the right-hand side produces
\begin{equation}\label{eq:toward_linear_inequality-2}
g(\balpha^{\ast}) - g(\balpha_k) \leq M
\left(g(\balpha^{\ast}) - g(\balpha_k)\right) - M
\left(g(\balpha^{\ast}) - g(\balpha_{k+1}) \right) \ .
\end{equation}
\noindent Equivalently,
\begin{equation}\label{eq:toward_linear_inequality-3}
\begin{aligned}
M \left(g(\balpha^{\ast}) - g(\balpha_{k+1}) \right) & \leq M \left(g(\balpha^{\ast}) - g(\balpha_k)\right) -\left(g(\balpha^{\ast}) - g(\balpha_k)\right)\\
g(\balpha^{\ast}) - g(\balpha_{k+1})  &\leq \left(1 -
\frac{1}{M}\right) \left(g(\balpha^{\ast}) - g(\balpha_k)\right) \ .
\end{aligned}
\end{equation}
\noindent Thus,
\begin{equation}\label{eq:convergence without_drops-copy}
\begin{aligned}
\frac{g(\balpha^{\ast}) - g(\balpha_{k+1})}{g(\balpha^{\ast}) -
g(\balpha_k)} \leq \left(1 - \frac{1}{M}\right)
\end{aligned} \ .
\end{equation}
\end{proof}
This result is analogous to the linear convergence theorems obtained in \cite{Damla06linearconvergence} and \cite{yildirim08}  for the MFW algorithm.

\begin{proposition} \label{prop:linear_convergence_SWAP} At any iteration of Algorithm \ref{alg:SWAP-generic}, the number of SWAP-drop steps does not exceed a half of the total
number of steps $T$ made by the algorithm, plus a finite constant.
\end{proposition}
\begin{proof} Let $F$ be the number of FW steps, $S$ the number of SWAP-add steps, $C$ the number of SWAP-drop steps and $A$ the number of steps that include points to the coreset $\cI_k$. We have $A \leq F + S$, because only FW steps and SWAP-add steps can add points to the coreset. Sometimes they include new points, sometimes they do not. Clearly, $T = F + S + C$. Thus, from the previous inequality we have $T \geq A + C$. Now, the number of steps $C$ that drop points from the coreset cannot be greater than the number of steps that add points to the coreset plus the number of points $I$ in the coreset just after initialization, that is, $A + I\geq C$. Combining the last two inequalities leads $T + A + I \geq A + 2C$, that is, $T + I \geq 2C$. Therefore $C \leq \frac{T}{2} + \frac{I}{2}$, which concludes the proof.
\end{proof}

Proposition \ref{prop:linear_convergence_without_drops} states that there exist a subsequence of the iterates $\{\balpha_k\}_k$ produced by Algorithm \ref{alg:SWAP-generic} such that $\{g(\balpha_k)\}_k$ converges linearly to the optimal value $g(\balpha^{\ast})$  of the objective function in problem (\ref{eq:generic-concave-on-the-simplex}). This subsequence is obtained by dropping from $\{\balpha_k\}_k$ the iterates corresponding to SWAP-drop steps, for which we can only say that the objective function value does not decrease. Thanks to Proposition \ref{prop:linear_convergence_SWAP}, we know that these steps do not affect the overall complexity bound on the number of iterations needed to achieve a given accuracy. 

\subsection{Iteration Complexity Bounds}

We start by proving the following lemma.
\begin{lemma}\label{prop:SWAP_improvement_with_stopping} Suppose B1 holds. By using the stopping condition of Eqn. (\ref{eq:STOPPING-CONDITION}),
any iteration marked as SWAP-add or FW in Algorithm \ref{alg:SWAP-generic} produces an improvement in the objective function
\begin{equation}\label{eq:minimum_improvement_SWAP_noDROP}
\delta_k \geq
\min\left(\frac{\varepsilon^{2}}{4L},\frac{\varepsilon}{2}\right) \ ,
\end{equation}
where $\varepsilon$ is the tolerance parameter.
\end{lemma}
\begin{proof} If the algorithm enters the loop after checking the stopping condition of Eqn. (\ref{eq:STOPPING-CONDITION}),
\begin{equation}\label{eq:condition_iterate}
\nabla g(\balpha_k)_{i\ast} - \balpha_k^{T}\nabla g(\balpha_k) > \varepsilon \ .
\end{equation}
From Eqn. (\ref{eq:fw-case-2}) we obtain 
\begin{equation}\label{eq:condition_iterate2} 
\max\left(2 \sqrt{L \delta_k} \ , \ 2 \delta_k\right) > \varepsilon \ ,
\end{equation}
which leads to the result \footnote{Note that the previous proof is based on the minimal improvement of a FW step. The results holds in general because a SWAP step is performed if and only if the unconstrained SWAP yields a larger improvement.}.
\end{proof}

Note that the converse is not true. The algorithm can stop even if the improvement in the objective function in the last iteration was greater than
$\min\left(\frac{\varepsilon^{2}}{4L},\frac{\varepsilon}{2}\right)$. This happens because the proposed termination criterion fundamentally looks at the possible improvement with standard FW steps.

\begin{proposition} \label{eq:first_bound_steps} Let $K$ be the number of iterations performed by Algorithm \ref{alg:SWAP-generic} until the stopping condition of Eqn. (\ref{eq:STOPPING-CONDITION}) is fulfilled. Then, under the hypotheses of Lemma \ref{prop:SWAP_improvement_with_stopping},
\begin{equation}\label{eq:bound_iterations_SWAPstops}
K \leq Q + \frac{M}{\varepsilon} \ ,
\end{equation}
where $Q$, $M$ are constants independent of $m$ and $\varepsilon$. 
\end{proposition}
\begin{proof} Let $k(\delta)$ denote the number of iterations of Algorithm \ref{alg:SWAP-generic} from the first iterate such that the primal-dual gap satisfies $\Delta^d \leq \delta$ until the first that satisfies $\Delta^d \leq \delta/2$. Since the total improvement in the objective function cannot be greater than $\delta$ and the improvement in the objective function given by a SWAP-add or a FW step is at least that of Lemma \ref{prop:SWAP_improvement_with_stopping} with $\varepsilon = \delta/2$, we can bound $k(\delta)$ as follows\footnote{We assume, for the sake of simplicity, that $\varepsilon < L/2$. Otherwise the proof can be adapted.}:
\begin{equation}
k(\delta) \leq 2 \ \frac{\delta}{\frac{\left(\delta/2\right)^2}{4L}} = 2 \ \frac{16L}{\delta} = \frac{32L}{\delta} \ ,
\end{equation}
where the multiplying factor $2$ comes from the fact that the total number of iterations is at most two times the number of SWAP-add and FW iterations plus a finite constant (see the discussion in the proof of Proposition \ref{prop:linear_convergence_SWAP}). Now, let $K(\varepsilon)$ the number of iterations from the first iterate such that the primal-dual gap satisfies $\Delta^d \leq 1$ until the first that satisfies $\Delta^d \leq \varepsilon$. Clearly, $K \leq K(\varepsilon)$. Now, it is not hard to see that $\lceil\log_2 1/\varepsilon \rceil - 1$ is the smallest positive integer $p$ such that $1/2^p \leq 2 \varepsilon$.
Therefore, we can bound $K(\varepsilon)$ as:
\begin{equation}
\begin{aligned}
K(\varepsilon) &\leq k(1) + k\left(\frac{1}{2}\right) + k\left(\frac{1}{4}\right) \ldots + k\left(\frac{1}{2^{\lceil\log_2 1/\varepsilon \rceil - 1}}\right)\\
K(\varepsilon) &\leq 32L \left(1 + 2 + 4 \ldots + 2^{\lceil\log_2 1/\varepsilon \rceil - 1}\right)\\
& \leq 32L \left(2^{\lceil\log_2 1/\varepsilon \rceil} - 1\right) \leq \frac{64L}{\varepsilon}.
\end{aligned}
\end{equation}
Set  $M = 64L$ and $Q$ to the number of iterations required to obtain an iterate satisfying $\Delta^d \leq 1$ (which is finite and independent of $\varepsilon$) to obtain the result.
\end{proof}

It is also possible to provide a logarithmic bound in $1/\varepsilon$. However, in this case, both the multiplicative and additive constants depend on $m$. Thus, from such result alone we cannot infer the important property that the overall complexity of the algorithm can be bounded independently from the problem size. Furthermore, if $m$ is comparable to or larger than $1/\varepsilon$ (which is often the case in large-scale applications), there is no guarantee that the obtained bound is tighter than the one given by Eqn. (\ref{eq:bound_iterations_SWAPstops}). The proof of this result, which we state below for completeness, can be found in \cite{SWAP_arxiv}.

\begin{proposition}\label{prop:size_coreset_log} Suppose B2 holds. Let $K$ be the number of iterations performed by Algorithm \ref{alg:SWAP-generic} until the stopping condition of Eqn. (\ref{eq:STOPPING-CONDITION}) is fulfilled. Then, there exists $\varepsilon_{0} > 0$ such that, if $\varepsilon < \varepsilon_{0}$,
\begin{equation}\label{eq:bound_iterations_SWAPstops_best}
K \leq \tilde{Q} + \tilde M \log_2\left(\frac{1}{\varepsilon}\right) \ ,
\end{equation}
where $\tilde{Q}$ and $\tilde M$ are constants independent of $\varepsilon$ but dependent on $m$. In particular, $\tilde M \propto m$.
\end{proposition}

\section{Relation of the SWAP to Geometric Algorithms, SMO and Other Results}

In the last years, a number of authors have proposed training SVMs by first reducing the task to a computational geometry problem, and then applying a dedicated algorithm to obtain an exact or approximate solution. Some of these approaches are indeed specialized versions of FW methods. For instance, the so-called Gilbert method \cite{gilbert1966} can be used to train SVMs by approaching the task as a Minimum Norm Problem (MNP) or a Nearest Point Problem (NPP) \cite{Keerthi2000}. It has been noted in \cite{GartnerJ09} that the FW method is equivalent to the Gilbert method on these geometric problems. Thus, up to some implementation details, specializations of the FW and Gilbert methods to the SVM problem coincide. Similarly, the B\u{a}doiu-Clarkson algorithm \cite{BadoiuClarkson03smaller-coresets} can be used to train SVMs which admit a reduction to a Minimum Enclosing Ball (MEB) problem \cite{coreSVMs05tsang}. Nowadays is well-understood that this algorithm is nothing else than the fully corrective FW method (Algorithm \ref{alg:BCFW-steps}) applied to the MEB \cite{clarkson08coresets}. Finally, the algorithms proposed in \cite{Kumar2011} are direct applications of the FW and MFW methods to SVM models which admit an interpretation as an NPP, and the algorithms proposed in \cite{IJPRAI11} are the corresponding application to SVM problems which admit an interpretation as an MNP.   

Recently, in \cite{Lopez2008}, the Mitchell-Demnyanov-Malozemov (MDM) algorithm \cite{mitchell1974finding}, another classic geometric algorithm to solve MNPs and NPPs, was found to be essentially equivalent to the SMO algorithm \cite{platt99smo-seminal,smo-second-order05fan} devised specifically for SVMs and similar quadratic programs \cite{keerthi02generalizedSMO}. In this section, we show that the method proposed in this paper is closely related to the Gilbert and MDM algorithms when applied to MNPs or NPPs. Indeed, on these specific problems, the SWAP can be considered a Gilbert method with the possibility of performing MDM steps. 
  
\paragraph{Polytope Problems and SVMs} Problem (\ref{eq:L2DUAL}) can be cast as an instance of the MNP, which consists in finding the point in a polytope nearest to the origin: $\minimize_{\bm{z}} \|\mathbf{z}\|^2$ $\st \,\, \mathbf{z} \in \cZ$. In this case, the polytope $\cZ$ is the convex hull of a finite set of points $Z = \{\bz_1,\bz_2,\ldots,\bz_m \}$ in a dot product space and the MNP admits the following formulation
\begin{equation}
\begin{aligned}
\minimize_{\bm{\alpha} \in \mathbb{R}^m} & \;\; \tfrac{1}{2} \bm{\alpha}^T\bZ^{T}\bZ  \bm{\alpha} \,\, \st \,\, \alpha_i \geq 0, \,\, \sum\nolimits_i \alpha_i = 1 \, , \label{eq:MNP-matrix} 
\end{aligned}
\end{equation}
where $\bZ$ is the matrix with the points $\bz_i$ arranged in the columns. Any feasible solution $\balpha$ for problem (\ref{eq:MNP-matrix}) yields a feasible solution $\bz$ to the original problem by setting $\bz= \bZ \balpha$. The feasible space in problem (\ref{eq:MNP-matrix}) corresponds to the unit simplex and the objective function is convex. 
Thus, the MNP is an instance of the more general problem (\ref{eq:generic-concave-on-the-simplex}) studied in this paper, with $g(\balpha)= -  \tfrac{1}{2} \, \bm{\alpha}^T \bZ^{T}\bZ  \bm{\alpha} =: g_G(\balpha)$. Furthermore, it is not hard to see that (\ref{eq:MNP-matrix}) is an instance of the SVM problem (\ref{eq:L2DUAL}) with $\bK =  \bZ^{T}\bZ$. Similarly, (\ref{eq:L2DUAL}) is an instance of (\ref{eq:MNP-matrix}) in the space spanned by the points $\bz_i = (y_i \phi(\bx_i)^T, y_i, \tfrac{1}{\sqrt{C}} \be_i^T)^{T}$, where $\phi(\bx_i)$ corresponds to any feature map associated to the kernel used by the SVM. Therefore, algorithms devised to solve MNPs may be adapted to train $L_2$-SVMs and vice-versa. This equivalence was first pointed out and used to build a training algorithm in \cite{Keerthi2000}. 

Other SVM formulations admit similar geometric interpretations (see \cite{clarkson08coresets} and Table 1 in \cite{GartnerJ09}). For instance, hard-margin SVMs have been shown to be equivalent to an NPP, i.e. the problem of computing the pair of nearest points between two polytopes $\mathcal{Z}_1,\mathcal{Z}_2$: $\minimize_{\bm{z}_1,\bm{z}_2} \|\mathbf{z}_1 - \bm{z}_2\|^2  \,\, \st \,\, \mathbf{z}_1 \in \cZ_1, \mathbf{z}_2 \in \cZ_2$. Some variants of the $L_2$-SVM considered here can be transformed into a classic hard-margin SVM and thus admit an NPP interpretation \cite{Keerthi2000,Friess1998KAF}. Soft margin $L_1$-SVMs and $\nu$-SVMs are essentially NPPs with additional constraints of the form ${\alpha}_i \leq \gamma$ \cite{Bennett97geometryin,burges2000geometric}. The authors of \cite{GartnerJ09} use all these equivalences to characterize the sparsity of the solutions and show the existence of linear time training algorithms for all the currently most used SVM variants.  

\paragraph{The Gilbert and MDM Algorithms}  Two classic iterative methods used to solve MNPs are the Gilbert algorithm \cite{gilbert1966} and the Mitchell-Demnyanov-Malozemov (MDM) algorithm \cite{mitchell1974finding}. These methods can be easily adapted to solve NPPs by means of the so called Minkowski difference trick \cite{GartnerJ09}. An iteration of the Gilbert  algorithm for problem (\ref{eq:MNP-matrix}) can be written as $\bz_{k+1} = (1-\lambda_k)\bz_{k} + \lambda_k \bz_{i^\ast}$, where $\lambda_k$ is chosen by performing an exact line-search to minimize $\|\mathbf{z}_{k+1}\|^2$, and 
\begin{equation}
\begin{aligned}
\bz_{i^\ast} \in \argmin_{\bz_i \in Z}(\bz_{k}^{T}\bz_i) \, . \label{eq:directionGilbert} 
\end{aligned}
\end{equation}

It can be shown that the specializations of the Gilbert algorithm and the FW method for problem (\ref{eq:MNP-matrix}) coincide. Note that $\nabla g_G(\balpha_k) = - \bZ^{T}\bZ\bm{\alpha}_k$ and $\nabla g_G(\balpha_k)_i = - \bz_i^T\bZ  \bm{\alpha}_k$. Therefore, applied to (\ref{eq:MNP-matrix}), an iteration of the FW method can be written as $\bm{\alpha}_{k+1} = (1-\lambda_k)\bm{\alpha}_k + \lambda_k \be_{i\ast}$ with ${i^\ast} \in \argmax_{i}(- \bz_i^{T}\bZ\bm{\alpha}_k)$. By setting $\bz_{k} = \bZ\bm{\alpha}_k$, this iteration translates into $\bz_{k+1} = (1-\lambda_k )\bz_{k} + \lambda_k \bz_{i^\ast}$, where $\lambda_k$ is obtained by a line search and ${i^\ast} \in \argmin_{i}(\bz_i^{T}\bz_k)$, which is exactly what the Gilbert method does. 

As the Gilbert method, the MDM algorithm updates $\bz_{k}$ towards the direction given by the point $\bz_{i^\ast}$. However, the search direction is determined using an additional point of the polytope $\cZ$ corresponding to   
\begin{equation}
\begin{aligned}
\bz_{j^\ast} \in \argmax_{\bz_j \in Z}(\bz_{k}^{T}\bz_j) \, . \label{eq:directionMDM} 
\end{aligned}
\end{equation}
An iteration of the MDM algorithm can be written as $\bz_{k+1} = \bz_{k} + \lambda_k \left( \bz_{i^\ast} - \bz_{j^\ast} \right)$, where $\lambda_k$ is computed by a line search. It has been recently shown in \cite{Lopez2008} that the well-known SMO algorithm \cite{platt99smo-seminal,smo-second-order05fan,keerthi02generalizedSMO}, devised to train SVMs and to solve similar quadratic programs, is essentially equivalent to MDM considering SVMs which reduce to NPPs. The key difference is that, on NPPs, MDM chooses $\bz_{i^\ast}$ and $\bz_{j^\ast}$ on the same polytope, which correspond to a class of the SVM problem. SMO instead can choose $\bz_{i^\ast}$ and $\bz_{j^\ast}$ from different classes. This difference however disappears in applications of MDM to MNPs, since in that case the geometric problem involves only one polytope.  

\paragraph{Relation of Gilbert, MDM and SMO to SWAP} 

From the previous discussion we have that the ascent direction explored by the Gilbert method is given by $(\bz_{i^\ast}-\bz_{k})$, with
\begin{equation}
\begin{small}
\begin{aligned}
{i^\ast} \in \argmax_{i} \nabla g_G(\balpha_k)_i \Leftrightarrow {i^\ast} \in \argmax_{i}(- \bz_i^{T}\bZ\bm{\alpha}_k) \Leftrightarrow {i^\ast} \in \argmin_{i}(\bz_i^{T}\bz_k) \ . \label{eq:directionGilbert_} 
\end{aligned}
\end{small}
\end{equation}
\noindent From here, it is straightforward to see that 
\begin{equation}
\begin{small}
\begin{aligned}
{j^\ast} \in \argmin_{j} \nabla g_G(\balpha_k)_j  \Leftrightarrow {j^\ast} \in \argmin_{j}(- \bz_j^{T}\bZ\bm{\alpha}_k)  \Leftrightarrow {j^\ast} \in \argmax_{j}(\bz_j^{T}\bz_k) \ , \label{eq:directionMDM_} 
\end{aligned}
\end{small}
\end{equation}
which corresponds to the \emph{away} vertex used by the SWAP method. Thus, the MDM algorithm updates the current iterate using the same vertices of the polytope that would be considered by a specialization of SWAP to problem (\ref{eq:MNP-matrix}) or to the equivalent SVM problem (\ref{eq:L2DUAL}): $\bz_{i^\ast}$ and $\bz_{j^\ast}$. If SWAP goes for a toward step, it is identical to the FW method on that iteration. From the equivalence between the FW and Gilbert methods, we conclude that the SWAP is identical to Gilbert at iteration $k$ if it decides to not explore the away direction. Otherwise, the search direction used by the SWAP on the simplex is $\left(\be_{i^{\ast}}-\be_{j^{\ast}}\right)$, which translates to $\bd^{\mbox{\tiny SWAP}} = \left( \bz_{i^\ast} - \bz_{j^\ast} \right)$ for an MNP. The update $\balpha_{k+1} = \balpha_{k} + \lambda \left(\be_{i^{\ast}}-\be_{j^{\ast}}\right)$ for problem (1) corresponds to updating $\bz_{k}$ as $\bz_{k+1} = \bz_{k}+ \lambda_k \left( \bz_{i^\ast} - \bz_{j^\ast} \right)$ and computing $\lambda_k$ by a line search. This is exactly the same direction and procedure to set the step size used by MDM. 

We conclude that, applied to polytope problems, the SWAP is equivalent to a Gilbert method with the possibility of performing MDM steps. In this sense, the SWAP is a kind of hybrid Gilbert-MDM which is presented and analyzed for problems beyond the MNP and NPP. Considering the equivalence between MDM and the SMO on SVM problems, we can also state that on these problems the SWAP is a FW method with the possibility of performing SMO steps. However, again, our presentation and analysis is not limited to quadratic forms. The minor variant of the method, using second order information, uses essentially the same criterion proposed in \cite{smo-second-order05fan} to improve on the original Platt's SMO \cite{platt99smo-seminal}.  

It is well known from its introduction that, in general, the Gilbert method converges sub-linearly, that is, the specialization of the FW method to the quadratic program in (\ref{eq:MNP-matrix}) does not improve its rate of convergence \cite{gilbert1966,lopez2012convergence}. However, it has been recently shown that the MDM algorithm converges linearly under some assumptions about the structure of problem (\ref{eq:MNP-matrix}), which include the positive definiteness of the matrix $\bZ^{T}\bZ$ \cite{lopez2012convergence}. In this paper, the analysis addresses a general maximization problem on the simplex and the results rely on hypotheses slightly more general than those used in \cite{lopez2012convergence}, in the sense that asking for a positive definite matrix $\bZ^{T}\bZ$ is equivalent to asking for a positive definite Hessian and this in turn implies the Robinson condition used in our analysis\footnote{See Remark \ref{remark:Robinson}. As we have only linear constraints, the Hessian of the Lagrangian coincides with the Hessian of the objective.}. 

\section{Experiments}
\label{sec:batch_experiments}

In this section, we present several experiments conducted on benchmark classification datasets to evaluate the performance of the proposed methods and related approaches in practice. 

\paragraph{Datasets} The datasets used in this section are listed in Table \ref{TabDatasets} and can be found in several public repositories \cite{SVMLIB,UCI2010}. In order to provide the reader with an idea of the size of each problem, we specify the size $m$ of the training set, the number of features $n$, and the number of classes $K$. We denote by $t$ the number of {\em test examples}, set aside to evaluate the expected accuracy of the computed classifier. 

In the case of multi-category classification problems, we adopt a one-versus-one approach (OVO) \cite{hofmann2008tutorial}\footnote{This was the method used in \cite{coreSVMs05tsang} to extend the CVM beyond binary classification, and according to \cite{Hsu02ComparisonMultiClassSVMs} it usually outperforms other approaches both in terms of accuracy and training time.}. Note that in these cases the number of examples $m$ does not necessarily reflect the complexity of the training problems to be addressed. For example, according to $m$, the {\textbf{MNIST}} and {\textbf{Web w8a}} datasets have a similar size. However, the {\textbf{MNIST}} problem has $10$ classes and the largest binary problem to solve in the OVO scheme has around $13.000$ training examples. The {\textbf{Web w8a}} problem is, in contrast, binary, and thus the whole dataset needs to be handled simultaneously.  For this reason, we also report in Table \ref{TabDatasets} the size $m_{{\max}}$ of the largest binary subproblem and the size $m_{{\min}}$ of the smallest binary subproblem in the OVO decomposition. 

\paragraph{Initialization and Parameters} 
For the initialization of the CVM, FW, MFW and SWAP methods, that is, the computation of a starting solution, we adopted the method proposed for the CVM in \cite{coreSVMs05tsang}. In this approach, the starting solution is obtained by solving problem (\ref{eq:L2DUAL}) on a random subset $\cI_0$ of $p$ training patterns. The indices of $\balpha_0$ corresponding to other data points are set to zero. We used $p=20$ points for initialization and $\epsilon=10^{-6}$ for all the algorithms.

In all but the last experiment described in this section, SVMs were trained using a RBF (Gaussian) kernel
\begin{equation}\label{eq:kernel_exps}
k(\mathbf{x}_{1},\mathbf{x}_{2}) = \exp\left(-\frac{\|\mathbf{x}_{1}
- \mathbf{x}_{2}\|^{2}}{2 \sigma^{2}}\right) \, ,
\end{equation}
with scale parameter $\sigma^{2}$.  For the relatively small datasets \textbf{Pendigits} and \textbf{USPS}, parameter $\sigma^{2}$ was determined together with parameter $C$ of SVMs using $10$-fold cross-validation on the logarithmic grid $[2^{-15},2^{5}] \times [2^{-5},2^{15}]$, where the first collection of values corresponds to parameter $\sigma^{2}$ and the second to parameter $C$. 

For the large-scale datasets, $\sigma^{2}$ was determined using the default method employed for CVM in \cite{coreSVMs05tsang}, i.e. it was set to the average squared distance among training patterns. Parameter $C$ was determined on the logarithmic grid $[2^0,2^{12}]$ using a validation set consisting in a randomly computed 30\% of the training-set. 

We emphasize that the aim of this paper is not to determine optimal parameter values by fine-tuning each algorithm to seek for the best possible accuracy. Our aim is to compare the performance of the presented methods and analyze their behavior in a manner consistent with our theoretical analysis. Therefore, it is necessary to perform the experiments under the same conditions on a given dataset. That is to say, the optimization problem to be solved should be the same for each algorithm. For this reason, we deliberately avoided using different training parameters when comparing different methods.  Specifically, parameters $\sigma^{2}$ and $C$ were tuned using the CVM method and the obtained values were used for all the algorithms discussed in this paper (CVM, FW, MFW and SWAP methods). 

\paragraph{Caching} We also adopted the LRR caching strategy designed in \cite{LibCVM09} for the CVM to avoid the
computation of recently used kernel values.

\begin{small}
\begin{table}[!ht]
\centering
\begin{tabular}{|@{\,\,\,\,} p{2.2cm} | @{\,\,\,\,}p{1.65cm}@{\,\,}p{1.6cm}@{\,\,}p{1.0cm}@{\,\,}p{1.6cm}@{\,\,}p{1.65cm}@{\,\,}p{1.0cm}@{\,\,\,}|}
\hline
   &   &  &  &  &   & \\
{Dataset} & ${m}$ & $t$ & ${K}$ & $m_{{\max}}$ & $m_{{\min}}$ & ${n}$\\
   &   &  &  &  &   & \\
\hline
   &   &  &  &  &   & \\
{\textbf{USPS}}   & 7291  & 2007 & 10  & 2199 & 1098  & 256\\
{\textbf{Pendigits}}     & 7494 & 3498 & 10   & 1560 & 1438  & 16\\
{\textbf{Letter}}    & 15000    & 5000 & 26   & 1213 & 1081  & 16\\
{\textbf{Protein}}  & 17766  & 6621 & 3  & 13701 & 9568 & 357\\
{\textbf{Shuttle}}   & 43500    & 14500  & 7    & 40856 & 17 & 9\\
{\textbf{IJCNN}}     & 49990     & 91701 & 2   & 49990 & 49990  & 22\\
{\textbf{MNIST}}     & 60000    & 10000 & 10   & 13007 & 11263  & 780\\
{\textbf{USPS-Ext}}  & 266079    & 75383 & 2   & 266079 & 266079 & 676\\
{\textbf{KDD-10pc}}    & 395216   & 98805 & 5  & 390901 & 976 & 127\\
{\textbf{KDD-Full}}    & 4898431   & 311029 & 2   & 4898431 & 4898431 & 127\\
{\textbf{Reuters}}   & 7770 & 3299 & 2    & 7770 & 7770 & 8315\\
{\textbf{Adult a1a}}     & 1605 & 30956 & 2    &  1605 & 1605 & 123\\
{\textbf{Adult a2a}}     & 2265 & 30296 & 2    &  2265 & 2265 & 123\\
{\textbf{Adult a3a}}     & 3185 & 29376 &  2    &  3185 & 3185 & 123\\
{\textbf{Adult a4a}}     & 4781 & 27780   & 2    &  4781 & 4781 & 123\\
{\textbf{Adult a5a}}     & 6414 & 26147 & 2    &  6414 & 6414 & 123\\
{\textbf{Adult a6a}}     & 11220    & 21341   & 2    & 11220 & 11220  & 123\\
{\textbf{Adult a7a}}     & 16100    & 16461  & 2    & 16100 & 16100  & 123\\
{\textbf{Web w1a}}   & 2477 & 47272 & 2    & 2477 & 2477 & 300\\
{\textbf{Web w2a}}   & 3470 & 46279 & 2    & 3470 & 3470 & 300\\
{\textbf{Web w3a}}   & 4912 & 44837 & 2    & 4912 & 4912 & 300\\
{\textbf{Web w4a}}   & 7366 & 42383 & 2    & 7366 & 7366 & 300\\
{\textbf{Web w5a}}   & 9888 & 39861 & 2    & 9888 & 9888 & 300\\
{\textbf{Web w6a}}   & 17188    & 32561 & 2    & 17188 & 17188 & 300\\
{\textbf{Web w7a}}   & 24692    & 25057 & 2    & 24692 & 24692 & 300\\
{\textbf{Web w8a}}   & 49749    & 14951  & 2    & 49749 & 49749 & 300\\
   &   &  &  &  &   & \\
\hline
\end{tabular}
\caption{\small \label{TabDatasets} Features of the selected datasets.}
\end{table}
\end{small}

\paragraph{Assessed Algorithms, Notation and Statistics}

In this paper we have introduced two variants of the FW method: the SWAP, and the second-order SWAP. The acronyms used to denote these algorithms in the figures will be SW and SW-2o, respectively. We will compare these methods against the CVM algorithm \cite{coreSVMs05tsang}, the FW method and the MFW method. 

In the next sections we report test accuracies, training times and model sizes obtained on the classification problems of Table \ref{TabDatasets}. By test accuracy we intend the fraction of correctly classified test instances. Training time is the time in seconds required to obtain a model from the training set. When times differ by more than one order of magnitude among the different methods, we use a logarithmic scale to present figures. Model size is the number of training examples with non-zero weights at the end of the training process, that is, the number of support vectors in the model.

To obtain a more detailed comparison, we compute the speed-ups obtained by the Frank-Wolfe based algorithms with respect to the CVM method. The speed-up of the FW method with respect to CVM will be measured as $s_1=t_{0}/t_{1}$ where $t_{0}$ is the training time of the CVM algorithm and $t_1$ is the training time of the FW method, both measured in seconds. Similarly, the speed-up of the MFW, SWAP and SWAP-2o methods with respect to CVM is measured as $s_2=t_{0}/t_{2}$, $s_3=t_{0}/t_{3}$, $s_4=t_{0}/t_{4}$ respectively, where $t_2$ is the training time of the MFW method, $t_3$ is that of SWAP, and $t_4$ that of SWAP-2o. In addition, we quantify the difference in testing performance with respect to the CVM method. If we denote by $a_0$ the accuracy of CVM and by $a_1$ the accuracy of the FW method, the relative difference in accuracy incurred by FW will be quantified as $d_1=(a_{0} - a_{1}) / a_{0}$. Similarly, differences in testing performance corresponding to the methods MFW, SWAP and SWAP-2o  will be measured as $d_2=(a_{0} - a_{2}) / a_{0}$, $d_3=(a_{0} - a_{3}) / a_{0}$ and $d_4=(a_{0} - a_{4}) / a_{0}$, where $a_2$, $a_3$ and $a_4$ are the testing accuracies of the MFW, SWAP and SWAP-2o methods respectively. 

\paragraph{Computational Environment}
The experiments were conducted on a personal computer with a 2.66GHz Quad Core CPU and 4 GB of RAM, running 64bit GNU/Linux. The algorithms were implemented based on the C++ source code available at \cite{LibCVM09}.    

\subsection{Experiments on the Web Dataset Collection}
\label{webdatasetsec}

The Web Dataset Collection is a series of classification problems extracted from a webpage categorization dataset, first appeared in Platt's paper on Sequential Minimal Optimization for training SVMs \cite{platt99smo-seminal}. The number of training patterns in each instance of the collection grows approximately as $m_{i}=1.4^{i}m_{0}$, $i=1,\ldots,8$, where $m_0$ is the number of training patterns in the first dataset. This scheme makes the series amenable for studying performance and scalability of different training algorithms. 

Figures \ref{web-and-adult-accuracy}(a), \ref{web-and-adult-time}(a) and \ref{web-and-adult-sizes}(a) report test accuracies, training times and model sizes (number of support vectors) obtained in this collection. Note that times are depicted in a logarithmic scale. From Figure \ref{web-and-adult-accuracy}(a) and Figure \ref{web-and-adult-time}(a) we confirm that all the Frank-Wolfe based methods are slightly less accurate than CVM but exhibit running times that scale considerably better as the number of training patterns increases. Each of them is faster than CVM on all the $8$ datasets of the collection. 

Figure \ref{web-and-adult-time}(a) illustrates one of the main points of this paper: the theoretical advantages of the MFW method over the basic FW routine often do not correspond to an improvement in practical performance.
This collection of problems is actually an extreme case, in which MFW is always significantly slower than FW. In contrast the proposed methods are faster than MFW and competitive with the FW method. 

From Figure \ref{web-and-adult-time}(a), we can observe that the speed-ups of the FW method seem to increase monotonically as the number of training patterns increases, ranging from $12.6\times$ faster up to $\sim 106\times$ faster than CVM. Speed-ups corresponding to the MFW method are in contrast significantly more limited. The SWAP algorithm is clearly more competitive than MFW, with a speed-up of $\sim 250 \times$ on the largest dataset. 

Both MFW and SWAP endow the basic Frank-Wolfe procedure with away-steps, and both, in contrast to FW, offer a guarantee on the rate of convergence. However, the away steps implemented by SWAP and SWAP-2o work significantly better on this collection of datasets. SWAP-2o however does not perform better than SWAP in this series. We argue that 
standard away steps do not provide any significant advantage on this particular problem, as proved by MFW resulting to be the slowest algorithm.
Since SWAP-2o invests more time in finding a good away direction, finding a solution takes more time in comparison with the simpler SWAP, which seems to provide a better compromise between away and toward steps. 
 
As regards accuracy, MFW is slightly more accurate than SWAP, which in turn is slightly more accurate than FW most of the time. SWAP-2o very often outperforms the other three methods, approaching the accuracy of CVM. Note however that all the relative differences in testing accuracy are most of the time below $0.5\%$. Note finally that FW is the less accurate among the Frank-Wolfe based methods.
 
As concerns model sizes, note that the additional computational time incurred by the MFW and SWAP-2o methods is not compensated by an improved ability to find smaller models. Figure \ref{web-and-adult-sizes}(a) actually shows that the two faster methods, SWAP and FW, obtain most of the time smaller models. Finally, the size of the models found by CVM is significantly larger than that of the proposed methods. In addition, the percentage of training data used by this method to build the model does not seem to decrease significantly as the series progresses.

\subsection{Experiments on the Adult Dataset Collection}
\label{adultdatasetsec}

The Adult Dataset Collection is a series of problems derived from the 1994 US Census database. The goal is to predict whether an individual's income exceeded $50000$US$\$$/year, based on personal data. Like the \textbf{Web} datasets, this collection was designed with the purpose of analyzing the scalability of SVM methods. The number of training patterns grows approximately with the same rate, i.e. it increases by a factor of $\sim 1.4$ each time \cite{platt99smo-seminal}. 
 
Figures \ref{web-and-adult-accuracy}(b), \ref{web-and-adult-time}(b) and  \ref{web-and-adult-sizes}(b) depict accuracies, running times and model sizes (number of support vectors) obtained on this collection. Times are depicted in a logarithmic scale. These results confirm that all the Frank-Wolfe based methods tend to be faster than the CVM algorithm as the number of examples becomes larger. Figure \ref{web-and-adult-time}(b) shows that SWAP, MFW and SWAP-2o always run faster than CVM, reaching speed-ups of $27\times$, $20\times$ and $15\times$ respectively. Figure \ref{web-and-adult-time}(a) shows in addition that most of the times the Frank-Wolfe based methods achieve a testing performance greater or equal than CVM. 

Note that the speed-ups obtained by the FW method in this experiment are significantly smaller than those obtained in the \textbf{Web} collection. The largest speed-up achieved by the algorithm is $3.6\times$ on the sixth dataset of the collection. In contrast, the methods investigated in this paper, SWAP and SWAP-2o, always show speed-ups larger than $10\times$, running faster than FW in all cases. If we compute the median speed-up among all the datasets of this collection, the results for SWAP and SWAP-2o are $15.5\times$ and $20.5\times$ respectively. In contrast, the FW method achieves a median of just $1.45\times$. We conclude that the proposed methods are one order of magnitude faster than the basic FW method in this experiment. 

The previous remark suggests that away steps are very useful to speed up the algorithm towards an optimal face in this problem. We confirm this observation by examining the performance of the MFW method in this experiment. Figure \ref{web-and-adult-time}(b) shows that the MFW method is always faster than FW. This result contrasts with our previous experiment in which MFW was always slower than FW. We conclude that in this experiment all the algorithms incorporating away steps are significantly faster the algorithms which do not. Note that the proposed methods SWAP and SWAP-2o always run faster than MFW. 

As regards testing accuracy, the CVM is most of the time slightly less accurate than Frank-Wolfe methods in this experiment. SWAP always obtains an accuracy greater or equal than FW and in all but one case an accuracy greater or equal than MFW. SWAP-2o is  most of the time as accurate as MFW. We conclude that the additional running time incurred by the CVM and FW methods is not compensated by a better accuracy in this series of datasets. 

Figure \ref{web-and-adult-sizes}(b) shows that the model sizes obtained by the different methods are quite similar.

\subsection{Experiments on Other Medium-scale and Large-scale Datasets}
\label{singledatasetsec}

Results of Figures \ref{small-and-large-accuracies} to \ref{small-and-large-time} show the accuracies, times, speed-ups and model sizes obtained in the other datasets of Table \ref{TabDatasets}. A detailed description of these datasets can be found in \cite{IJPRAI11} or in the public repositories \cite{SVMLIB} and \cite{UCI2010}. 

To simplify the presentation and further analysis, datasets were separated into two groups: medium-scale and large-scale datasets. A dataset was included in the first group if the largest binary subproblem (see column $m_{{\max}}$ of Table \ref{TabDatasets}) to be addressed was lower than $15.000$ training examples, and was included in the second group otherwise. According to this criterion, datasets \textbf{Letter}, \textbf{Pendigits}, \textbf{USPS}, \textbf{Reuters} and \textbf{MNIST} were put together in the first group and datasets \textbf{Shuttle}, \textbf{IJCNN}, \textbf{USPS-Ext}, \textbf{KDD-10pc} and \textbf{KDD-Full} were included in the second group. Results for dataset \textbf{Protein} were presented/analyzed independently because accuracies and training times were significantly different from other results in the medium-scale group. Note again that most of the problems using in this experiment have been already used to compare CVM against other algorithms to train SVMs \cite{coreSVMs05tsang}. Times and model sizes are depicted in a logarithmic scale.

By examining Figure \ref{small-and-large-accuracies} we again observe a slight advantage of CVM in terms of testing accuracy. In addition, we confirm that the accuracy of the SWAP and SWAP-2o methods tends to be the closest to the best observed performance. The FW method is very often the least accurate among the Frank-Wolfe based algorithms. Note that if we compute the difference in accuracy with respect to CVM we always obtain results lower than $2\%$.

Results in Figure \ref{small-and-large-time} show that the FW, MFW, SWAP and SWAP-2o methods are most of the time faster than CVM. The speed-up achieved by these methods becomes more significant as the size of the training set grows, with peaks of around $100\times$ and $25\times$ on the largest datasets.  Differences among the Frank-Wolfe methods depend on the size of the problem. Among the medium-scale datasets all the methods achieve running times of the same order of magnitude. Speed-ups in the large-scale group are clearly more significant, with medians of $27.3\times$, $15.0\times$, $30.7\times$, $29.5\times$ for FW, MFW, SWAP and SWAP-2o respectively. 

The advantage of the methods explored in this paper against standard FW routines can be summarized as follows. The FW and MFW methods can sometimes be faster than SWAP and SWAP-2o, but in that case the advantage is very tight. Often, however, our methods can improve on FW and MFW with more significant speed-ups. MFW in particular tends to be significantly outperformed in the cases where the FW works better. In those cases the performance of our methods tends to be competitive or better. On medium-scale problems all the methods are evenly matched in performance, with a slight advantage for SWAP-2o and MFW. In the large-scale group, SWAP and SWAP-2o tend to outperform FW and  MFW more significantly.        
 
Results on the \textbf{Protein} dataset deserve a particular comment. This is a dataset of around $18.000$ examples distributed into $3$ classes, which leads to binary subproblems of around $10000$ examples. According to this size, the problem should be included in the group of medium-scale datasets on which we have seen that the Frank-Wolfe algorithms obtain fairly similar and small speed-ups. On the \textbf{Protein} problem however the methods obtain peculiar results. The FW method achieves here a speed-up of $20.8\times$ against CVM. However the standard  MFW runs here $123.5\times$ faster than the CVM and $5.95\times$ faster than FW. This suggests that in this problem, away steps significantly help the algorithm to find the solution to the SVM problem more quickly. Since our methods tend to be better when aways steps work, we should observe important improvements on the CVM using the proposed methods. Indeed, the respective speed-ups for the SWAP and SWAP-2o methods on this datasets are $157.3\times$ and $358.0\times$. This means that SWAP runs $17.25\times$ faster than FW and $1.27\times$ faster than MFW. SWAP-2o runs $7.58\times$ faster than FW and $2.90\times$ faster than MFW.

Note finally that Figure \ref{small-and-large-sizes} suggests that there are no significant differences among the sizes of the models built by the different methods.

\subsection{Statistical Tests}
\label{subsec:significance}

\setlength\floatsep{5pt} \setlength\intextsep{6pt}
\setlength{\abovecaptionskip}{1pt}
\setlength{\belowcaptionskip}{1pt} \setlength{\tabcolsep}{4pt}

In this section, we perform some statistical tests to assess the significance of the experimental results reported in this paper. To this end we adopt the guidelines suggested in \cite{Demsar06Statistical}. We first conduct a multiple test to determine whether the hypothesis that all the algorithms perform equally can be rejected or not. Then, we conduct separate binary tests to compare the performances of each algorithm against each other. For the binary tests we adopt the \emph{Wilcoxon Signed-Ranks Test} method. For the multiple test we use the non-parametric \emph{Friedman Test}. In \cite{Demsar06Statistical}, Demsar recommends these tests as safe alternatives to the classical parametric t-tests to compare classifiers over multiple datasets.   

From the multiple test, we conclude that there is indeed a statistically significant difference among the running times and accuracies of all the algorithms ($p$-values were lower than $0.001$ in both cases). 

We then conduct a binary test on each pair of algorithms. The main hypothesis of this paper is that the SWAP method outperforms the MFW and FW methods in terms of training time without significant differences in terms of predictive accuracy. In contrast, we claim that no significant differences between the MFW and FW methods are observed in practice (although MFW seems to be slightly more accurate).  We have also observed that the SWAP method significantly outperforms CVM, sometimes at the expense of a little test accuracy. Finally, we have observed that the SWAP-2o usually exhibits larger running times than the SWAP method but outperforms the other FW based methods in terms of predictive power. As regards the comparison of the proposed methods, there is no apparent advantage in terms of running time of one against the other. We thus conduct a two-tailed test for the running times but adopt a one-tailed test for testing accuracy. Considering all the observations above, our design for the binary tests is that of Table \ref{table:designTests}.

\begin{table}[!ht]
\centering
\begin{small}
\vspace{0.2cm}
\begin{tabular}{|p{3.85cm}|p{1.35cm}|p{4.4cm}|p{1.35cm}|}
  \hline
\cellcolor{RoyalBlue} Time SWAP vs. FW &  p-value & \cellcolor{RoyalBlue}  Accuracy SWAP vs. FW & p-value\\
     \hline
   $H_0: \, \mbox{Both equally fast} $ & $0.03757$ & \cellcolor{Cyan} $H_0: \, \mbox{Both equally accurate}$ & $0.2389$ \\
\cellcolor{Cyan}   $H_1: \, \mbox{SWAP faster} \,$ &   & $H_1: \, \mbox{Different accuracies}\,$ & \\
  \hline
  \hline
\cellcolor{RoyalBlue} Time SWAP vs. MFW & p-value & \cellcolor{RoyalBlue} Accuracy SWAP vs. MFW & p-value\\
     \hline
   $H_0: \, \mbox{Both equally fast} $ & $1.526e$-$05$ & \cellcolor{Cyan} $H_0: \, \mbox{Both equally accurate}$ & $0.1019$\\
 \cellcolor{Cyan}  $H_1: \, \mbox{SWAP faster} \,$ &  & $H_1: \, \mbox{Different accuracies}\,$ &\\
  \hline
  \hline
  
\cellcolor{RoyalBlue} Time SWAP vs. CVM & p-value & \cellcolor{RoyalBlue} Accuracy SWAP vs. CVM & p-value\\
     \hline
   $H_0: \, \mbox{Both equally fast} $ & $5.528e$-$06$ &$H_0: \, \mbox{Both equally accurate}$ & $1.873e$-$04$ \\
 \cellcolor{Cyan}  $H_1: \, \mbox{SWAP faster} \,$ &  & \cellcolor{Cyan} $H_1: \, \mbox{CVM more accurate}\,$ &\\
  \hline
  \hline
  
\cellcolor{RoyalBlue} Time FW vs. MFW & p-value & \cellcolor{RoyalBlue} Accuracy FW vs. MFW & p-value\\
     \hline
 \cellcolor{Cyan}  $H_0: \, \mbox{Both equally fast} $ & $0.6893$  & \cellcolor{Cyan} $H_0: \, \mbox{Both equally accurate}$& $0.1118$ \\
   $H_1: \, \mbox{Different times} \,$ &  & $H_1: \, \mbox{Different accuracies}\,$& \\
  \hline
  \hline
  
\cellcolor{RoyalBlue} Time SWAP-2o vs. FW & p-value & \cellcolor{RoyalBlue} Accuracy SWAP-2o vs. FW & p-value\\
     \hline
 \cellcolor{Cyan}   $H_0: \, \mbox{Both equally fast} $ & $0.3403$ & $H_0: \, \mbox{Both equally accurate}$ & $0.01071$\\
   $H_1: \, \mbox{Different times} \,$ &  &  \cellcolor{Cyan} $H_1: \, \mbox{SWAP-2o more accurate}\,$ &\\
  \hline
  \hline
  
\cellcolor{RoyalBlue} Time SWAP-2o vs. MFW & p-value & \cellcolor{RoyalBlue} Accuracy SWAP-2o vs. MFW & p-value\\
     \hline
   $H_0: \, \mbox{Both equally fast} $ & $1.087e$-$04$ &$H_0: \, \mbox{Both equally accurate}$ & $0.01634$\\
 \cellcolor{Cyan}   $H_1: \, \mbox{SWAP-2o faster} \,$ &  &  \cellcolor{Cyan} $H_1: \, \mbox{SWAP-2o more accurate}\,$ &\\
  \hline
  \hline
  
\cellcolor{RoyalBlue} Time SWAP-2o vs. CVM & p-value & \cellcolor{RoyalBlue} Accuracy SWAP-2o vs. CVM & p-value\\
     \hline
   $H_0: \, \mbox{Both equally fast} $ & $4.47e$-$08$ & $H_0: \, \mbox{Both equally accurate}$ & $2.25e$-$04$\\
 \cellcolor{Cyan}   $H_1: \, \mbox{SWAP-2o faster} \,$ &  &  \cellcolor{Cyan} $H_1: \, \mbox{CVM more accurate}\,$ &\\
  \hline
  \hline
  
\cellcolor{RoyalBlue} Time SWAP-2o vs. SWAP & p-value & \cellcolor{RoyalBlue} Accuracy SWAP-2o vs. SWAP & p-value\\
     \hline
 \cellcolor{Cyan}   $H_0: \, \mbox{Both equally fast} $ & $0.1901$ & $H_0: \, \mbox{Both equally accurate}$& $1.418e$-$03$\\
   $H_1: \, \mbox{SWAP faster} \,$ &  &  \cellcolor{Cyan} $H_1: \, \mbox{SWAP-2o more accurate}\,$& \\
  \hline

  \end{tabular} \vspace{0.2cm} 
  \caption{\label{table:designTests} Null hypotheses, alternative hypotheses and p-values for the binary statistical tests. The conclusion of the test adopting a significance level of $5\%$ is highlighted in blue.}
\end{small}
\end{table}

In Table \ref{table:designTests}, we also report the p-values corresponding to each test\footnote{In some cases we implement one-sided alternative hypotheses, and in others two-sided tests. If a two-sided test is preferred to a one-sided alternative, it's enough to double the p-value reported here. Vice versa, if a one-sided test is preferred to a two-sided test, it's enough to halve the p-value reported here.}. For reproducibility concerns, p-values were computed using the statistical software R \cite{R}. For the Wilcoxon Signed-Ranks Test, the exact p-values were preferred to the asymptotic ones. The Pratt method to handle ties is employed by default. In the case of the Friedman test, the Iman and Davenport's correction was adopted, as suggested in \cite{Demsar06Statistical}. 

We now point out some of the conclusions which can be obtained from Table \ref{table:designTests}. At commonly used significance levels ($10\%$, $5\%$, $1\%$ or lower), the hypothesis that FW and MFW are equally fast cannot be rejected. Adopting a significance level of $5\%$, the running times of SWAP method are found to be significantly different from those of all the baseline methods (FW, MFW and CVM), so the null hypothesis is rejected in favor of the alternative hypothesis than the SWAP method is faster. At the same significance level, or better, the hypotheses than the SWAP-2o method is as fast as MFW or CVM are rejected in favor of the conclusion that the SWAP-2o method is faster. Empirical data is however insufficient to reject the hypothesis that the SWAP-2o method is as fast as the FW or the SWAP methods. As regards the testing accuracy, FW, MFW and SWAP are found to be equally as accurate at reasonable significance levels ($10\%$, $5\%$, $1\%$ or lower). In contrast, the hypothesis that the SWAP-2o method has similar accuracies to FW, MFW and SWAP is rejected in favor of the conclusion that SWAP-2o is more accurate.

\begin{figure}[hp]
\centering
\begin{tabular}{cc}
\textbf{(a)} \textbf{Web} collection & \textbf{(b)} \textbf{Adult} collection \\
\includegraphics[width=0.48\textwidth, height=0.92\textheight]{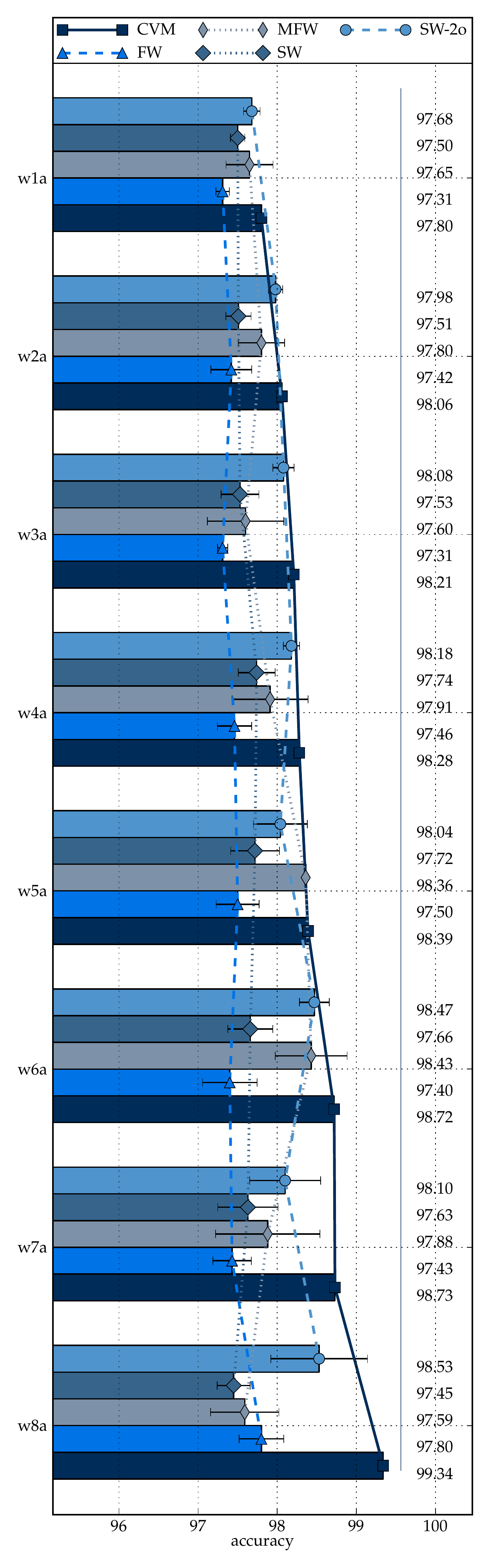} &
\includegraphics[width=0.48\textwidth, height=0.92\textheight]{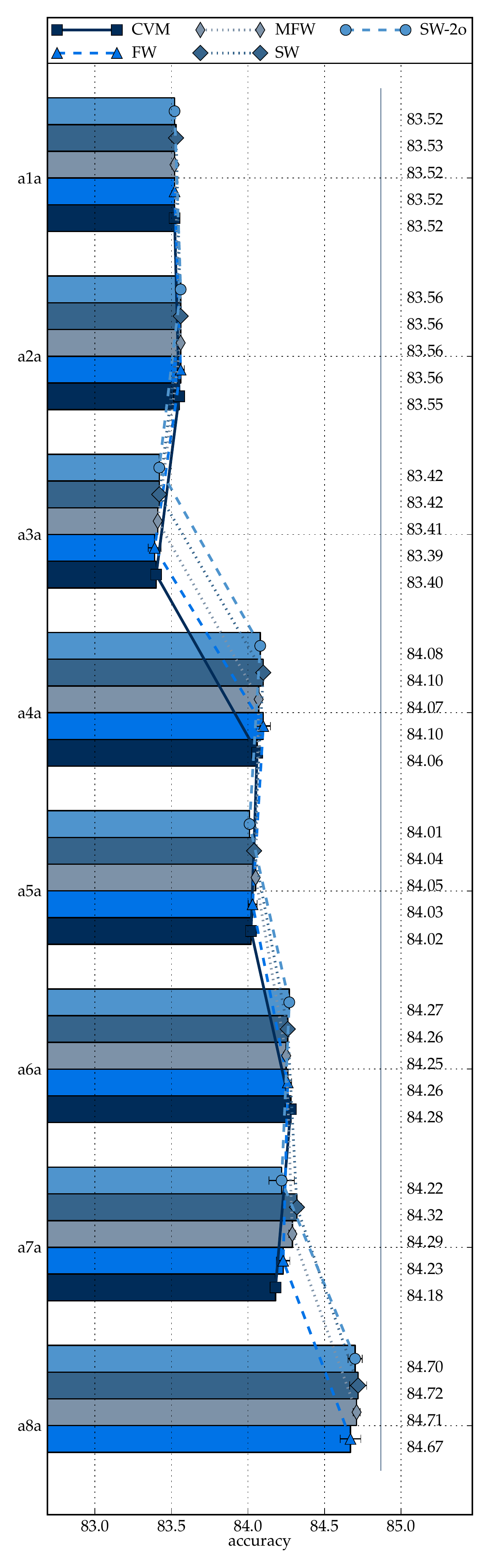}
\end{tabular}
\caption{\small Testing accuracies in the \textbf{Web} and \textbf{Adult} collections. \label{web-and-adult-accuracy}}
\end{figure}

\begin{figure}[hp]
\centering
\begin{tabular}{cc}
\textbf{(a)} \textbf{Web} collection & \textbf{(b)} \textbf{Adult} collection \\
\includegraphics[width=0.48\textwidth, height=0.92\textheight]{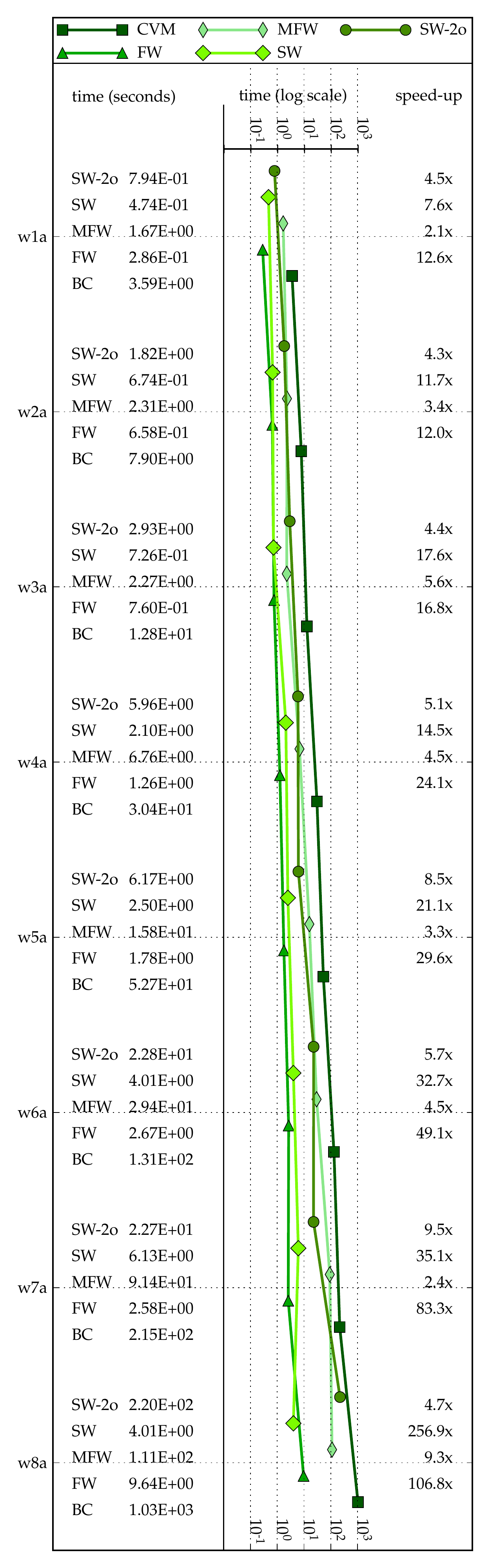} &
\includegraphics[width=0.48\textwidth, height=0.92\textheight]{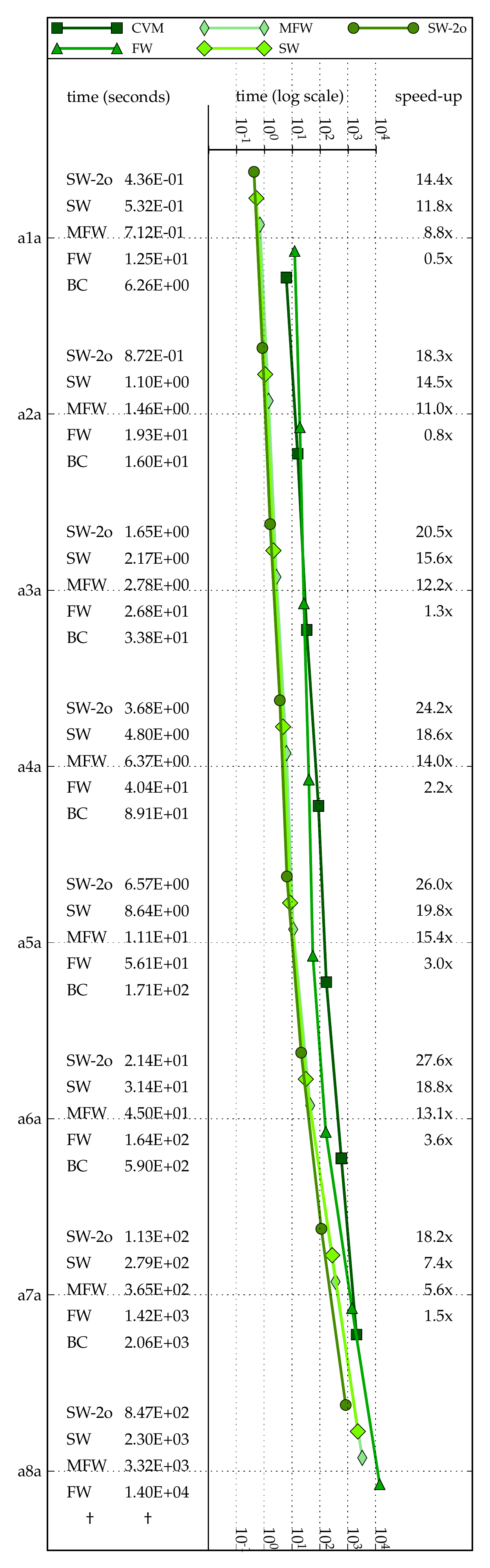}
\end{tabular}
\caption{\small Running times in the \textbf{Web} and \textbf{Adult} collections. The column on the left shows speed-ups with respect to CVM.  \label{web-and-adult-time}}
\end{figure}

\begin{figure}[hp]
\centering
\begin{tabular}{cc}
\textbf{(a)} \textbf{Web} collection & \textbf{(b)} \textbf{Adult} collection \\
\includegraphics[width=0.48\textwidth, height=0.92\textheight]{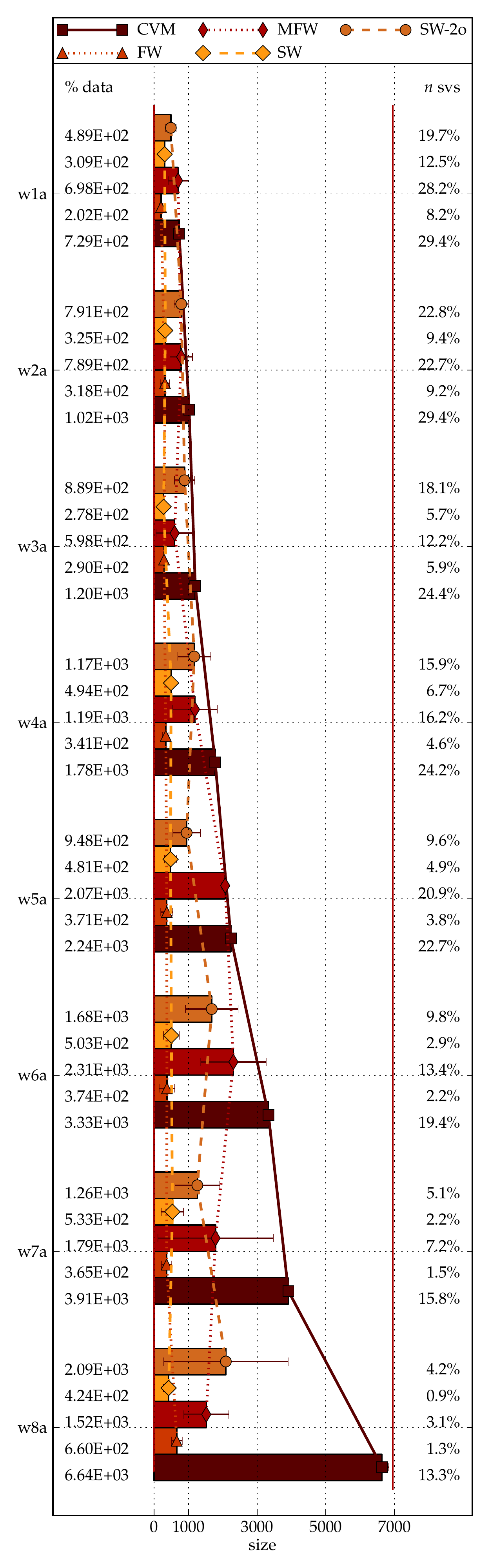} &
\includegraphics[width=0.48\textwidth, height=0.92\textheight]{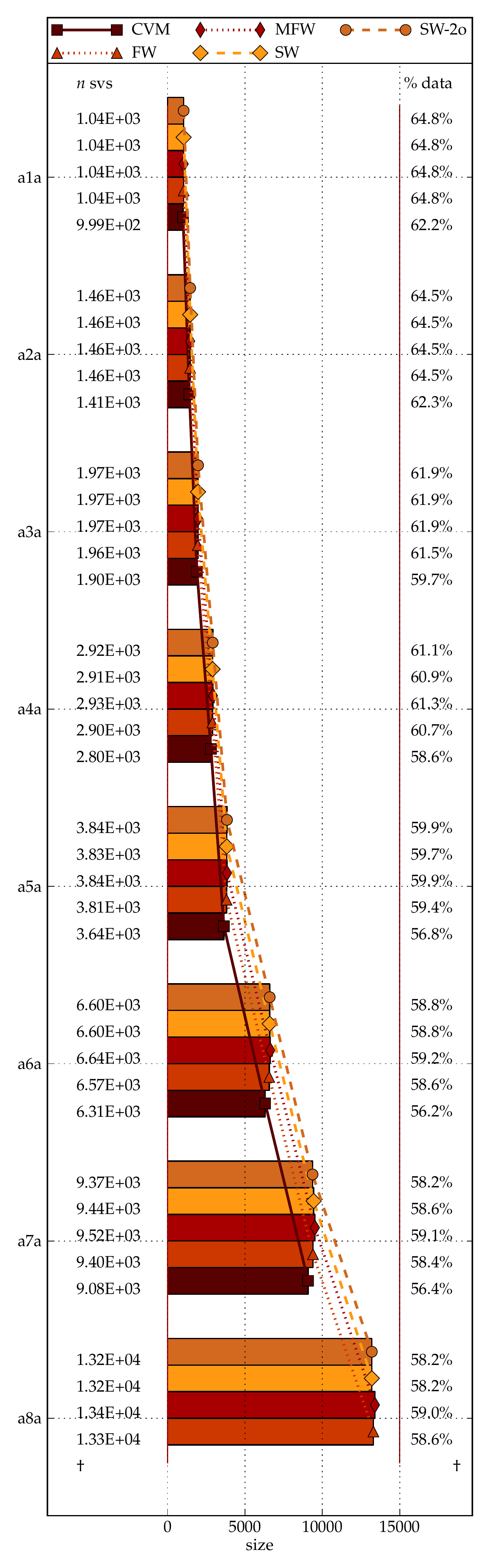}
\end{tabular}
\caption{\small Model sizes in the \textbf{Web} and \textbf{Adult} collections. The column on the left shows the percentage of the total number of examples. 
\label{web-and-adult-sizes}} 
\end{figure}

\begin{figure}[hp]
\centering
\begin{tabular}{cc}
\textbf{(a)} \textbf{Medium-scale} datasets & \textbf{(b)}\textbf{Large-scale} datasets \\
\includegraphics[width=0.48\textwidth,height=0.63\textheight]{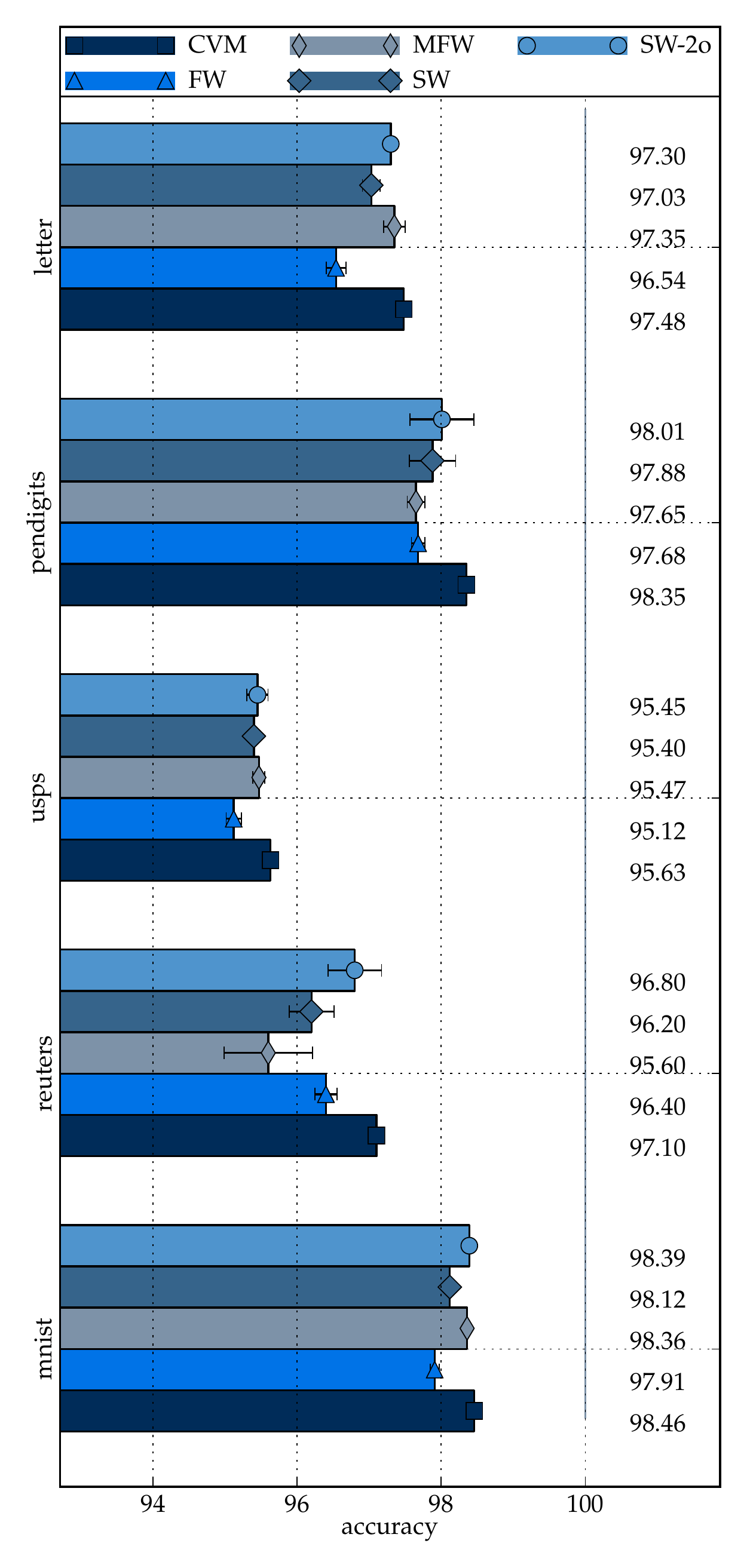} &
\includegraphics[width=0.48\textwidth,height=0.63\textheight]{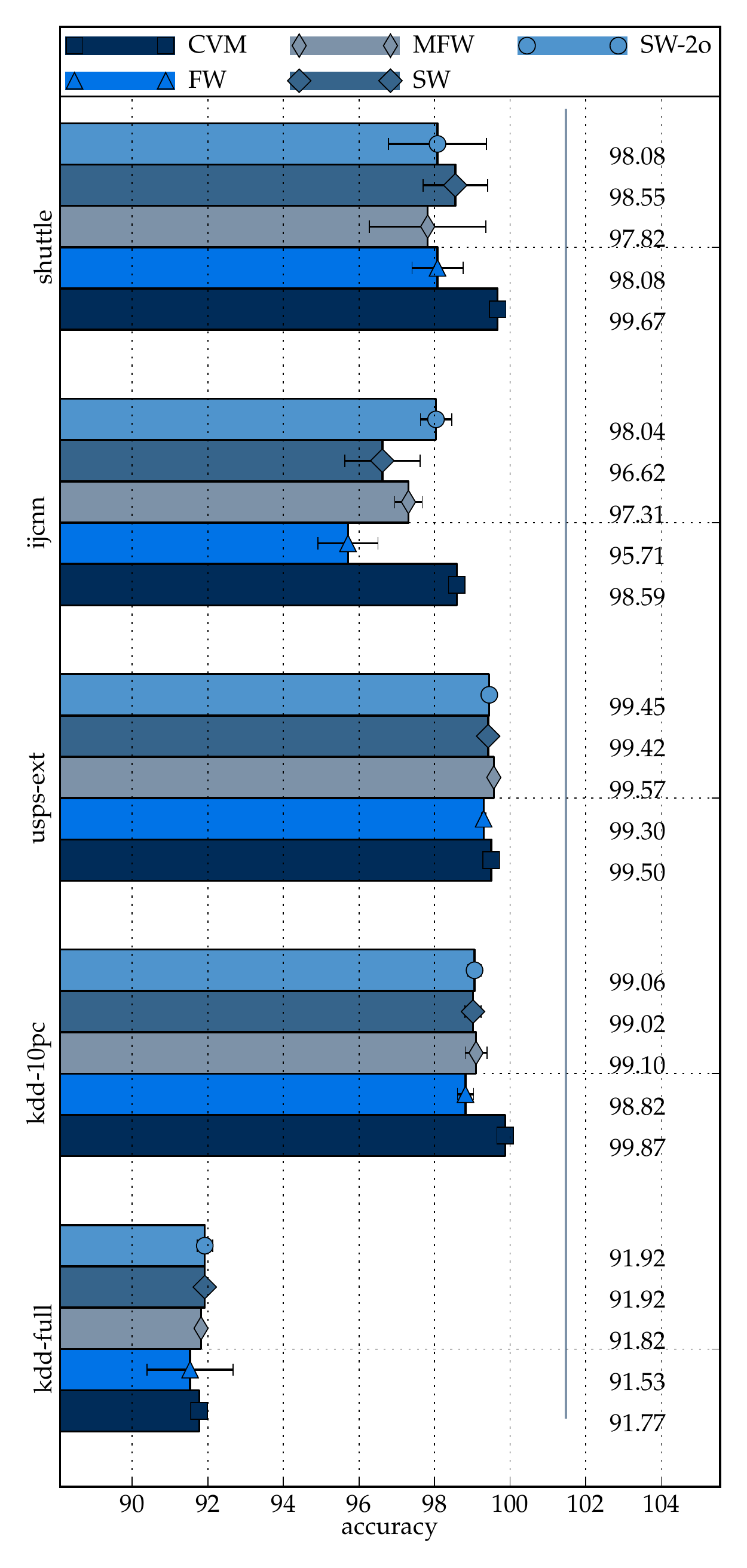}
\end{tabular}
\caption{\small On the left, testing accuracies in the medium-scale datasets: \textbf{Letter}, \textbf{Pendigits}, \textbf{USPS}, \textbf{Reuters}, \textbf{MNIST}. On the right, testing accuracies in the large dataset collection: \textbf{Shuttle}, \textbf{IJCNN},x \textbf{USPS-Ext}, \textbf{KDD-10pc}, \textbf{KDD-Full}. \label{small-and-large-accuracies}} 
\end{figure}
\begin{figure}[hp]
\centering
\begin{tabular}{cc}
\includegraphics[width=0.48\textwidth,height=0.18\textheight]{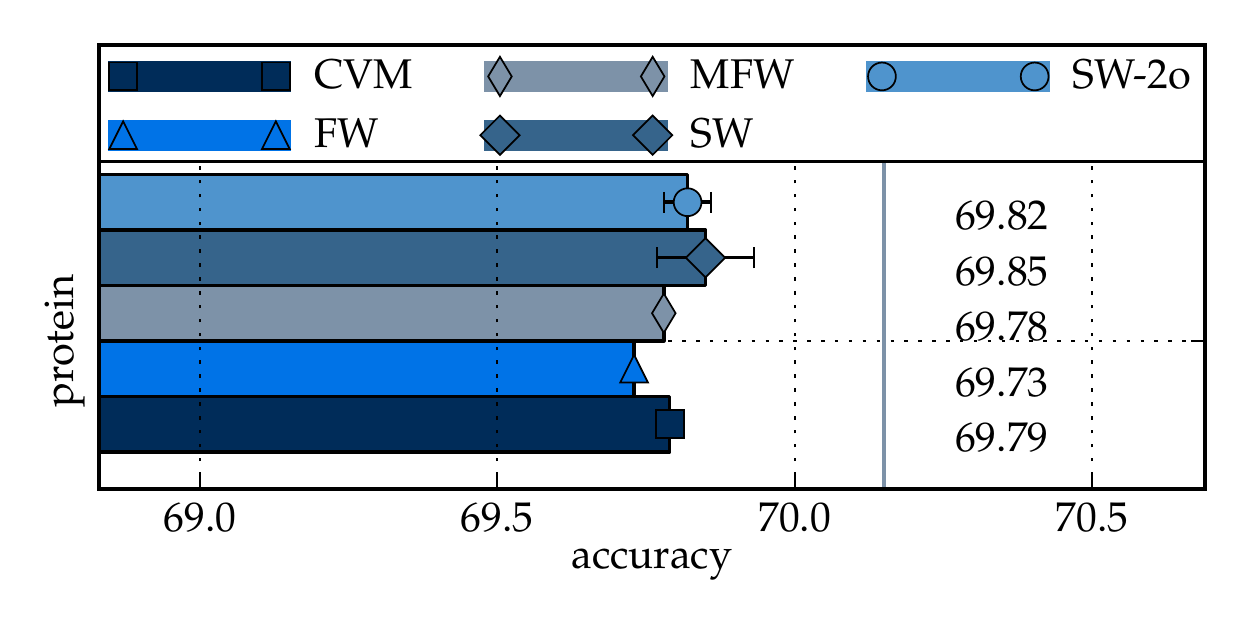} &
\includegraphics[width=0.48\textwidth,height=0.18\textheight]{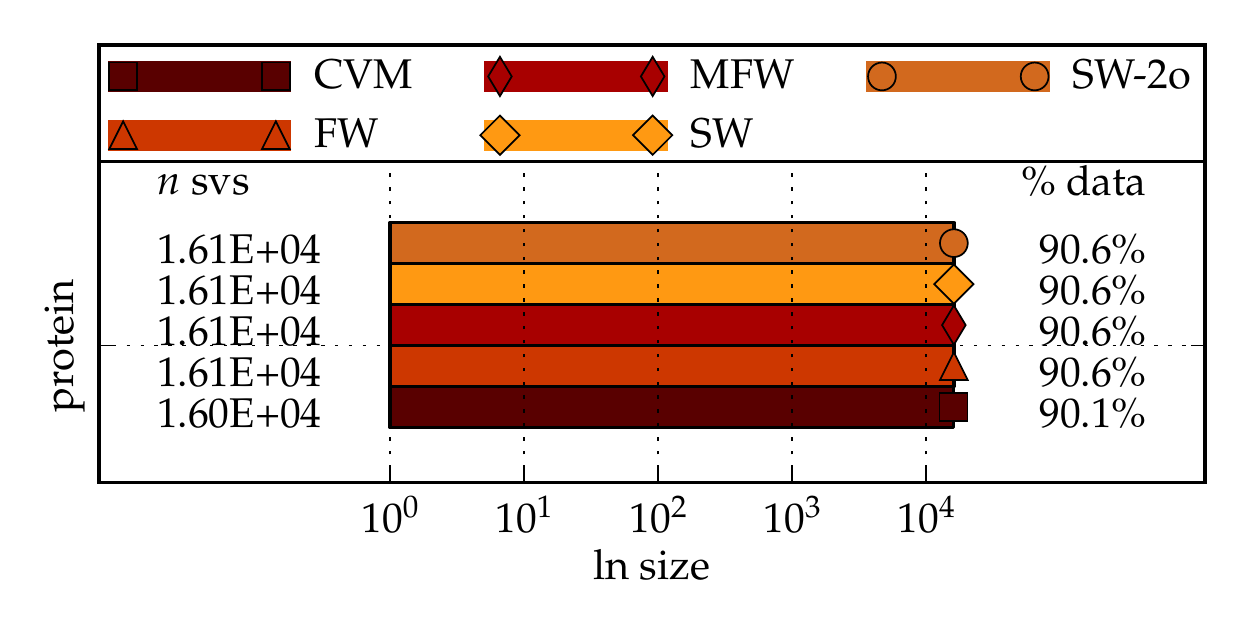}
\end{tabular}
\caption{\small On the left, testing accuracies in the \textbf{Protein} dataset. On the right, model sizes in the \textbf{Protein} dataset. \label{protein-accuracy-and-time}} 
\end{figure}

\begin{figure}[hp]
\centering
\begin{tabular}{cc}
\textbf{(a)} \textbf{Medium-scale} datasets & \textbf{(b)}\textbf{Large-scale} datasets \\
\includegraphics[width=0.48\textwidth,height=0.65\textheight]{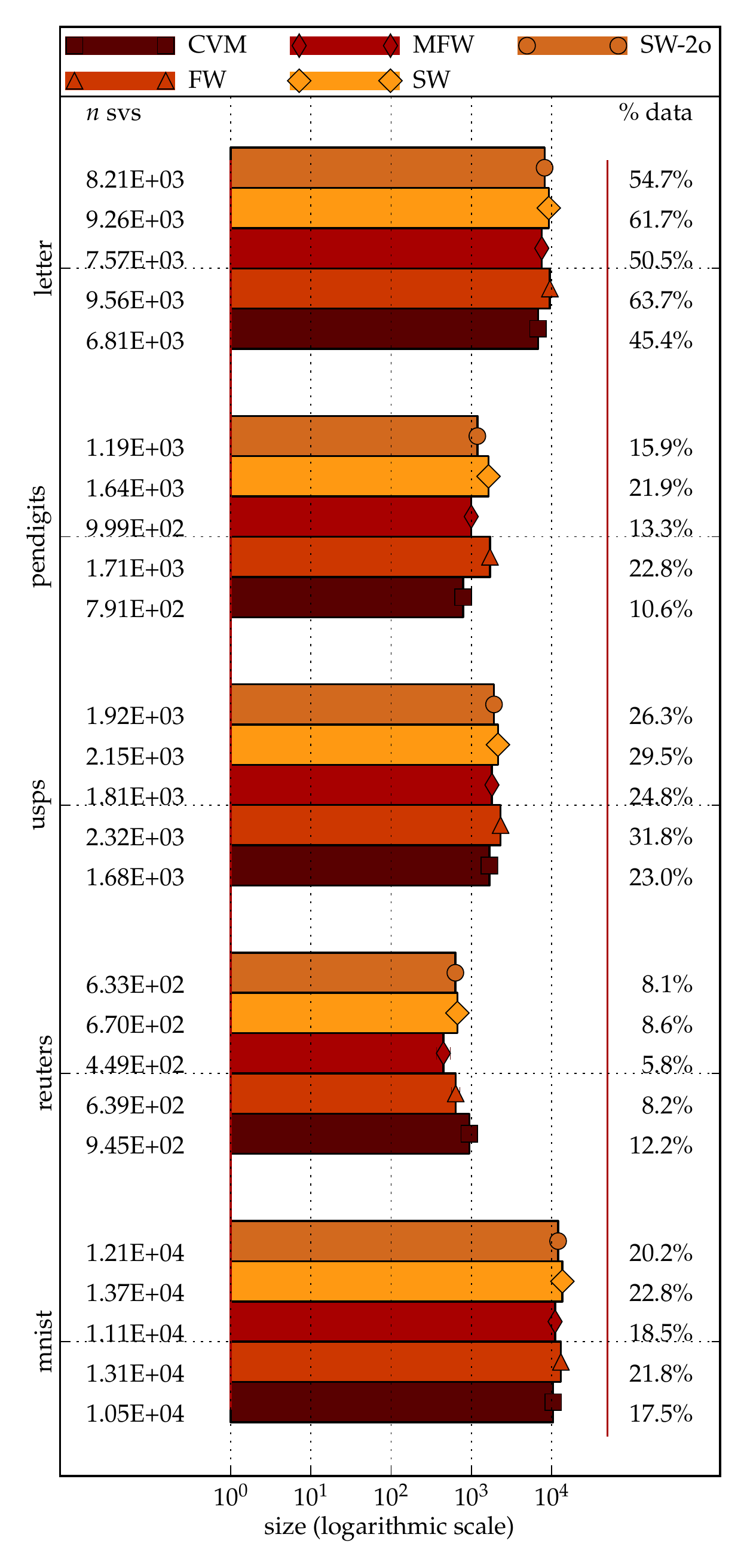} &
\includegraphics[width=0.48\textwidth,height=0.65\textheight]{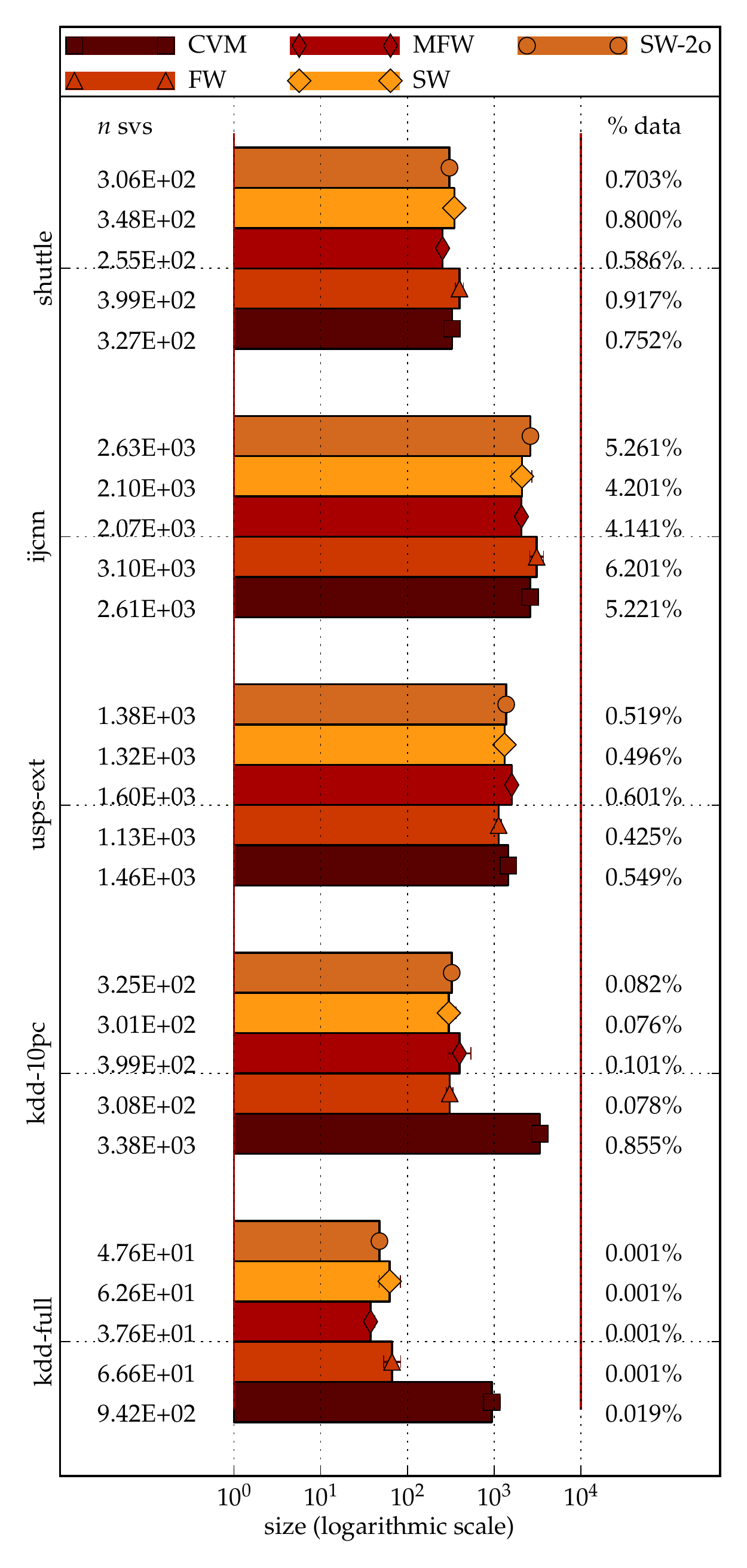}
\end{tabular}
\caption{\small On the left, model sizes in the medium-scale datasets \textbf{Letter}, \textbf{Pendigits}, \textbf{USPS}, \textbf{Reuters}, \textbf{MNIST}. On the right, model sizes in the large dataset collection: \textbf{Shuttle}, \textbf{IJCNN}, \textbf{USPS-Ext}, \textbf{KDD-10pc}, \textbf{KDD-Full}.  \label{small-and-large-sizes}}
\end{figure}

\begin{figure}[hp]
\centering
\begin{tabular}{cc}
\includegraphics[width=0.48\textwidth,height=0.18\textheight]{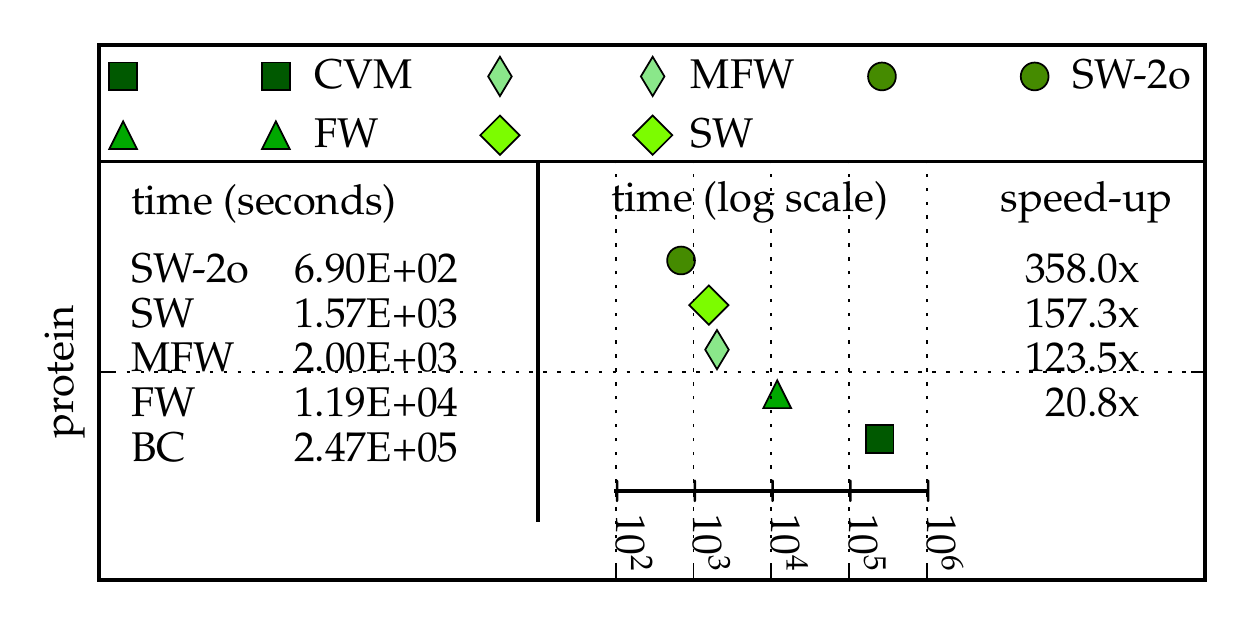}  & 
\end{tabular}
\caption{\small Running times in the \textbf{Protein} dataset. \label{protein-sizes}}
\end{figure}

\begin{figure}[hp]
\centering
\begin{tabular}{cc}
\textbf{(a)} \textbf{Medium-scale} datasets & \textbf{(b)}\textbf{Large-scale} datasets \\
\includegraphics[width=0.48\textwidth]{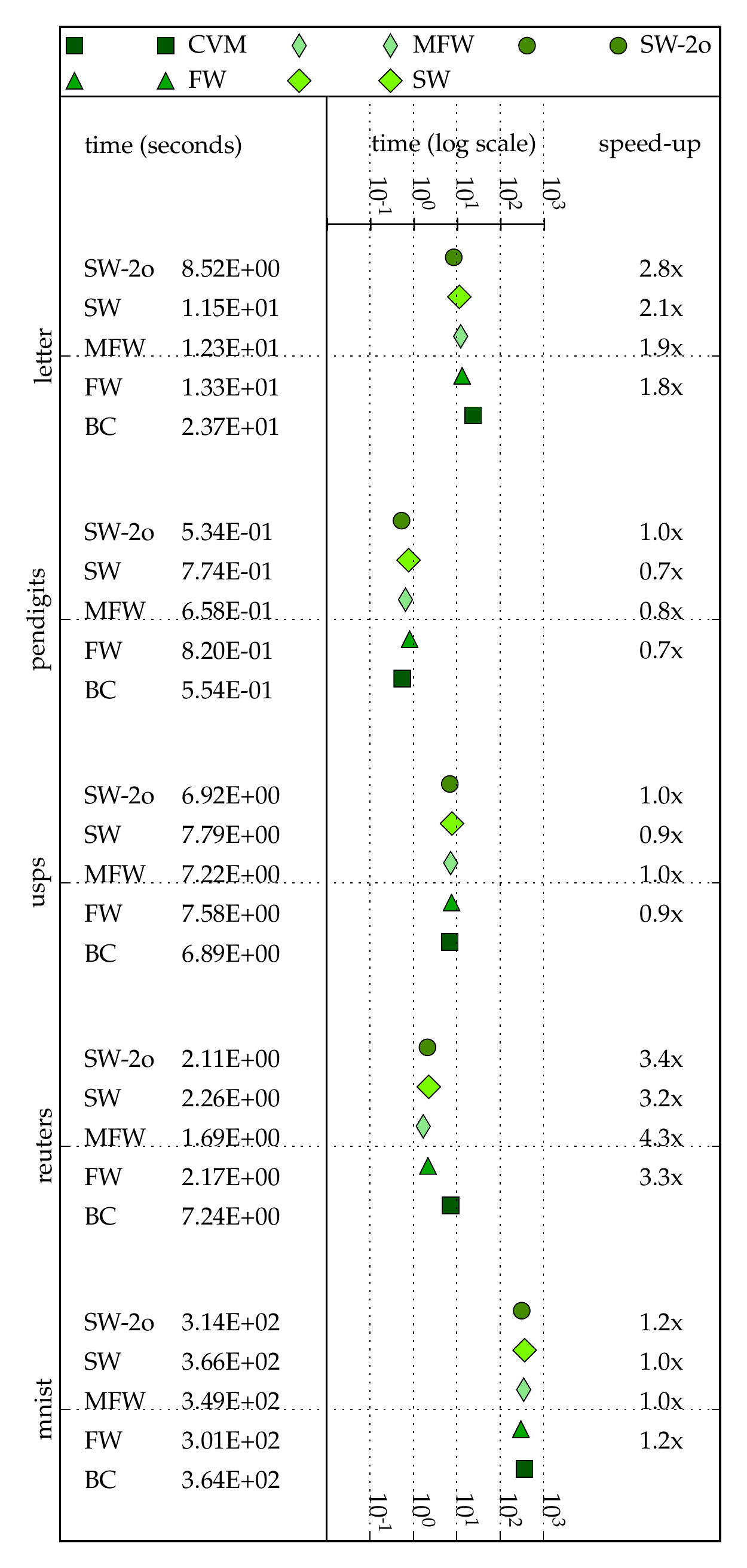} &
\includegraphics[width=0.48\textwidth]{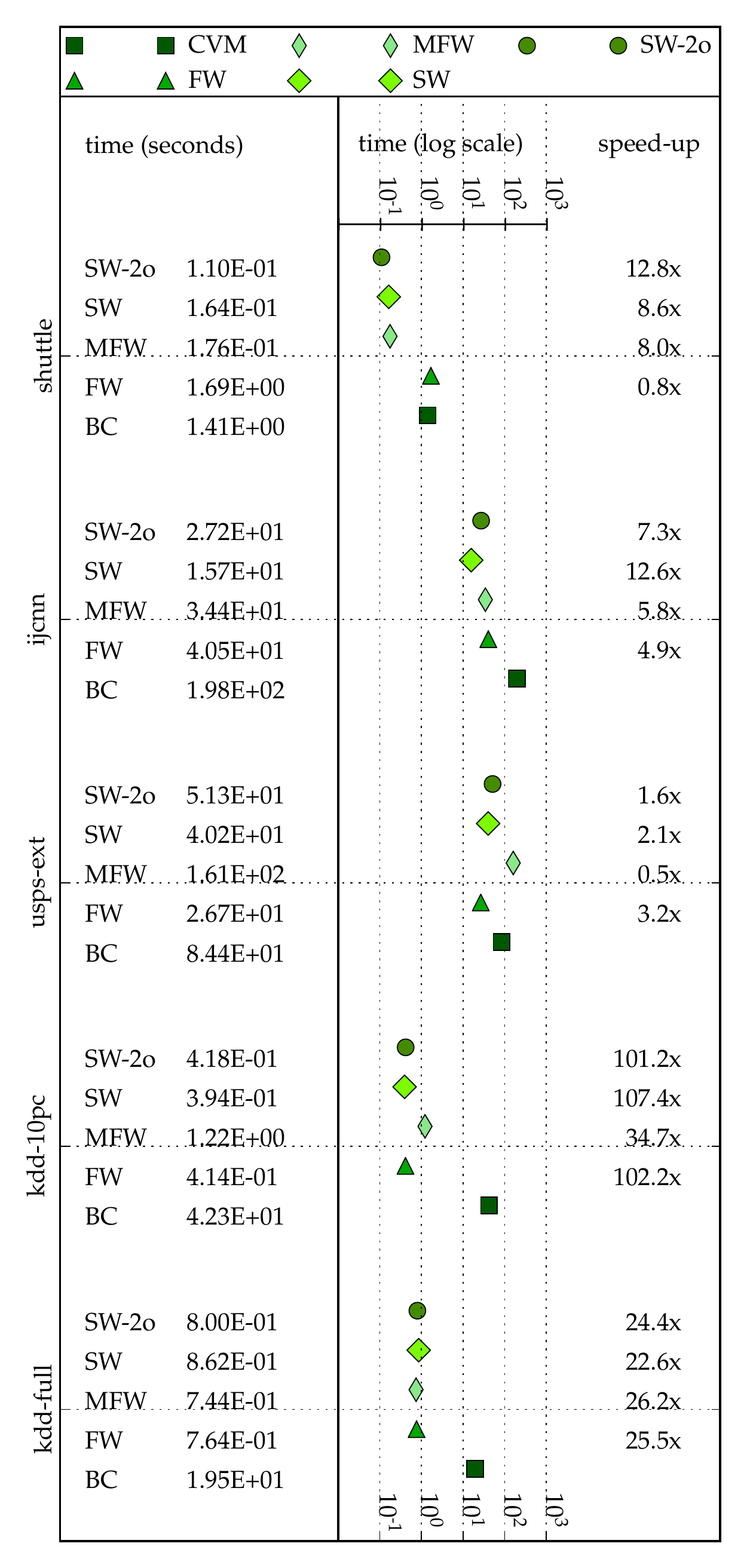}
\end{tabular}
\caption{\small On the left, running times in the medium-scale datasets: \textbf{Letter}, \textbf{Pendigits}, \textbf{USPS}, \textbf{Reuters}, \textbf{MNIST}. On the right, running times in the large dataset collection: \textbf{Shuttle}, \textbf{IJCNN}, \textbf{USPS-Ext}, \textbf{KDD-10pc}, \textbf{KDD-Full}. \label{small-and-large-time}} 
\end{figure}

\subsection{Experiments with Non-Normalized Kernels}

Solving a classification problem using SVMs requires to select a kernel function. Since the optimal kernel for a given application cannot be specified \emph{a priori}, the capability of a training method to work with any (or the widest possible) family of kernels is an important feature.

In order to illustrate that the proposed methods can obtain effective models even if the kernel does not satisfy the conditions required by CVM, we conduct experiments using the homogenous second order polynomial kernel $k(\bx_i,\bx_j)=(\gamma \bx_i^{T}\bx_j)^{2}$. Here, parameter $\gamma$ is estimated as the inverse of the average squared distance among training patterns \cite{coreSVMs05tsang}. 

Figures \ref{poly-web-accuracy-and-time} and \ref{poly-rest} summarize the results obtained in some of the datasets used in this section. We can see that both test accuracies and training times are comparable to those obtained using the Gaussian kernel. It should be noted that the CVM algorithm cannot be used to train an SVM using the kernel selected for this experiment, thus we only incorporate the Frank-Wolfe based methods in the figures. These results demonstrate the capability of our methods to be used with kernels other than those satisfying the normalization condition imposed by CVM.

\begin{figure}[hp]
\centering
\begin{tabular}{cc}
\includegraphics[width=0.48\textwidth]{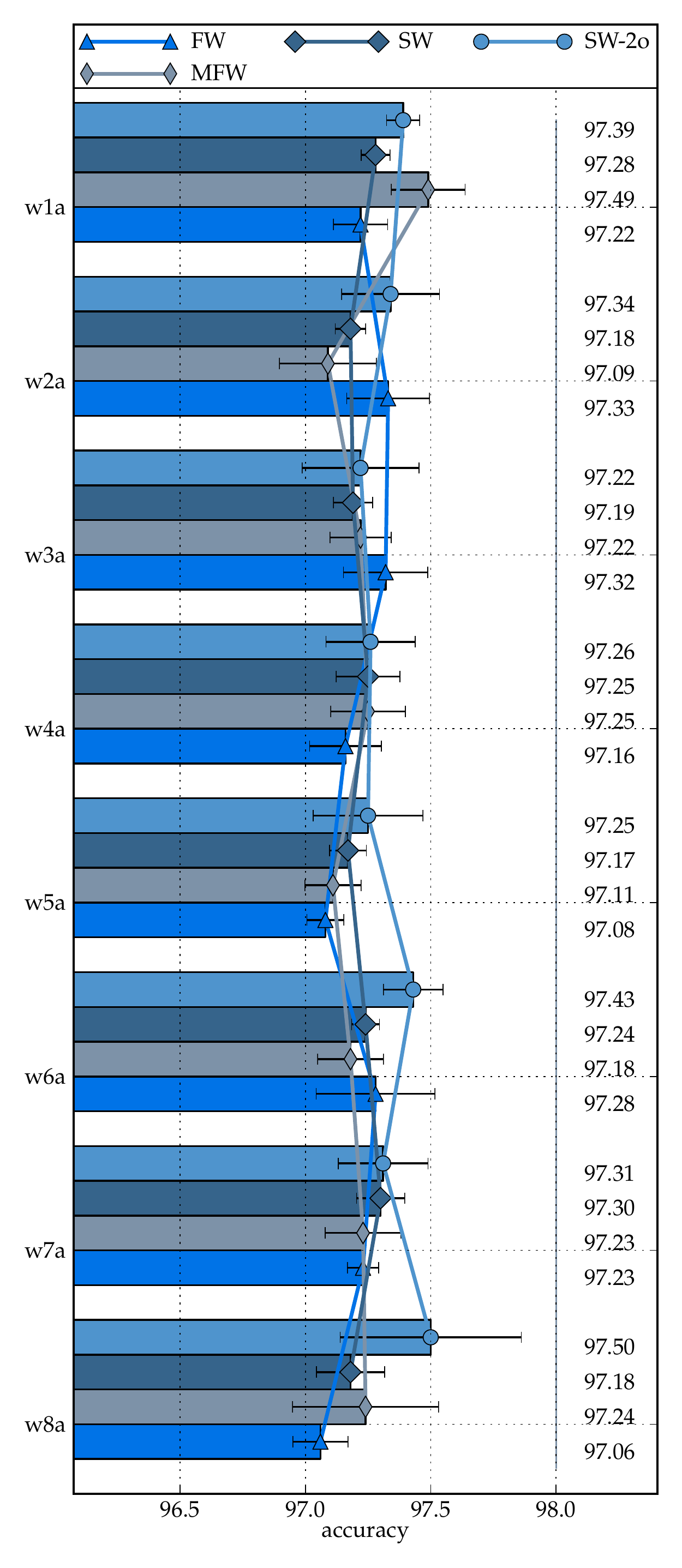} &
\includegraphics[width=0.48\textwidth]{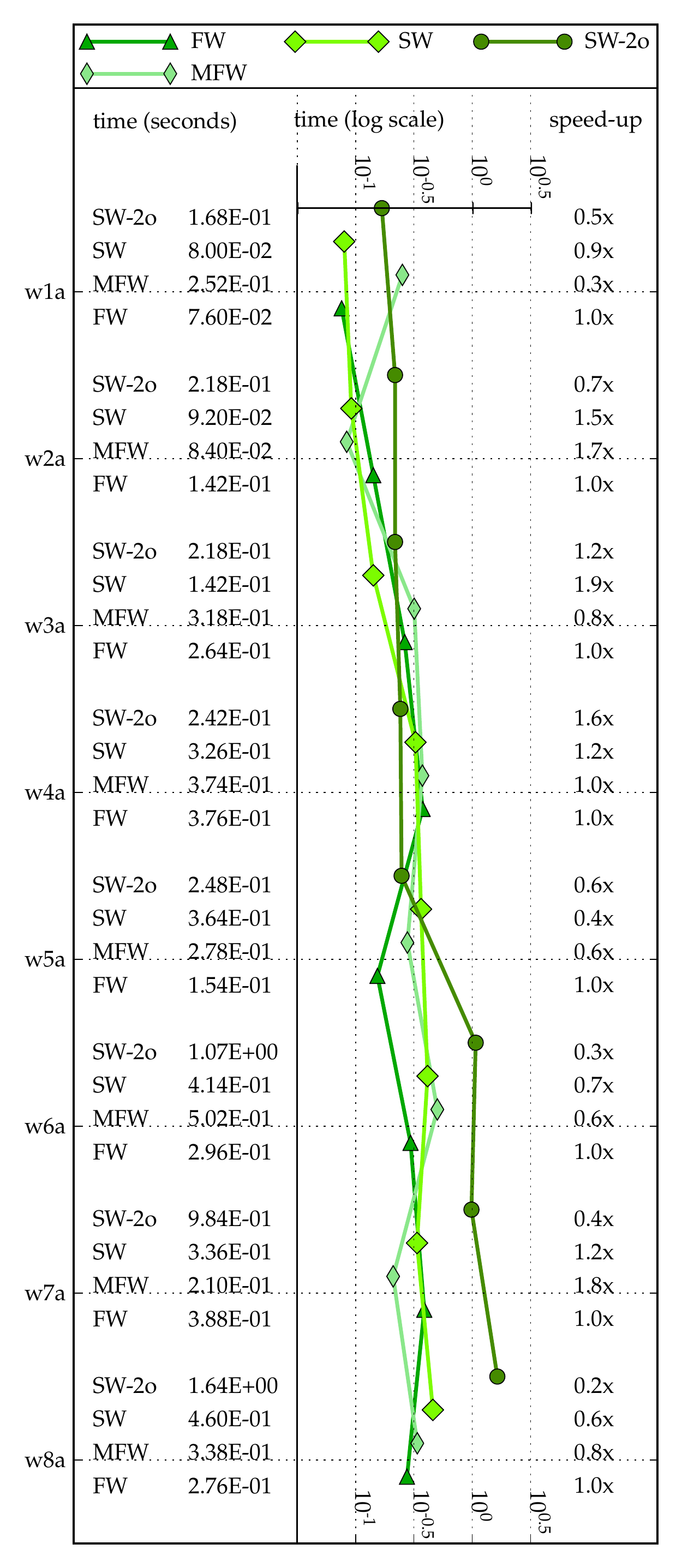}
\end{tabular}
\caption{\small On the left, testing accuracies obtained with the polynomial kernel in the datasets of the \textbf{Web} collection, \emph{w1a}, \emph{w2a}, \emph{w3a}, \emph{w4a}, \emph{w5a}, \emph{w6a}, \emph{w7a} and \emph{w8a}. On the right, the corresponding running times. \label{poly-web-accuracy-and-time}}
\end{figure}

\begin{figure}

	\begin{minipage}[c][0.95\height]{%
	   0.5\textwidth}
	   \centering%
	   \includegraphics[width=\textwidth]{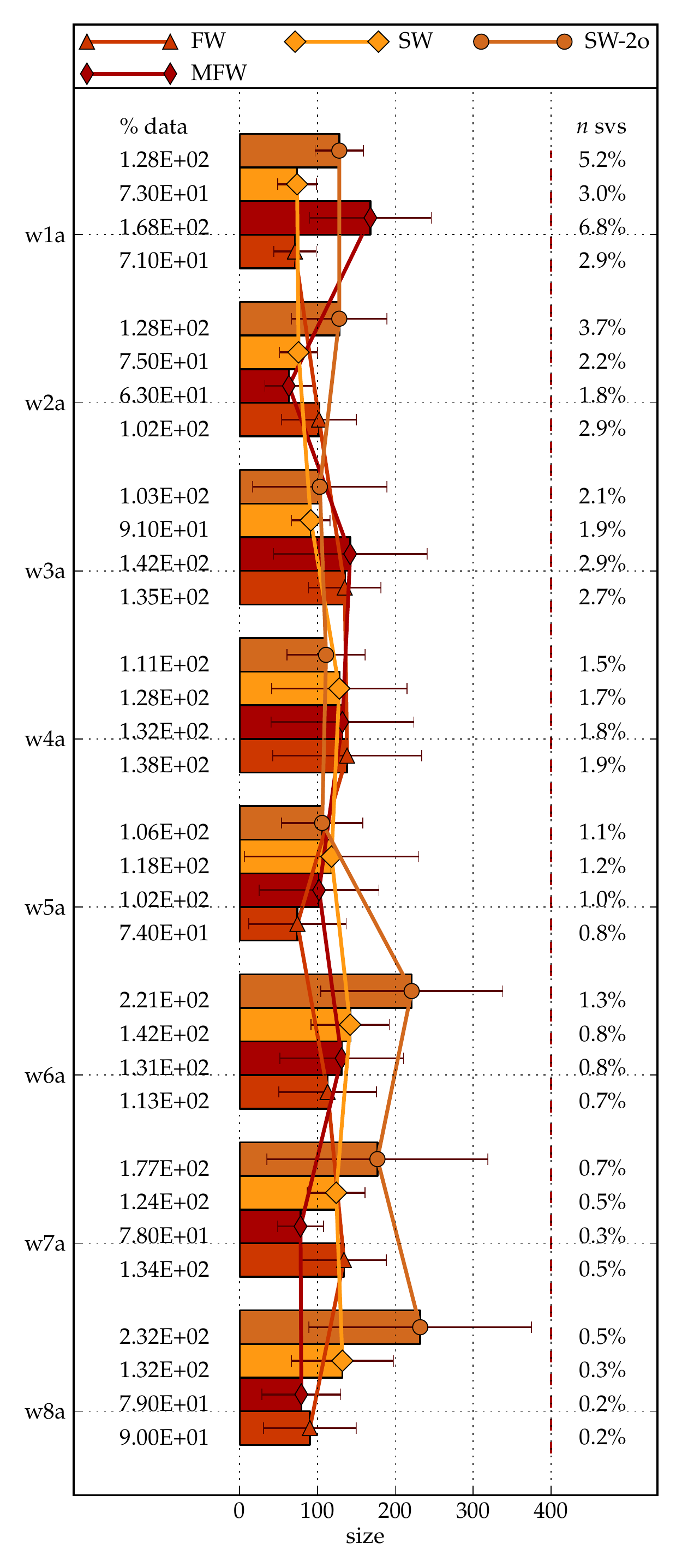}
	\end{minipage}
\qquad
	\begin{minipage}[c][0.95\height]{%
	   0.5\textwidth}
	   \centering%
	   \begin{tabular}{c}
\includegraphics[width=\textwidth]{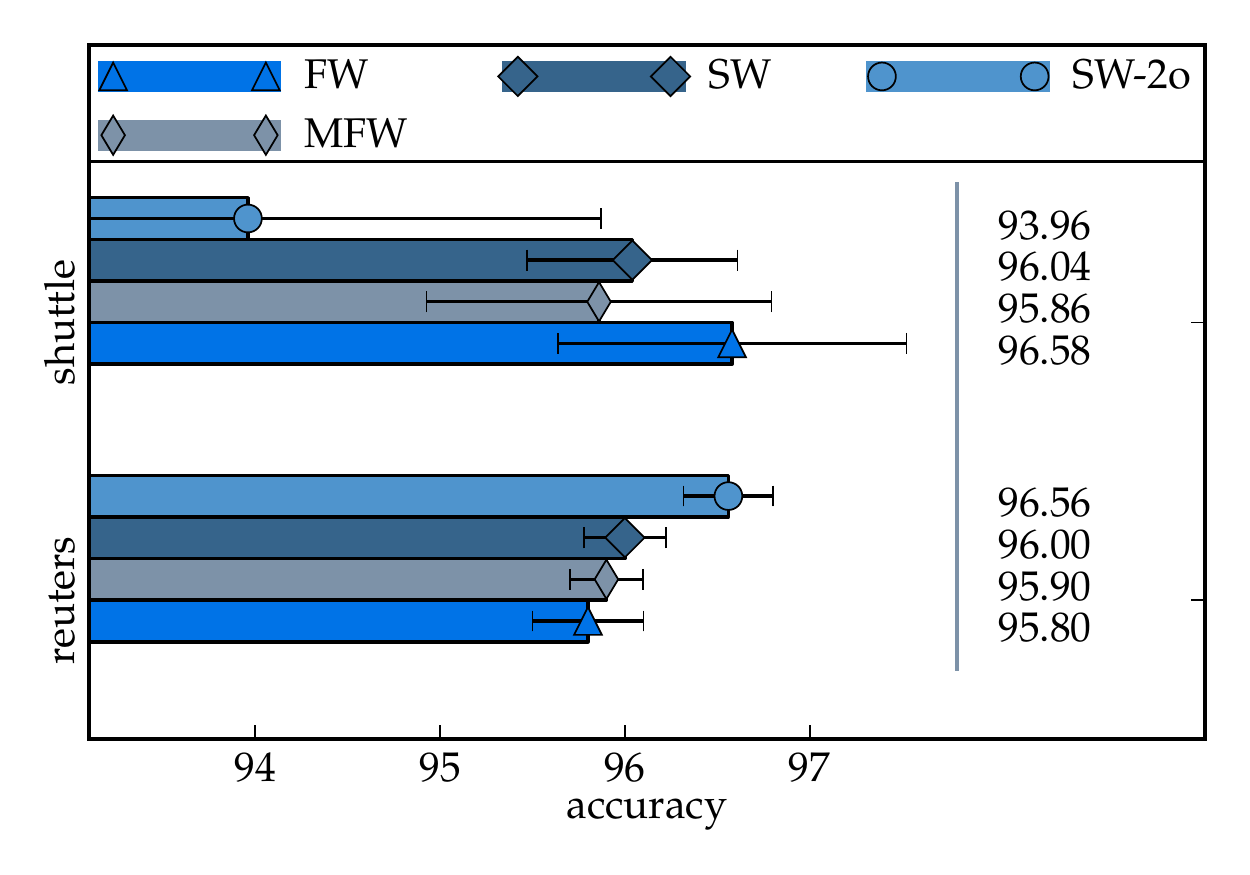}\\
\includegraphics[width=\textwidth]{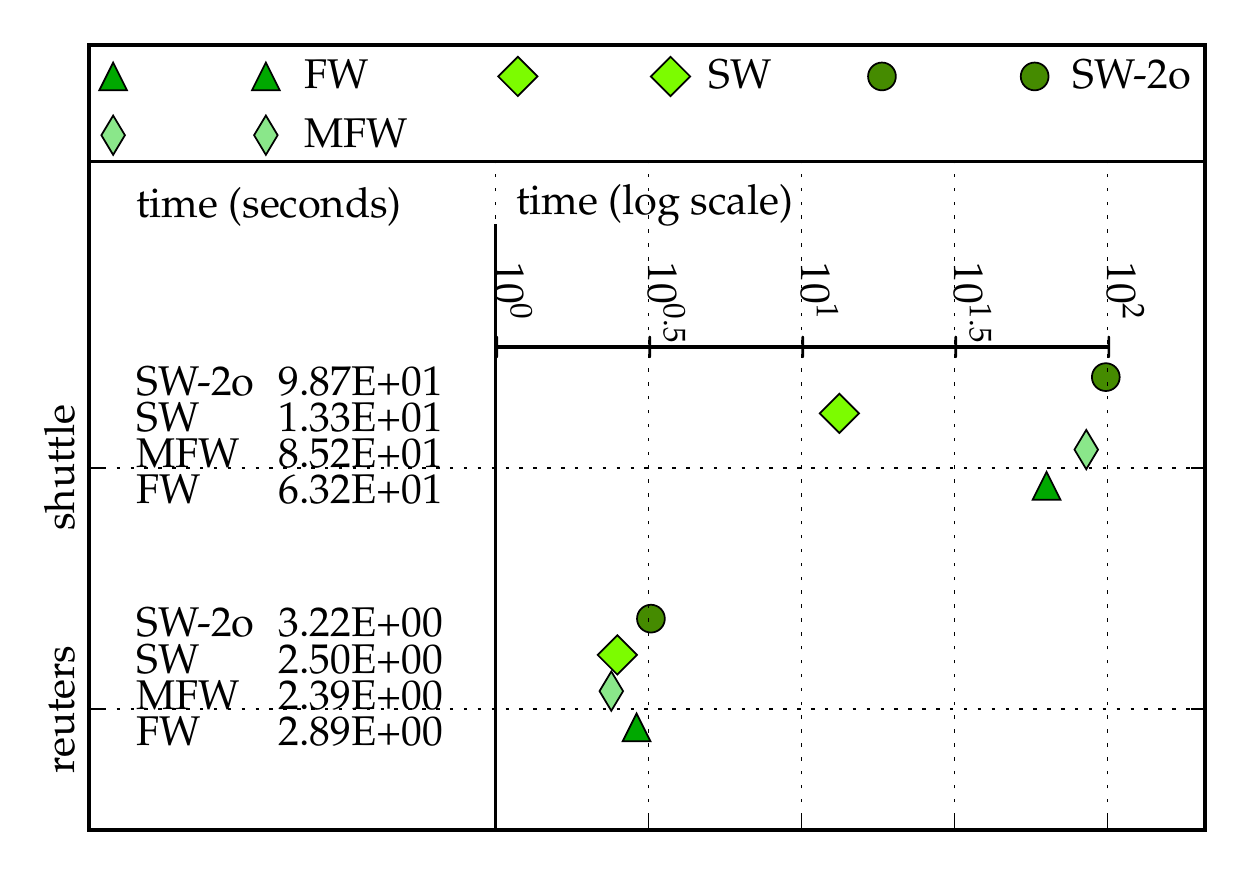}\\
\includegraphics[width=\textwidth]{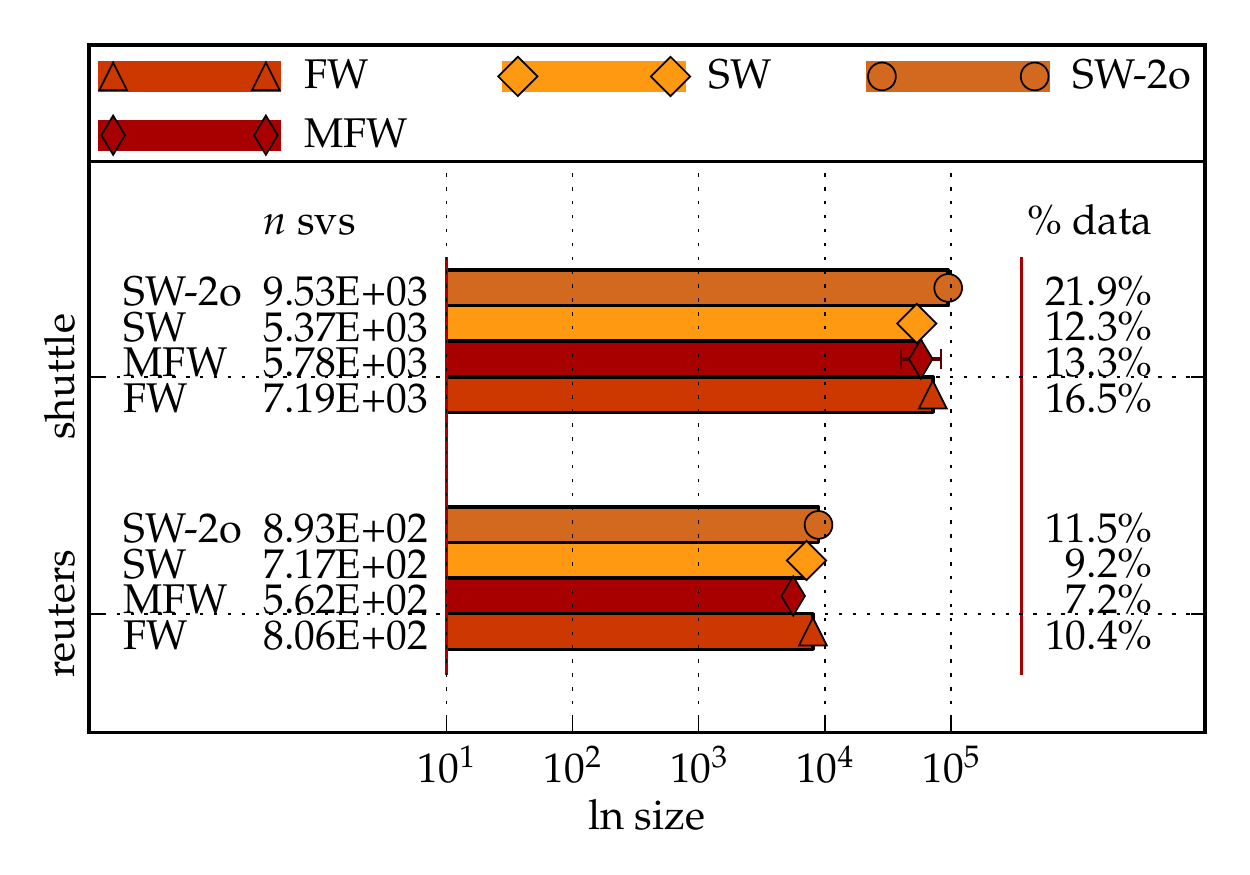}\\
\end{tabular}
	\end{minipage}
\vspace{0.2cm}	
\caption{\small On the left, model sizes obtained with the polynomial kernel in the datasets of the \textbf{Web} collection. On the right, testing accuracies, running times and coreset sizes obtained in the \emph{Shuttle} and \emph{Reuters} datasets. \label{poly-rest}}
\end{figure}

\section{Conclusions}

The main contribution of this paper is twofold. On the theoretical side, we proposed a new variant of the FW method for the general problem of maximizing a concave function on the unit simplex, introducing a novel way to perform away steps in the FW method devised to boost its convergence. On the practical side, we demonstrated that our approach is very effective in improving the performance of state-of-the-art SVM learners for large datasets, further expanding on the research about FW methods for Machine Learning problems.

We presented two variants of the procedure, SWAP and SWAP-2o, for which we provided a thorough theoretical analysis. First, we demonstrated that they converge globally. Second, we showed that SWAP and SWAP-2o asymptotically exhibit a linear rate of convergence, which is, as in the case of the MFW method, the main additional property with respect to the standard FW method. Finally, we proved that they achieve a primal-dual gap lower than a given tolerance $\varepsilon$ in $\cO(1/\varepsilon)$ iterations, independently of $m$, the dimensionality of the feasible space and the number of examples in SVM problems.

We then carried out an extensive set of performance evaluation experiments for both variants of the algorithm. The obtained results demonstrated that, in contrast to the MFW method, our approach provides a useful and robust alternative to the FW method for training SVMs.   

Most often, the proposed methods SWAP and SWAP-2o improved on the performance of MFW. The SWAP method was faster than MFW on all the datasets of the \textbf{Adult} collection, the \textbf{Web} collection and the \textbf{Protein} problem.  In the large-scale group of Figure \ref{small-and-large-time}(b) SWAP outperformed MFW on $4$ (out of $5$) datasets. In the medium-scale problems of Figure \ref{small-and-large-time}(a) SWAP was slightly slower. 

The SWAP-2o method was faster than MFW on all the datasets of the \textbf{Adult} collection, $6$ (out of $8$) datasets in the \textbf{Web} collection and $4$ (out of $5$) datasets in the large-scale group of Figure \ref{small-and-large-time}(b). SWAP-2o was also faster in the \textbf{Protein} problem and slightly faster on the medium-scale problems of Figure \ref{small-and-large-time}(a). 

The conclusion that SWAP and SWAP-2o are faster than MFW was found statistically significant at significance levels of $1\%$ or better. Often, the SWAP method improved on MFW by one order of magnitude and sometimes by two orders of magnitude. In addition, in the cases in which MFW was faster, the advantage was less significant than the improvements of our techniques on MFW. 

The proposed methods were also faster than the basic FW method several times. For example, SWAP ran in median $15$ times faster than FW in the \textbf{Adult} collection and SWAP-2o ran $20$ times faster. Similar results were observed in the \textbf{Shuttle} and \textbf{Protein} datasets. We found that the conclusion that SWAP is faster than FW is statistically significant at a critical value of around $4\%$. In contrast, we were not able to reject the hypothesis that MFW and FW lead to similar training times.  
Similarly, we cannot conclude that FW and SWAP-2o have different running times.  
      
Another important conclusion of our experimental results arises after an analysis of the cases in which either FW or MFW \emph{fail} in improving running times of CVM by a significant amount. 

\begin{itemize}
\item In some cases, away steps of MFW significantly speed-up the FW method. Some examples were the \textbf{Adult} collection, the \textbf{Shuttle} and \textbf{Protein} datasets. In those cases, the SWAP method is competitive with or faster than MFW and significantly faster than FW.  

\item In some other cases, classic away steps fail. MFW achieves in those cases noticeably worse running times. For instance, we observed this behavior in the \textbf{Web} collection, the \textbf{USPS-Ext} and \textbf{KDD-10pc} datasets. In those cases, the SWAP method is clearly faster than MFW. In addition, it is competitive with the fastest algorithm (FW). 

\end{itemize}
  
We conclude that the SWAP method can be expected to be faster than MFW in those cases in which classic away steps effectively boost the convergence of the FW method but also very competitive against FW when away steps fail. Thus, SWAP is a robust alternative to FW, MFW or CVM. From this point of view, the SWAP-2o method is less appealing. Even if SWAP-2o outperforms more significantly the standard FW when away steps are useful, this technique seems to fail very often in the same cases in which MFW fails. If we knew that away steps were going to be useful for a given problem, SWAP-2o would be the algorithm of choice. However, since we cannot predict that in advance, MFW and SWAP-2o are less reliable in practice.

Finally, our experiments have demonstrated that the improvements in running time that we obtain on FW or MFW do not come at the expense at the expense of testing accuracy. Most of the time SWAP is slightly more accurate than FW and as accurate as MFW.

\appendix
\section{Technical Results}\label{appendice}
Here we report statements and proofs of a number of technical results, which are used in the theoretical analysis of Section 4.

\subsection{Perturbation Analysis}
We follow the analysis presented in \cite{Damla06linearconvergence}, which is in turn based on the perturbation method of Robinson \cite{Robinson1982Perturbed}. Consider the following \emph{perturbed} variant of
(\ref{eq:generic-concave-on-the-simplex}),
\begin{equation}\label{eq:perturbed-concave-on-the-simplex}
\begin{aligned}
\maximize_{\balpha} &\;\; w(\balpha) = g(\balpha) - \bz^{T}\balpha \\
\mbox{subject to} &\;\; \uno^{T}\balpha = 1, \;\; \balpha \geq 0 \ ,
\end{aligned}
\end{equation}
where $\bz \in \bbR^{m}$ is perturbation vector.

Now, suppose we have a $\Delta_\star$-approximate solution $\balpha_\star \in \bbR^{m}$. We are aimed to show that $\balpha_\star$ is
the solution of a perturbed problem with a certain $\bz$. We define
$\bz = \bz(\balpha_\star,\Delta_\star)$ by
\begin{equation}\label{eq:perturbation_vector}
z_i = \left\{
\begin{aligned}
  \Delta_{\star} & \quad \mbox{if} \quad & \alpha_{\star i} = 0, \\
  \nabla g(\balpha_\star)_i - \balpha_\star^{T}\nabla g(\balpha_\star) & \quad \mbox{if} \quad & \alpha_{\star i} > 0. \\
\end{aligned} \right. \ 
\end{equation}
Note first that if $\alpha_{\star i} \neq 0$
\begin{equation}\label{eq:perturbation_vector_coordinate}
\begin{aligned}
z_i &= \nabla g(\balpha_\star)_i - \balpha_\star^{T}\nabla g(\balpha_\star) \geq  - \Delta_\star\\
z_i &= \nabla g(\balpha_\star)_i - \balpha_\star^{T}\nabla g(\balpha_\star) \leq  + \Delta_\star
\ , %
\end{aligned}
\end{equation}
because $\balpha_\star$ is a $\Delta_\star$-approximate solution. If $\alpha_{\star i} = 0$, $z_i=\Delta_\star$ by construction. Then,
\begin{equation}\label{eq:bound_norm_perturbation_vector}
\|\bz\|^{2} = \sum_i |z_i|^{2} \leq m \Delta_\star^{2} \ .
\end{equation}
Note in addition that
\begin{equation}\label{eq:kkt_gap_perturbation_vector}
\begin{aligned}
\balpha^{T}_\star\bz = \sum_i \alpha_{\star i} z_i = \sum_{i:
\alpha_{\star i} \neq 0} \alpha_{\star i} z_i & = \balpha^{T}_\star
\nabla g(\balpha_\star) - \left(\balpha^{T}_\star \nabla
g(\balpha_\star)\right) \balpha^{T}_\star \uno = 0 \ ,\\
\end{aligned}
\end{equation}
because $\balpha_\star$ is feasible for (\ref{eq:generic-concave-on-the-simplex}). Note finally that
\begin{equation}\label{eq:perturbed_gradient}
\nabla w(\balpha) = \nabla g(\balpha) - \bz \ .
\end{equation}
Thus, from Eqn. (\ref{eq:kkt_gap_perturbation_vector}) we obtain that the following stationarity condition is fulfilled
\begin{equation}\label{eq:perturbed_stationarity}
\balpha_\star^{T} \nabla w(\balpha_\star) = \balpha_\star^{T} \nabla
g(\balpha_\star) - \balpha_\star^{T} \bz = \balpha_\star^{T} \nabla
g(\balpha_\star) \ .
\end{equation}
The following lemma follows easily from the previous remarks.
\begin{lemma} If $\balpha_{\star}$ is a $\Delta_\star$-approximate solution, then $\balpha_{\star}$ is optimal for problem
(\ref{eq:perturbed-concave-on-the-simplex}) with $\bz = \bz(\balpha_\star,\Delta_\star)$ as defined in Eqn. (\ref{eq:perturbation_vector}).
\end{lemma}
\begin{proof} It follows from the concavity of problem (\ref{eq:perturbed-concave-on-the-simplex}) and the remarks above.
See \cite{Damla06linearconvergence}, Lemma 3.3 for details.
\end{proof}

The next lemma is the basis of the analysis of the rate of convergence for the modified Frank-Wolfe methods.

\begin{lemma} Let $\balpha^{\ast}$ be the solution of problem (\ref{eq:generic-concave-on-the-simplex}) and
$\balpha_{\star}$ a $\Delta_\star$-approximate solution. Then,
\begin{equation}\label{eq:key_inequality}
g(\balpha^{\ast}) - g(\balpha_\star) \leq \|\bz\| \|\balpha^{\ast} -
\balpha_\star\| \leq \sqrt{m} \Delta_\star \|\balpha^{\ast} -
\balpha_\star\| \ .
\end{equation}
\end{lemma}
\begin{proof} The vector $\balpha^{\ast}$ is feasible for the perturbed problem
(\ref{eq:perturbed-concave-on-the-simplex}) with $\bz =
\bz(\balpha_\star,\Delta_\star)$. Since $\balpha_\star$ is optimal
for this problem, we have $g(\balpha^{\ast}) - \bz^{T}\balpha^{\ast}
\leq g(\balpha_\star) - \bz^{T}\balpha_\star$. This demonstrates the
first inequality. The other follows from Eqn. 
(\ref{eq:bound_norm_perturbation_vector}). 
\end{proof}
From here, the following lemma follows easily.
\begin{lemma}\label{lemma:key_lemma_for_linear_convergence} 
Suppose condition B2 holds. Let $\balpha^{\ast}$ be the solution of problem (\ref{eq:generic-concave-on-the-simplex}) and
$\balpha_{\star}$ a $\Delta_\star$-approximate solution, where $\Delta_\star$ is 
sufficiently small. Then,
\begin{equation}\label{eq:key_inequality}
g(\balpha^{\ast}) - g(\balpha_\star) \leq N m
\Delta_\star^{2} 
\end{equation}
for some Lipschitz constant $N$.
\end{lemma}
\begin{proof} See
\cite{Damla06linearconvergence} to see how from the Robinson
condition it follows that there exists a Lipschitz  constant $N$
such that, for sufficiently small $\Delta_\star$,  $\|\balpha^{\ast}
- \balpha_\star\| \leq N \|\bz\| \leq N \sqrt{m} \Delta_\star$.
Combining this result with the previous lemma yields the result.
\end{proof}

Now, since $g$ is twice differentiable, the Taylor expansion 
for $g(\balpha_k + \lambda\bd)$ as a function of $\lambda$ is
\begin{small}
\begin{equation}\label{eq:taylor_second_order}
\begin{aligned}
g\left(\balpha_k + \lambda\bd\right) &= g(\balpha_k) + \lambda
\nabla g(\balpha_k)^{T}\bd + \frac{1}{2} \lambda^{2} \bd^{T}
\nabla^{2}g(\tilde{\balpha}) \bd\ ,
\end{aligned}
\end{equation}
\end{small}
where $\tilde{\balpha}$ is some point on the line between $\balpha_k
+ \lambda\bd$ and $\balpha_k$. Since $g$ is concave, the Hessian
matrix $g(\tilde{\balpha})$ is negative semi-definite, so the
last term is always non-positive. To obtain a bound for
$g\left(\balpha_k + \lambda\bd\right) - g(\balpha_k)$, we need a
bound $L$ on the norm of $\nabla^{2}g({\balpha})$ over the simplex.
We can set $L$ to the largest absolute value of an eigenvalue of
this matrix. 
We therefore obtain the following bound:
\begin{small}
\begin{equation}\label{eq:generic_bound_second_order}
\begin{aligned}
g\left(\balpha_k + \lambda\bd\right) - g(\balpha_k) & \geq  \lambda \nabla g(\balpha_k)^{T}\bd - \frac{1}{2}
\lambda^{2}L \|\bd\|^{2} \ .
\end{aligned}
\end{equation}
\end{small}

We now exploit the previous bound to analyze the
improvement in the objective function $\delta_{\mbox{\tiny fw}}$
after a standard FW step, and the improvement
$\delta_{\mbox{\tiny swap}}$ after a SWAP step in Algorithm \ref{alg:SWAP-generic}.

\subsection{Objective Function Improvement after Frank-Wolfe Steps}
Under hypothesis B1, we now derive a lower bound for the improvement in the objective function $g(\balpha_k + \lambda d_k) - g(\balpha_k)$ in the case a FW step is performed.

For a FW step we have $\bd_k = \be_{i\ast} - \balpha_k$. Thus,
\begin{small}
\begin{equation}\label{eq:fw-application-generic-bound}
\begin{aligned}
\delta_{\mbox{\tiny fw}} &= g\left(\balpha_k +
\lambda(\be_{i\ast} - \balpha_k)\right) - g(\balpha_k)\\
&\geq \lambda \left(\nabla g(\balpha_k)_{i\ast} - \nabla
g(\balpha_k)^{T}\balpha_k\right) - \frac{L}{2} \lambda^{2}
\left\|\be_{i\ast} - \balpha_k\right\|^{2} \ .
\end{aligned}
\end{equation}
\end{small}
But both $\be_{i\ast}$ and $\balpha_k$ lie in the simplex. Hence
$\left\|\be_{i\ast} - \balpha_k\right\|^{2} \leq 2$. This leads to
\begin{equation}\label{eq:fw-case-0}
\begin{aligned}
\delta_{\mbox{\tiny fw}} &\geq g\left(\balpha_k +
\lambda(\be_{i\ast} - \balpha_k)\right) - g(\balpha_k)\\
&\geq \lambda \left(\nabla g(\balpha_k)_{i\ast} - \nabla
g(\balpha_k)^{T}\balpha_k\right) - L\lambda^{2} \ .
\end{aligned}
\end{equation}
The maximum of the right-hand side is obtained for
\begin{small}
\begin{equation}\label{eq:fw-case-1}
\begin{aligned}
\lambda^{\star}_{\mbox{\tiny fw}} & = \frac{\nabla
g(\balpha_k)_{i\ast} - \nabla g(\balpha_k)^{T}\balpha_k}{2L} \ .
\end{aligned}
\end{equation}
\end{small}
If $\lambda^{\star}_{\mbox{\tiny fw}} \leq 1$, the improvement in
the objective function after an iteration marked
as a FW step in Algorithm \ref{alg:SWAP-generic} is bounded by
\begin{equation}\label{eq:fw-bound_improvement-a}
\begin{aligned}
\delta_{\mbox{\tiny fw}} & \geq \frac{\left(\nabla
g(\balpha_k)_{i\ast} - \nabla g(\balpha_k)^{T}\balpha_k\right)^{2}
}{4L} \ ,
\end{aligned}
\end{equation}
and by reordering we obtain
\begin{equation}\label{eq:fw-case-2-a}
\nabla g(\balpha_k)_{i\ast} - \nabla
g(\balpha_k)^{T}\balpha_k \leq 2 \sqrt{L \delta_{\mbox{\tiny
fw}}} \leq 2 \sqrt{L \delta_k} \ ,
\end{equation}
where the latter inequality follows from the definition of
$\delta_k=\max\left(\delta_{\mbox{\tiny fw}}, \delta_{\mbox{\tiny
swap}}\right)$. Now, if $\lambda^{\star}_{\mbox{\tiny fw}} > 1$, we
cannot use this step-size. In that case we use the step-size
$\lambda = 1$. But $\lambda^{\star}_{\mbox{\tiny fw}} > 1$ implies
\begin{small}
\begin{equation}\label{eq:fw-case_infeasisble}
\begin{aligned}
\nabla g(\balpha_k)_{i\ast} - \nabla
g(\balpha_k)^{T}\balpha_k & \geq 2L \ .
\end{aligned}
\end{equation}
\end{small}
Thus, using Eqn. (\ref{eq:fw-case-0}) with $\lambda = 1$ and exploiting the inequality above, the improvement in the objective function for a FW step can be
bounded in this case as
\begin{equation}\label{eq:fw-bound_improvement-b}
\begin{aligned}
\delta_{\mbox{\tiny fw}} & \geq \frac{\left(\nabla
g(\balpha_k)_{i\ast} - \nabla g(\balpha_k)^{T}\balpha_k\right)}{2} \
,
\end{aligned}
\end{equation}
which leads to
\begin{equation}\label{eq:fw-case-2-b}
\nabla g(\balpha_k)_{i\ast} - \nabla
g(\balpha_k)^{T}\balpha_k \leq 2 \delta_{\mbox{\tiny fw}}
\leq 2 \delta_k \ .
\end{equation}
In any case, we have the following bound for the improvement of the
objective function:
\begin{small}
\begin{equation}\label{eq:fw-bound_improvement}
\begin{aligned}
\delta_{\mbox{\tiny fw}} & \geq \min\left( \frac{\left(\nabla
g(\balpha_k)_{i\ast} - \nabla g(\balpha_k)^{T}\balpha_k\right)^{2}
}{4L} \ , \ \frac{\left(\nabla g(\balpha_k)_{i\ast} - \nabla
g(\balpha_k)^{T}\balpha_k\right)}{2}\right) \ ,
\end{aligned}
\end{equation}
\end{small}
which guarantees that for any $k$
\begin{small}
\begin{equation}\label{eq:fw-case-2}
\nabla g(\balpha_k)_{i\ast} - \nabla
g(\balpha_k)^{T}\balpha_k \leq \max\left(2 \sqrt{L
\delta_{\mbox{\tiny fw}}} \ , \ 2 \delta_{\mbox{\tiny fw}}\right)
\leq \max\left(2 \sqrt{L \delta_k} \ , \ 2 \delta_k\right) \ .
\end{equation}
\end{small}
Now, since $\nabla g(\balpha_k)_{i} \leq \nabla
g(\balpha_k)_{i\ast}\,$ $\forall i$, the following inequality is
guaranteed at each iteration of Algorithm \ref{alg:SWAP-generic} for
any $i$:
\begin{equation}\label{eq:bound_weak}
\nabla g(\balpha_k)_{i} - \balpha_k^{T} \nabla g(\balpha_k) \leq \max\left(2 \sqrt{L \delta_k} \ , \ 2 \delta_k\right) \
.
\end{equation}

\subsection{Objective Function Improvement after SWAP Steps}
We now bound the improvement obtained by SWAP steps. In this case, $\bd = \be_{i\ast} -
\be_{j\ast}$. Thus,
\begin{small}
\begin{equation}\label{eq:swap-application-generic-bound}
\begin{aligned}
\delta_{\mbox{\tiny swap}} &= g\left(\balpha_k +
\lambda(\be_{i\ast} - \be_{j\ast})\right) - g(\balpha_k)\\
&\geq \lambda \left(\nabla g(\balpha_k)_{i\ast} - \nabla
g(\balpha_k)_{j\ast}\right) - \frac{L}{2} \lambda^{2}
\left\|\be_{i\ast} - \be_{j\ast}\right\|^{2} \ .
\end{aligned}
\end{equation}
\end{small}
But $\|\be_{i\ast} - \be_{j\ast}\|^{2}=2$. Thus,
\begin{small}
\begin{equation}\label{eq:swap-case-0}
\begin{aligned}
\delta_{\mbox{\tiny swap}} & \geq \lambda \left(\nabla
g(\balpha_k)_{i\ast} - \nabla g(\balpha_k)_{j\ast}\right) -
\lambda^{2} L \ .
\end{aligned}
\end{equation}
\end{small}
The maximum of the right-hand side is obtained for
\begin{small}
\begin{equation}\label{eq:swap-case-1}
\begin{aligned}
\lambda^{\star}_{\mbox{\tiny swap}} & = \frac{\nabla
g(\balpha_k)_{i\ast} - \nabla g(\balpha_k)_{j\ast}}{2L} \ .
\end{aligned}
\end{equation}
\end{small}
If $\lambda^{\star}_{\mbox{\tiny swap}} \leq 1$, the improvement in
the objective function for an unconstrained SWAP step, that is, an
iteration marked as SWAP-add in Algorithm \ref{alg:SWAP-generic}, is bounded as
\begin{equation}\label{eq:swap-case-2-a}
\begin{aligned}
\delta_{\mbox{\tiny swap}} & \geq \frac{\left(\nabla
g(\balpha_k)_{i\ast} - \nabla g(\balpha_k)_{j\ast}\right)^{2}}{4L} \ .
\end{aligned}
\end{equation}
Note now that $\balpha_k^{T} \nabla g(\balpha_k) \leq \nabla
g(\balpha_k)_{i\ast}$ because $\nabla g(\balpha_k)_{i\ast} = \max_i
\nabla g(\balpha_k)_{i}$ and $\balpha_k^{T}\uno = 1$. This
observation leads to
\begin{equation}\label{eq:swap-bound_improvement-a}
\begin{aligned}
\delta_{\mbox{\tiny swap}} &\geq \frac{\left(\balpha_k^{T} \nabla
g(\balpha_k) - \nabla g(\balpha_k)_{j\ast}\right)^{2}}{4L} \ .
\end{aligned}
\end{equation}
By reordering, we obtain the following inequality,
\begin{equation}\label{eq:swap-case-3-a}
\balpha_k^{T} \nabla g(\balpha_k) - \nabla
g(\balpha_k)^{T}_{j\ast} \leq 2 \sqrt{L \delta_{\mbox{\tiny
swap}}} \leq 2 \sqrt{L \delta_k} \ ,
\end{equation}
where we have used the definition of $\delta_k$. Now, if
$\lambda^{\star}_{\mbox{\tiny swap}} > 1$, we cannot use this
step-size. In that case we use the step-size $\lambda = 1$. Recall that
we are supposing to be performing a SWAP-add step. In a way analogous
to the Frank-Wolfe case, $\lambda^{\star}_{\mbox{\tiny
swap}} > 1$ implies
\begin{small}
\begin{equation}\label{eq:swap-case_infeasisble}
\begin{aligned}
\nabla g(\balpha_k)_{i\ast} - \nabla
g(\balpha_k)_{j\ast} & \geq 2L \ .
\end{aligned}
\end{equation}
\end{small}
Thus, the improvement in the objective function is bounded in this case by
\begin{equation}\label{eq:swap-bound_improvement-b}
\begin{aligned}
\delta_{\mbox{\tiny swap}} & \geq \frac{\balpha_k^{T} \nabla
g(\balpha_k) - \nabla g(\balpha_k)_{j\ast}}{2} \ ,
\end{aligned}
\end{equation}
which leads to
\begin{equation}\label{eq:swap-case-3-b}
\balpha_k^{T} \nabla g(\balpha_k) - \nabla
g(\balpha_k)_{j\ast} \leq 2 \delta_{\mbox{\tiny swap}} \leq 2
\delta_k \ .
\end{equation}
In any case, we have the following bounds for the SWAP case
\begin{small}
\begin{equation}\label{eq:swap-bound_improvement}
\begin{aligned}
\delta_{\mbox{\tiny swap}} & \geq
\min\left(\frac{\left(\balpha_k^{T} \nabla g(\balpha_k) - \nabla
g(\balpha_k)_{j\ast}\right)^{2}}{4L} \ , \ \frac{\left(\balpha_k^{T}
\nabla g(\balpha_k) - \nabla g(\balpha_k)_{j\ast}\right)}{2} \right)
\ ,
\end{aligned}
\end{equation}
\end{small}
which leads to
\begin{small}
\begin{equation}\label{eq:swap-case-3}
\balpha_k^{T} \nabla g(\balpha_k) - \nabla
g(\balpha_k)_{j\ast} \leq \max\left(2 \sqrt{L
\delta_{\mbox{\tiny swap}}} \ , \ 2 \delta_{\mbox{\tiny
swap}}\right) \leq \max\left(2 \sqrt{L \delta_k} \ , \ 2
\delta_k\right) \ .
\end{equation}
\end{small}
Note now that the definition of $j\ast$ can be rearranged as
\begin{equation}\label{eq:swap-case-4}
j\ast \in \argmin_{j \in \cI_k} \nabla g(\balpha_k)_j - \balpha_k^{T}\nabla g(\balpha_k) = \argmax_{j \in \cI_k} \balpha_k^{T}\nabla
g(\balpha_k) - \nabla g(\balpha_k)_j \ .
\end{equation}
Thus, we obtain that the following inequality is guaranteed at each
iteration of Algorithm \ref{alg:SWAP-generic} $\forall i$ such that
$\alpha_{k,i} > 0$:
\begin{equation}\label{eq:bound_strong}
\balpha_k^{T} \nabla g(\balpha_k) - \nabla
g(\balpha_k)_{i} \leq \max\left(2 \sqrt{L \delta_k} \ , \ 2
\delta_k\right) \ .
\end{equation}

\begin{remark} After a swap-drop step in Algorithm \ref{alg:SWAP-generic} we cannot bound the improvement in the objective function, because the clipped value of the step-size $\lambda_{\mbox{\tiny swap}\star}$ may be arbitrarily small. However, it is not hard to show that the objective function value does not decrease. 
\end{remark}


\bibliographystyle{plain}	
\bibliography{SWAP_bibliography}

\end{document}